\documentclass[11pt]{article}
\usepackage[margin=1in]{geometry}
\usepackage[round,authoryear]{natbib}

\usepackage{amsmath,amssymb,amsthm,mathtools}
\usepackage{enumitem}
\usepackage{graphicx}
\usepackage{placeins}
\usepackage{xcolor}
\usepackage{hyperref}
\usepackage{url}
\hypersetup{
  colorlinks=true,
  linkcolor=blue,
  citecolor=blue,
  urlcolor=blue
}

\newtheorem{theorem}{Theorem}[section]
\newtheorem{lemma}[theorem]{Lemma}
\newtheorem{proposition}[theorem]{Proposition}
\newtheorem{corollary}[theorem]{Corollary}
\newtheorem{assumption}[theorem]{Assumption}
\newtheorem{remark}[theorem]{Remark}

\newcommand{\R}{\mathbb{R}}

\newcommand{\E}{\mathbb{E}}
\newcommand{\KL}{\mathrm{KL}}
\newcommand{\dd}{\,\mathrm{d}}
\newcommand{\norm}[1]{\left\lVert #1\right\rVert}
\newcommand{\abs}[1]{\left\lvert #1\right\rvert}

\newcommand{\one}{\mathbf{1}}
\newcommand{\divergence}{\nabla\!\cdot}

\allowdisplaybreaks

\title{\bf Finite-Particle Convergence Rates for Conservative and Non-Conservative Drifting Models}
\author{Krishnakumar Balasubramanian\\ Department of Statistics, University of California, Davis\\ \texttt{kbala@ucdavis.edu}}
\date{}
\hypersetup{pdfauthor={},pdftitle={Finite-Particle Convergence Rates for Conservative and Non-Conservative Drifting Models}}

\begin{document}
\maketitle

\begin{abstract}
We analyze finite-particle drifting models for one-step generative modeling. For a conservative velocity given by the difference of the kernel-smoothed data and model scores, a joint-entropy identity yields continuous-time bounds for the smoothed Fisher discrepancy and the squared particle velocity. The finite-particle correction involves reciprocal kernel density estimates. We give local-occupancy conditions and exact expectation bounds, and show that these expectations diverge for full-support initial densities in a shrinking-bandwidth regime. Keeping the bandwidth dependence of the quadrature constants explicit yields a conditional root residual-velocity rate of $N^{-1/(d+4)}$ under uniform regularity, with a corresponding rate under weaker quadrature growth conditions. We also analyze the original displacement field with the exact Laplace kernel. A companion kernel gives a weighted coercivity estimate that accommodates the unbounded empirical scale and the kernel's nondifferentiability. Localization controls the particle velocity, while relative scale alignment absorbs the mismatch into dissipation. Both analyses quantify the size of an additional drift correction under explicit regularity and stability assumptions.
\end{abstract}

\section{Introduction}

Drifting models~\citep{deng2026drifting} learn one-step generators by moving the model distribution during training, avoiding a long iterative sampler at inference time. We study their finite-particle behavior. Given the data law $\nu$, a model law $\mu$, and a positive kernel $K_h$, we define
\begin{equation}
 M_{\alpha,h}(z)
 :=\frac{\int (y-z)K_h(z-y)\,\alpha(\dd y)}
          {\int K_h(z-y)\,\alpha(\dd y)},\qquad
 u_{\nu,\mu,h}^{\mathrm{disp}}:=M_{\nu,h}-M_{\mu,h}.
 \label{main:eq:displacement}
\end{equation}
The data term attracts a point toward the data, whereas the model term repels it from regions represented by the model. Training regresses a generated point $z$ toward the frozen target $z+\eta u_{\nu,\mu,h}^{\mathrm{disp}}(z)$. We show that exact regression on a fixed latent batch gives a particle Euler update (see Appendix~\ref{app:stop-gradient-ode}  for more details), and analyze its small-step continuous-time limit. Specifically, we analyze how batch size and kernel bandwidth control the residual drift.

The displacement field is generally non-conservative because its normalization varies with position~\citep{franz2026drifting}. We analyze it alongside a conservative alternative based on kernel density estimator (KDE) scores. With $\rho_{\alpha,h}=K_h*\alpha$ and $s_{\alpha,h}=\nabla\log\rho_{\alpha,h}$, we define
\begin{equation}
 b_{\nu,\mu,h}:=s_{\nu,h}-s_{\mu,h}
 =\nabla\{\log\rho_{\nu,h}-\log\rho_{\mu,h}\}.
 \label{main:eq:conservative}
\end{equation}
For the Gaussian kernel with covariance $h^2I_d$, $M_{\alpha,h}=h^2s_{\alpha,h}$, so the dynamics agree after rescaling time. For other kernels, displacement averaging and KDE differentiation produce distinct fields. We write $b$ for a conservative velocity and $u$ for a displacement velocity.

\begin{figure}[t]
 \centering
 \includegraphics[width=0.7\linewidth,trim=0 0 0 36bp,clip]{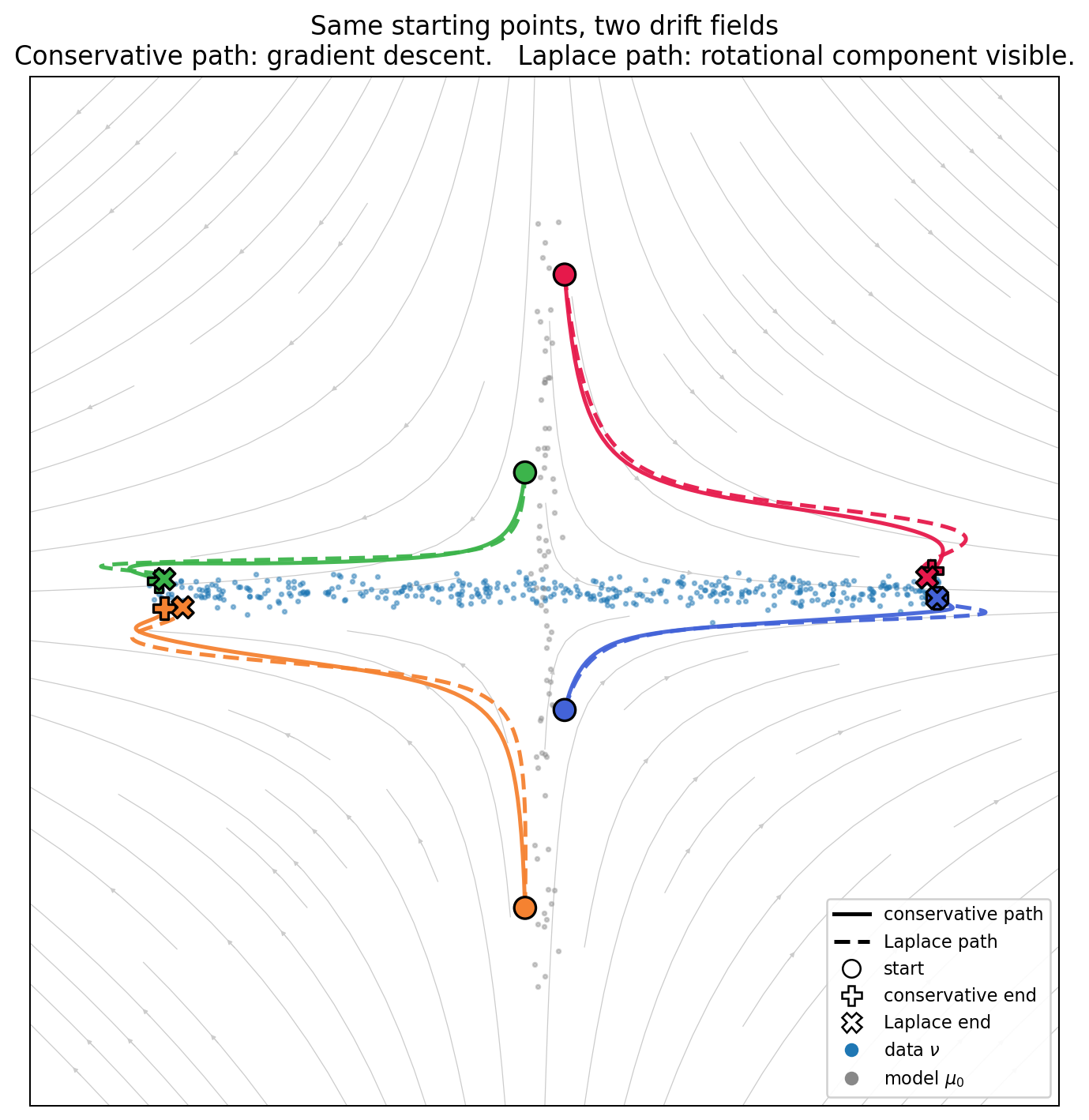}
 \caption{Conservative Gaussian drift (solid) and displacement-based Laplace drift (dashed) from the same initial cloud. The data occupy a thin horizontal cluster (blue), and the initial model occupies a thin vertical cluster (grey). Four trajectories are shown from $N=80$ particles at bandwidth $h=0.55$. The conservative field is a spatial gradient at each time; the Laplace field can have nonzero circulation. Curvature of the paths alone does not distinguish these properties. Appendix~\ref{subsec:conservative-vs-non} gives the accompanying discussion and a Helmholtz--Hodge decomposition of the two fields.}
 \label{fig:xy-paths}
\end{figure}

Figure~\ref{fig:xy-paths} illustrates the fields from a common initial
configuration.  We show that the comparison is more nuanced than
``conservative is better.''  Although the conservative field has zero spatial
circulation, its potential evolves with the particle law, so conservativeness
alone does not imply monotonicity along trajectories.  We also identify an
exact reciprocal self-interaction correction generated by the finite-particle
KDE.  The self-masked Laplace field avoids this correction, but its
position-dependent normalization introduces circulation and scale mismatch.
We derive $u_x=a_{x,h}b_x^\#+e_x$, expressing the field as a positively
preconditioned conservative score mismatch plus an error.  We show that scale
derivatives govern circulation, whereas data--model scale differences
determine $e_x$; these effects are related but distinct.  Thus, the fields
face different finite-particle obstructions, and neither uniformly dominates.

For conservative drifting, we prove a joint-entropy identity and quadrature
comparisons controlling the smoothed Fisher discrepancy and particle
velocity.  Under uniform quadrature control, we obtain the conditional root
rate $N^{-1/(d+4)}$ and, for growing quadrature constants, the modified rate
in Section~\ref{main:sec:conservative-rates}.  We control reciprocal KDEs via
local occupancy and an exact binomial identity, and show that full-support
initialization makes their expectation diverge in the shrinking-bandwidth
regime.  For the exact Laplace field, we construct a companion kernel yielding
weighted coercivity despite the cusp and unbounded scale.  We combine
localization with relative scale alignment to control particle velocity and
absorb the mismatch with coefficient $1-\sqrt{\theta_h}$.

In Section~\ref{main:sec:experiments}, we verify that these terms are active.
We show that omitting the conservative correction accounts for the full
entropy-identity defect, and identify rare low-density particles as the source
of full-support growth.  For Laplace drifting, we show that
$\theta=\mathcal D_h^{\mathrm{Lap}}/\mathcal W_h^{\mathrm{Lap}}$ organizes a
clear phase diagram.  The relative bound certifies a median 93.3\% of observed
dissipation, compared with 83.2\% for the absolute bound, and remains
informative along simulated flows.  Full proofs appear in the appendices.

\paragraph{Related work.}
Drifting~\citep{deng2026drifting} is related to score-difference flows~\citep{weber2023the} and to population gradient-flow interpretations of kernel-smoothed divergences~\citep{cao2026gradient,lai2026unified,turan2026generative,gretton2026wasserstein}. Concurrent work also introduces kernel-gradient drifting~\citep{estebancasadevall2026kernel}. Our contribution is a finite-particle entropy-rate analysis, including reciprocal-KDE corrections and a separate treatment of the Laplace displacement field. The sharp-normalization construction of \citet{franz2026drifting} produces a conservative field by scoring a companion KDE. Here we use that construction to analyze the original Laplace normalization. Identifiability and stability of companion kernels are studied by \citet{lee2026identifiability}.

Our joint-entropy argument builds on finite-particle analyses of Stein
variational gradient descent (SVGD)~\citep{shi2023finite,balasubramanian2024improved,he2026finite}.
SVGD uses a kernelized Stein velocity~\citep{liu2016stein,liu2017stein},
whereas conservative drifting uses the KDE-score residual itself.  Amortized
SVGD~\citep{feng2017amortized} likewise regresses a generator toward
transported samples.  We discuss further connections to generative transport,
constrained drifting, and Sinkhorn and lookahead methods in
Appendix~\ref{sec:related-work}.

\section{Conservative Drifting: Set-up and Assumptions}
\label{main:sec:setup}

Let $K_h(u)=h^{-d}K(u/h)$, $\rho=K_h*\nu$, and $s_\rho=\nabla\log\rho$. For $x=(x_1,\ldots,x_N)$, define
\begin{equation}
 \mu_x^N=\frac1N\sum_i\delta_{x_i},\qquad
 q_x(z)=\frac1N\sum_iK_h(z-x_i),\qquad
 b_x=s_\rho-\nabla\log q_x.
 \label{main:eq:kde}
\end{equation}
The interacting-particle system is $\dot X_i(t)=b_{X(t)}(X_i(t))$. For its joint density $p_t^N$, set $H_N(t)=\KL(p_t^N\|\rho^{\otimes N})$. Expectations $\E_t$ are with respect to this joint law. The Stein divergence is $\mathcal A_\rho f=\divergence f+s_\rho\cdot f$, and the relevant quantities are
\begin{align}
 \mathsf S_N(x)&=\frac1N\sum_i\mathcal A_\rho b_x(x_i),
 &\mathsf R_N(x)&=\frac1N\sum_i\frac1{q_x(x_i)},\notag\\
 \mathsf I_N(x)&=\int\norm{b_x(z)}^2q_x(z)\dd z,
 &\mathsf V_N(x)&=\frac1N\sum_i\norm{b_x(x_i)}^2.
 \label{main:eq:quantities}
\end{align}
Thus $\mathsf I_N$ is the smoothed Fisher discrepancy, and $\mathsf V_N$ is the mean squared particle speed. These quantities need not coincide with the center-evaluated Stein drift $\mathsf S_N$.

\begin{assumption}[Conservative regularity and stability]
\label{main:ass:conservative}
The kernel $K$ is positive, even, normalized, and $C^3$, with $m_2(K):=\int\norm{u}^2K(u)\dd u<\infty$ and finite $\Delta K(0)$. The target score is $C^2$. The particle ODE has a unique solution on $[0,T]$ with a differentiable joint density solving its Liouville equation. The entropy is absolutely continuous, differentiation under its integral is justified, and the spatial and joint-space integrations by parts used below have vanishing boundary terms. All displayed integrals are finite. There is a deterministic $\Lambda_T<\infty$ such that
\[
 \sup_{0\le t\le T}\E_t[\mathsf R_N(X(t))]\le\Lambda_T.
\]
For every visited configuration, the scalar functions $\mathcal A_\rho b_x$ and $\norm{b_x}^2$ are twice continuously differentiable and satisfy
\[
 \sup_z\norm{D^2(\mathcal A_\rho b_x)(z)}_{\mathrm{op}}\le B_{A,T}(h),\qquad
 \sup_z\norm{D^2\norm{b_x(z)}^2}_{\mathrm{op}}\le B_{V,T}(h).
\]
\end{assumption}

The constants may depend on $N,h,T$ and the initial law. Uniformity is required only when an asymptotic rate invokes it. These are assumptions on the evolving system, including sufficient smoothness for the scalar Hessians; $K\in C^3$ alone need not provide that higher regularity. Appendix~\ref{app:conservative-setup} gives the separate assumptions and sufficient derivative conditions. No log-concavity, log-Sobolev, or Poincar\'e inequality is imposed for the residual bounds.

\section{Finite-particle Rates for Conservative Drifting}
\label{main:sec:conservative-rates}

\begin{theorem}[Conservative entropy identity]
\label{main:thm:entropy}
Under the kernel and regular-flow conditions of Assumption~\ref{main:ass:conservative},
\begin{equation}
 \frac{\dd}{\dd t}H_N(t)
 =-N\E_t[\mathsf S_N(X(t))]
   -\Delta K_h(0)\E_t[\mathsf R_N(X(t))].
 \label{main:eq:entropy}
\end{equation}
\end{theorem}

The correction in \eqref{main:eq:entropy} comes from the dependence of $q_x$ on the center $x_i$ as well as on the evaluation point. For kernels with $\Delta K_h(0)<0$, it contributes positively to the entropy derivative. A spatial integration by parts gives $\int\mathcal A_\rho b_x\,q_x=\mathsf I_N(x)$. Symmetric-kernel quadrature then compares this integral with $\mathsf S_N$ and compares $\mathsf I_N$ with $\mathsf V_N$. The complete entropy calculation and both quadrature comparisons are in Appendices~\ref{app:conservative-entropy}--\ref{app:conservative-rates}.

\begin{theorem}[Conservative Fisher and velocity bounds]
\label{main:thm:rates}
Under Assumption~\ref{main:ass:conservative}, let $a_1=\{-\Delta K(0)\}_+$. Then
\begin{align}
 \frac1T\int_0^T\E_t[\mathsf I_N(X(t))]\dd t
 &\le\frac{H_N(0)}{NT}+\frac{a_1\Lambda_T}{Nh^{d+2}}
       +\frac{m_2(K)}2 B_{A,T}(h)h^2,
 \label{main:eq:fisher-rate}\\
 \frac1T\int_0^T\E_t[\mathsf V_N(X(t))]\dd t
 &\le\frac{H_N(0)}{NT}+\frac{a_1\Lambda_T}{Nh^{d+2}}+C_T(h)h^2,
 \label{main:eq:velocity-rate}
\end{align}
where $C_T(h)=\tfrac12m_2(K)\{B_{A,T}(h)+B_{V,T}(h)\}$.
\end{theorem}

The proofs are Theorem~\ref{thm:main-rate} and Corollary~\ref{cor:bandwidth-rate}. If $H_N(0)\le\kappa_0N$ and $T=N$, the three terms in \eqref{main:eq:velocity-rate} become
\begin{equation}
 \frac{\kappa_0}{N}+\frac{a_1\Lambda_N(h)}{Nh^{d+2}}+C_N(h)h^2.
 \label{main:eq:bandwidth}
\end{equation}
The first is the initial entropy per particle divided by the horizon; the second is the finite-particle self-interaction correction; the third is the quadrature error.

\paragraph{Bandwidth dependence.}
Suppose $\kappa_0$ and $a_1\Lambda_N(h)$ are uniformly bounded and $C_N(h)\le C_*h^{-\beta}$ for $0\le\beta<2$ throughout the admissible bandwidth range. Choosing $h\asymp N^{-1/(d+4-\beta)}$ gives
\begin{equation}
 \left(\frac1N\int_0^N\E_t[\mathsf V_N(X(t))]\dd t\right)^{1/2}
 =O\!\left(N^{-\frac{2-\beta}{2(d+4-\beta)}}\right).
 \label{main:eq:root-rate}
\end{equation}
The case $\beta=0$ yields $N^{-1/(d+4)}$. Similar to classical minimax bounds in nonparametric estimation~\citep{tsybakov2009introduction}, our bounds also deteriorate with the ambient dimension, so the curse of dimensionality is naturally expected here rather than merely an artifact of this upper-bound argument. The bandwidth calculation is given in Appendix~\ref{subsec:kde-optimized-rate}. 

The factor $h^2$ from the kernel second moment does not by itself make the quadrature error small. For example, derivative bounds at the kernel scale can give $B_{A,N}(h)+B_{V,N}(h)=O(h^{-4})$, producing an $O(h^{-2})$ term. The scalar Hessian assumptions must therefore be checked along the trajectory, including any cancellations between the target and model scores. If $\beta\ge2$, this bound does not yield a vanishing quadrature error by shrinking $h$.

\section{Reciprocal KDE Control and Initialization}
\label{main:sec:reciprocal}

Suppose $K(u)\ge\kappa_K>0$ for $\norm{u}\le r_K$. Let $M_i(x;h)$ count the centers within $r_Kh$ of $x_i$, including $i$. Proposition~\ref{prop:deterministic-occupancy} shows that
\[
 M_i(x;h)\ge\alpha Nh^d\ \text{for all }i
 \quad\Longrightarrow\quad
 q_x(x_i)\ge\kappa_K\alpha,\qquad
 \mathsf R_N(x)\le(\kappa_K\alpha)^{-1}.
\]
This controls denominators at particle centers without requiring a uniform density lower bound on $\R^d$.

\begin{proposition}[Independent samples]
\label{main:prop:iid}
Let $Y_1,\ldots,Y_N$ be i.i.d. from $\mu$, with $N\ge2$, and write $p_h(y)=\mu(B(y,r_Kh))$. Under the kernel lower bound above,
\begin{equation}
 \E[\mathsf R_N(Y)]
 \le\frac{h^d}{\kappa_K}\int
           \frac{1-(1-p_h(y))^N}{p_h(y)}\,\mu(\dd y).
 \label{main:eq:exact-reciprocal}
\end{equation}
The ratio is defined by its limit $N$ at $p_h(y)=0$. If $p_h(y)\ge p_0h^d$ for $\mu$-almost every $y$, then
\[
 \E[\mathsf R_N(Y)]\le\frac1{\kappa_Kp_0},\qquad
 \mathbb P\!\left(\mathsf R_N(Y)>\frac4{\kappa_Kp_0}\right)
 \le N\exp\!\left(-\frac{p_0Nh^d}{16}\right).
\]
\end{proposition}

The proof uses the exact conditional binomial identity for $1/M_i$, together with a lower-tail bound; see Proposition~\ref{prop:iid-occupancy} and \eqref{eq:reciprocal-tail-adaptive}. The expectation estimate does not require $Nh^d\gtrsim\log N$. That condition arises when simultaneous occupancy is required with high probability. A density bounded below on a sufficiently thick bounded support supplies the lower local mass condition; the corresponding example is in Appendix~\ref{sec:reciprocal-control}.

\paragraph{Propagation along the flow.}
If $\norm{b_{X(t)}(z)-b_{X(t)}(w)}\le L(t)\norm{z-w}$ with deterministic $\Gamma_T=\int_0^TL(t)\dd t$, initial neighbors within $r_Khe^{-\Gamma_T}$ remain within $r_Kh$. For i.i.d. initialization satisfying the lower local mass condition at bandwidth $he^{-\Gamma_T}$, Corollary~\ref{cor:iid-trajectory-reciprocal} gives
\begin{equation}
 \E\!\left[\sup_{0\le t\le T}\mathsf R_N(X(t))\right]
 \le\frac{e^{d\Gamma_T}}{\kappa_Kp_0}.
 \label{main:eq:propagation}
\end{equation}
Uniform distortion on a growing horizon is a separate stability condition. In particular, a Lipschitz bound of order $h^{-2}$ integrated over $[0,N]$ does not give a uniform right-hand side in \eqref{main:eq:propagation}.

\begin{proposition}[Full-support initialization]
\label{main:prop:full-support}
Let $Y_1,\ldots,Y_N$ be i.i.d. from a bounded continuous density $f$ positive on $\R^d$. Suppose $K$ is nonnegative and normalized, $0<K(0)<\infty$, $h_N\to0$, and $Nh_N^d\to\infty$. Then $\E[\mathsf R_N(Y)]\to\infty$.
\end{proposition}

The proof is Proposition~\ref{prop:full-support-reciprocal}. This includes Gaussian initialization. Since the time supremum includes time zero, $\Lambda_N(h_N)$ cannot be uniformly bounded in this regime. The finite-particle inequality remains valid with its actual reciprocal coefficient, whose growth must enter any bandwidth optimization.

\section{Non-conservative Drifting with the Laplace Kernel}
\label{main:sec:laplace}

Let $K_h(u)=c_dh^{-d}e^{-\norm{u}/h}$, normalized on $\R^d$. For $N\ge2$, set $\mu_{x,-i}^N=(N-1)^{-1}\sum_{j\ne i}\delta_{x_j}$. We analyze the self-masked dynamics
\begin{equation}
 \dot X_i(t)=u_{X(t),-i}(X_i(t)),\qquad
 u_{x,-i}=M_{\nu,h}-M_{\mu_{x,-i}^N,h},\qquad
 u_x=M_{\nu,h}-M_{\mu_x^N,h}.
 \label{main:eq:laplace-ode}
\end{equation}
Self-masking agrees with the diagonal masking in the original drifting implementation~\citep[Algorithm 2]{deng2026drifting}. It also removes the model's dependence on $x_i$ when differentiating $u_{x,-i}(z)$ at $z=x_i$.

\subsection{Sharp scores and the local scale}

Define the companion kernel and its convolutions by
\begin{equation}
 L_h(u)=h(\norm{u}+h)K_h(u),\qquad
 Q_{\alpha,h}=K_h*\alpha,\qquad R_{\alpha,h}=L_h*\alpha.
 \label{main:eq:sharp}
\end{equation}
Write $\sigma_{\alpha,h}=\nabla\log R_{\alpha,h}$ and $a_{\alpha,h}=R_{\alpha,h}/Q_{\alpha,h}$. Lemmas~\ref{lem:laplace-sharp-score} and~\ref{lem:sharp-kernel-moments} establish
\begin{equation}
 \nabla L_h(u)=-uK_h(u),\quad
 M_{\alpha,h}=a_{\alpha,h}\sigma_{\alpha,h},\quad
 Z_{\#,h}:=\int L_h=(d+1)h^2,\quad
 \norm{\sigma_{\alpha,h}}_\infty\le h^{-1}.
 \label{main:eq:sharp-identities}
\end{equation}
Let $\rho_{\alpha,h}^{\#}=R_{\alpha,h}/Z_{\#,h}$, and use the subscript $x$ for $\alpha=\mu_x^N$. The full field decomposes as
\begin{equation}
 u_x=a_{x,h}b_x^\#+e_x,\qquad
 b_x^\#=\sigma_{\nu,h}-\sigma_{x,h},\qquad
 e_x=(a_{\nu,h}-a_{x,h})\sigma_{\nu,h}.
 \label{main:eq:decomposition}
\end{equation}
The scale is $a_{\alpha,h}=h(\bar r_{\alpha,h}+h)$, where $\bar r_{\alpha,h}$ is the kernel-weighted mean distance from $z$ to $y$. It is at least $h^2$. For a finite configuration with $R_x=\max_j\norm{x_j}$,
\[
 h(\norm{z}-R_x+h)\le a_{x,h}(z)\le h(\norm{z}+R_x+h).
\]
Thus no finite global upper bound is available. We retain the scale inside the energies
\begin{equation}
 \mathcal W_h^{\mathrm{Lap}}(x)=\int\frac{\norm{u_x}^2}{a_{x,h}}\rho_{x,h}^\#,
 \qquad
 \mathcal D_h^{\mathrm{Lap}}(x)=\int\frac{\norm{e_x}^2}{a_{x,h}}\rho_{x,h}^\#.
 \label{main:eq:weighted}
\end{equation}

\subsection{Entropy, coercivity, and particle rates}

Set $\mathcal A_\#f=\divergence f+\sigma_{\nu,h}\cdot f$, and define
\begin{align}
 \mathsf S_N^{\mathrm{Lap}}(x)&=\frac1N\sum_i\mathcal A_\#u_{x,-i}(x_i),
 &\mathcal J_h^{\mathrm{Lap}}(x)&=\int\mathcal A_\#u_x\,\rho_{x,h}^\#,\notag\\*
 \mathsf V_N^{\mathrm{Lap}}(x)&=\frac1N\sum_i\norm{u_{x,-i}(x_i)}^2,
 &\mathcal V_h^{\mathrm{Lap}}(x)&=\int\norm{u_x}^2\rho_{x,h}^\#.
 \label{main:eq:laplace-quantities}
\end{align}
The full field is locally Lipschitz but may be nondifferentiable at particle centers. Its divergence in $\mathcal J_h^{\mathrm{Lap}}$ is a weak derivative, represented almost everywhere. Classical point evaluations are used only for the leave-one-out fields.

\begin{assumption}[Laplace regularity and approximation]
\label{main:ass:laplace}
The regular-flow and integration conditions in Assumption~\ref{ass:laplace-regularity} hold on $[0,T]$. In particular, entropy differentiation is justified, the joint entropy is absolutely continuous, spatial integrations by parts are valid for the weak derivatives, and each $u_{x,-i}$ is classically differentiable at $x_i$. Uniformly over visited configurations, all energies are finite and
\begin{align*}
 \mathcal D_h^{\mathrm{Lap}}(x)&\le\Delta_h^2,\\
 \abs{\mathsf S_N^{\mathrm{Lap}}(x)-\mathcal J_h^{\mathrm{Lap}}(x)}&\le\varepsilon_{S,h,N},\qquad
 \abs{\mathsf V_N^{\mathrm{Lap}}(x)-\mathcal V_h^{\mathrm{Lap}}(x)}\le\varepsilon_{V,h,N}.
\end{align*}
For the unweighted bound, also assume there is a measurable $\Omega_x$ with
\[
 a_{x,h}(z)\le L_h\quad(z\in\Omega_x),\qquad
 \int_{\Omega_x^c}\norm{u_x}^2\rho_{x,h}^\#\le\tau_{h,N}.
\]
Here the scalar $L_h>0$ and the other constants are independent of configuration and time on $[0,T]$; they may depend on $N,h,T$. The notation $L_h$ without a spatial argument denotes this scalar bound.
\end{assumption}

Write $H_N^\#(t)=\KL(p_t^N\|(\rho_{\nu,h}^\#)^{\otimes N})$. Theorems and lemmas in Appendix~\ref{sec:nonconservative-laplace-drifting} give the two identities and the coercivity bound
\begin{align}
 \frac{\dd}{\dd t}H_N^\#(t)&=-N\E_t[\mathsf S_N^{\mathrm{Lap}}(X(t))],\notag\\
 \mathcal J_h^{\mathrm{Lap}}(x)
 &=\int b_x^\#\cdot u_x\,\rho_{x,h}^\#
 \ge\frac12\mathcal W_h^{\mathrm{Lap}}(x)-\frac12\mathcal D_h^{\mathrm{Lap}}(x).
 \label{main:eq:laplace-entropy-coercivity}
\end{align}
There is no reciprocal self-interaction correction in this entropy identity. The scale mismatch enters the coercivity comparison instead.

\begin{theorem}[Laplace residual bounds]
\label{main:thm:laplace}
Under Assumption~\ref{main:ass:laplace},
\begin{align}
 \frac1T\int_0^T\E_t[\mathcal W_h^{\mathrm{Lap}}(X(t))]\dd t
 &\le\frac{2H_N^\#(0)}{NT}+\Delta_h^2+2\varepsilon_{S,h,N},
 \label{main:eq:laplace-weighted}\\
 \frac1T\int_0^T\E_t[\mathsf V_N^{\mathrm{Lap}}(X(t))]\dd t
 &\le\frac{2L_hH_N^\#(0)}{NT}+L_h\Delta_h^2 +2L_h\varepsilon_{S,h,N}+\tau_{h,N}+\varepsilon_{V,h,N}.
 \label{main:eq:laplace-rate}
\end{align}
The first bound does not require localization. If the absolute mismatch bound is replaced by $\mathcal D_h^{\mathrm{Lap}}\le\theta_h\mathcal W_h^{\mathrm{Lap}}$ with $0\le\theta_h<1$, then
\begin{equation}
 \frac1T\int_0^T\E_t[\mathsf V_N^{\mathrm{Lap}}(X(t))]\dd t
 \le\frac{L_h}{1-\sqrt{\theta_h}}
       \left\{\frac{H_N^\#(0)}{NT}+\varepsilon_{S,h,N}\right\}
       +\tau_{h,N}+\varepsilon_{V,h,N}.
 \label{main:eq:laplace-relative}
\end{equation}
\end{theorem}

The proofs are Theorem~\ref{thm:laplace-main-rate} and Corollary~\ref{cor:laplace-relative-rate}. Absolute alignment gives an additive upper bound on the residual; it does not imply a positive limiting error floor. Relative alignment absorbs the mismatch into the weighted dissipation with coefficient $1-\sqrt{\theta_h}$.

\subsection{Checking the Laplace assumptions}

The supporting estimates separate alignment, tails, and approximation. If $\norm{\sigma_{\nu,h}}_\infty\le G_h$ and $\int|\bar r_{\nu,h}-\bar r_{x,h}|^2\rho_{x,h}^\#\le\delta_{r,h}^2$, then $\Delta_h^2=G_h^2\delta_{r,h}^2$ is admissible. The universal choice is $G_h=h^{-1}$; additional target regularity may give a smaller value.

For data supported in $B(0,R_\nu)$ and particles confined to $B(0,R_T)$, Proposition~\ref{prop:laplace-confinement} gives, for every $r>0$,
\begin{equation}
 \Omega_x=B(0,R_T+r),\qquad
 L_h=h(2R_T+r+h),\qquad
 \tau_{h,N}=(R_\nu+R_T)^2p_d(r/h),
 \label{main:eq:laplace-tails}
\end{equation}
where the exact companion-kernel tail is
\[
 p_d(s)=e^{-s}\left\{
  \frac d{d+1}\sum_{k=0}^{d}\frac{s^k}{k!}
  +\frac1{d+1}\sum_{k=0}^{d-1}\frac{s^k}{k!}\right\}.
\]
Taking $r=ch\log N$ yields a tail proportional to $(1+(\log N)^d)N^{-c}$, with the confinement factor retained. A bounded initial cloud gives $R_T\le(R_0+R_\nu)e^T-R_\nu$ on a regular finite horizon. Uniform control on $T=N$ requires a stronger confinement estimate.

For quadrature, let $k_i(z)=L_h(z-x_i)/Z_{\#,h}$ and $d_i=u_{x,-i}-u_x$. Let $q_S$ and $q_V$ be the average absolute errors between the value at $x_i$ and the $k_i$-integral of $\mathcal A_\#u_{x,-i}$ and $\norm{u_{x,-i}}^2$, respectively. Lemma~\ref{lem:laplace-quadrature-sufficient} gives
\begin{align}
 \abs{\mathsf S_N^{\mathrm{Lap}}-\mathcal J_h^{\mathrm{Lap}}}
 &\le q_S+\frac2{Nh}\sum_i\int k_i\norm{d_i},\notag\\
 \abs{\mathsf V_N^{\mathrm{Lap}}-\mathcal V_h^{\mathrm{Lap}}}
 &\le q_V+\frac1N\sum_i\int k_i
       (\norm{u_{x,-i}}+\norm{u_x})\norm{d_i}.
 \label{main:eq:integrated-comparison}
\end{align}
The comparison uses integration by parts against each $k_i$ and differences of field values. It never differentiates the full field at a particle center. Local Taylor estimates for $q_S,q_V$ apply on balls where the leave-one-out integrands are smooth; the complementary integrals are bounded explicitly. With finite data support, these balls must avoid the other centers and all data atoms. Their radii, local Hessian bounds, and integrated leave-one-out errors all enter the final constants. Center-denominator control alone does not ensure a uniform quadrature rate.

For $T=N$ and $H_N^\#(0)\le\kappa_0N$, the right side of \eqref{main:eq:laplace-rate} is
\begin{equation}
 \mathcal R_{N,h}
 :=\frac{2\kappa_0L_h}{N}+L_h\Delta_h^2
      +2L_h\varepsilon_{S,h,N}+\tau_{h,N}+\varepsilon_{V,h,N}.
 \label{main:eq:laplace-bandwidth}
\end{equation}
A bandwidth sequence gives a vanishing root residual if $\kappa_0$ stays bounded and $L_h/N$, $L_h\Delta_h^2$, $L_h\varepsilon_{S,h,N}$, and $\tau_{h,N}+\varepsilon_{V,h,N}$ tend to zero, with the assumptions valid throughout $[0,N]$. The condition $Nh^d\gtrsim\log N$ can support local occupancy, but does not by itself establish these requirements.

\section{Numerical Experiments}
\label{main:sec:experiments}

We use low-dimensional simulations to isolate two mechanisms that are specific
to the theory: the reciprocal self-interaction correction in the conservative
entropy identity, and the scale-alignment coefficient in the Laplace
coercivity estimate.  The experiments are diagnostic rather than generative
benchmarks.  Target scores and kernel sums are evaluated analytically or as
finite sums, and uncertainty is computed across independent particle clouds;
particles within one interacting cloud are not treated as independent
replicates.

\subsection{Conservative correction and reciprocal tails}

\paragraph{Aim and setup.}
The first experiment asks whether the correction in
\eqref{main:eq:entropy} is numerically necessary and whether the reciprocal-KDE
assumptions distinguish bounded- from full-support initialization.  We use
the Gaussian conservative field with analytic one- and two-dimensional
Gaussian-mixture targets.  For the identity check, we take
$d\in\{1,2\}$, $N\in\{8,16,32\}$,
$h\in\{0.15,0.30,0.60\}$, and five independent clouds from each of a
variance-matched uniform, Gaussian, and Student-$t_5$ law.  At times
$0,0.05,0.25$ along the conservative ODE, coordinatewise central differences give
 an independent configuration divergence and hence the pathwise derivative $\widehat{\dot G}
 =-\widehat{\nabla_x\!\cdot F}-\sum_i s_{\rho,h}(X_i)\cdot F_i.$ We compare its error under the naive prediction $-N\mathsf S_N$ and the
corrected prediction
$-N\mathsf S_N-\Delta K_h(0)\mathsf R_N$.  This gives 810 derivative
comparisons.  Separately, at
$h_N=N^{-1/(d+4)}$, we estimate $\E\mathsf R_N$ up to $N=8192$ in $d=1$
and $N=4096$ in $d=2$.  Each cell uses 250 clouds through $N=1024$, 100
through $N=4096$, and 40 at $N=8192$; additional runs use bandwidth
constants $c\in\{0.5,2\}$.

\begin{figure}[t]
 \centering
 \includegraphics[width=\linewidth]{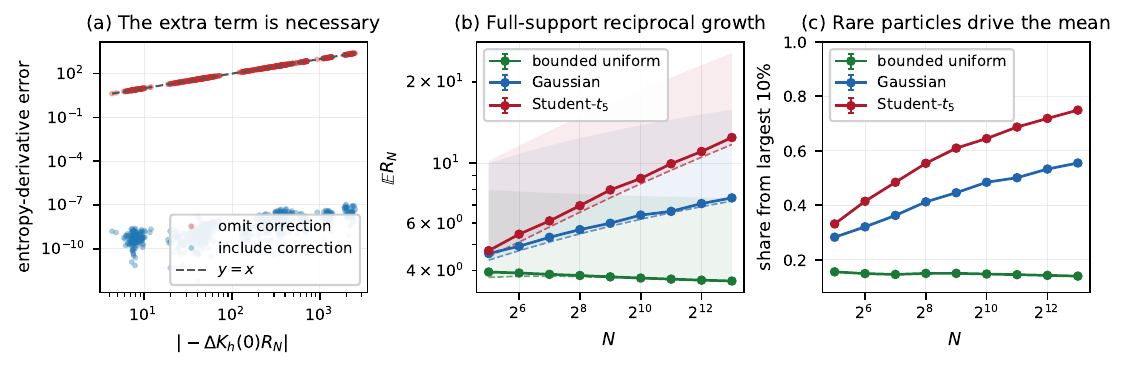}
 \caption{\textbf{The conservative correction and reciprocal-KDE obstruction.}
 (a) Error in the independently evaluated pathwise entropy derivative,
 plotted against the magnitude of the reciprocal correction.  Omitting the
 correction follows the identity line; including it leaves only
 finite-difference error.  (b) In $d=1$ with
 $h_N=N^{-1/5}$, the mean reciprocal KDE remains bounded for uniform
 initialization but grows for the two full-support laws.  Error bars are
 95\% Monte Carlo intervals; shaded regions run from the Jensen lower bound
 to the exact expectation upper bound, with the lower bound dashed.
 (c) Fraction of $\mathsf R_N$ contributed by the largest 10\% of particle
 reciprocals.}
 \label{fig:correction-obstruction}
\end{figure}

\paragraph{Results and interpretation.}
Figure~\ref{fig:correction-obstruction}(a) shows that the naive error equals
the omitted correction to relative error below $4\times10^{-10}$ in every
comparison.  With the correction included, the median error relative to the
correction is $2.3\times10^{-11}$ and the maximum is
$3.8\times10^{-10}$.  The correction is not a negligible numerical detail:
its median magnitude is $0.72$ times that of the Stein term
$-N\mathsf S_N$, with 10th--90th percentile range $0.25$--$2.67$.

In $d=1$, increasing $N$ from 32 to 8192 changes
$\E\mathsf R_N$ from $3.95$ to $3.66$ for bounded uniform
initialization, from $4.62$ to $7.43$ for Gaussian initialization, and from
$4.74$ to $12.45$ for Student-$t_5$.  At the largest $N$, the most isolated
10\% of particles contribute respectively 14\%, 56\%, and 75\% of the
sum.  The same separation occurs in $d=2$, and every measured mean lies
between the evaluated lower and upper bounds.  Thus the self-interaction term
is required even at moderate $N$, while the full-support obstruction is
carried by increasingly rare low-density particles, precisely the behavior
that a median diagnostic would conceal.

\subsection{Laplace scale-alignment phase diagram}

\paragraph{Aim and setup.}
The second experiment asks whether the observable ratio
$\theta=\mathcal D_h^{\mathrm{Lap}}/\mathcal W_h^{\mathrm{Lap}}$
predicts the strength of the Laplace coercivity certificate.  The target is
eight fixed atoms in $\R^2$, so all target-side Laplace quantities are exact.
We draw $N=128$ particles from
$\mathcal N(m,s^2I_2)$, with five shifts $ m\in\{(0,0),(1.5,0),(3,3),(-4.5,1),(7,0)\},$ five spreads $s\in\{0.2,0.4,1,2.5,5\}$, three bandwidths
$h\in\{0.15,0.30,0.60\}$, and six seeds, giving 450 configurations.
For each, approximately 24,000 stratified exact draws from
$\rho_{x,h}^{\#}$ estimate $\mathcal W_h^{\mathrm{Lap}}$,
$\mathcal D_h^{\mathrm{Lap}}$, and $\mathcal J_h^{\mathrm{Lap}}$.
We also evolve four representative initialization families under the ODE
induced by the exact leave-one-out Laplace field, using five seeds at each
bandwidth.  The
confinement radius $r=ch\log N$ used in the accompanying velocity diagnostics
is optimized over $c\in[0.15,6]$ rather than fixed in advance.

\begin{figure}[t]
 \centering
 \includegraphics[width=\linewidth]{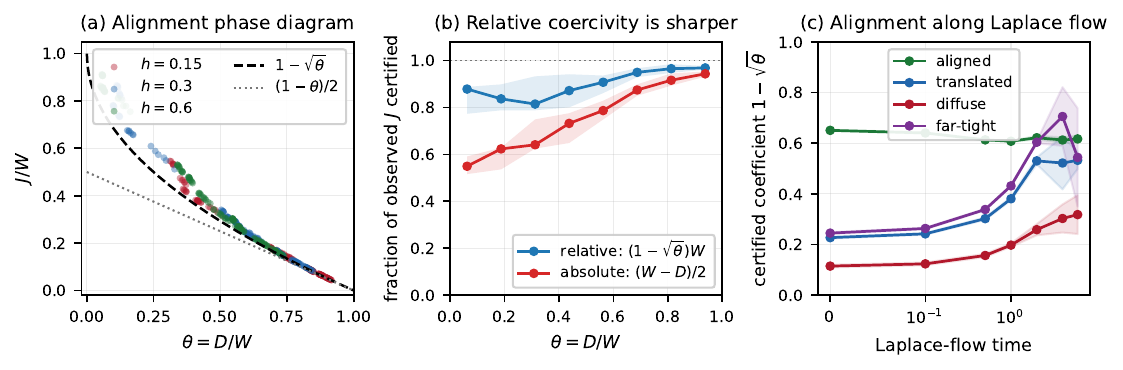}
 \caption{\textbf{Weighted scale alignment for exact Laplace drifting.}
 (a) The measured ratio
 $\mathcal J_h^{\mathrm{Lap}}/\mathcal W_h^{\mathrm{Lap}}$ against
 $\theta=\mathcal D_h^{\mathrm{Lap}}/\mathcal W_h^{\mathrm{Lap}}$.
 The dashed curve is the relative certificate $1-\sqrt{\theta}$ and the
 dotted curve is the absolute certificate $(1-\theta)/2$.
 (b) Fraction of the observed $\mathcal J_h^{\mathrm{Lap}}$ certified by
 each lower bound.  Lines are binwise medians and bands are 10th--90th
 percentiles.  (c) The relative coefficient along selected Laplace flows at
 $h=0.30$; lines and bands are means and 95\% intervals across five
 independent clouds.}
 \label{fig:laplace-alignment}
\end{figure}

\paragraph{Results and interpretation.}
The configurations span $\theta=0.051$ to $0.917$.
Every measured point in Figure~\ref{fig:laplace-alignment}(a) lies above
$1-\sqrt{\theta}$; the smallest observed normalized slack is
$8.6\times10^{-4}$.  More importantly, the relative lower bound certifies a
median 93.3\% of the observed $\mathcal J_h^{\mathrm{Lap}}$
(10th--90th percentiles 81.4\%--97.1\%), compared with 83.2\%
(55.6\%--93.2\%) for $(\mathcal W-\mathcal D)/2$.

The trajectory experiment shows how the diagnostic behaves on states actually
visited by the particle system.  At $h=0.30$, the coefficient
$1-\sqrt{\theta}$ rises from $0.23$ to $0.53$ for translated clouds, from
$0.12$ to $0.32$ for diffuse clouds, and from $0.25$ to $0.54$ for
far but tight clouds; an initially well-aligned centered cloud remains near
$0.6$.  This is not a claim that alignment is monotone for every trajectory.
Rather, it shows that the scale-mismatch term is quantitatively active and
that $\theta$ separates strongly and weakly coercive regimes along the
exact-field dynamics.  Dropping $\mathcal D_h^{\mathrm{Lap}}$ would miss this distinction,
whereas the relative theory gives a nearly sharp pointwise certificate over
the tested range.

\section{Conclusion and Limitations}
\label{main:sec:discussion}

We developed finite-particle convergence guarantees for conservative kernel-score drifting and the original non-conservative Laplace displacement field, showing that conservativeness alone does not determine performance. For conservative dynamics, we control the smoothed score discrepancy and particle velocity and identify an exact self-interaction correction involving reciprocal kernel-density estimates: local particle occupancy keeps it finite, but full-support initializations can make its expectation grow as bandwidth shrinks. Although circulation-free, this field has a distribution-dependent potential that need not improve monotonically along individual trajectories. The leave-one-out Laplace field avoids the correction, but its position-dependent normalization introduces circulation and scale mismatch; a companion-kernel representation nevertheless controls velocity despite these scale issues, the kernel cusp, and unbounded empirical scales. Simulations confirm that the conservative correction explains the observed entropy-identity defect and that relative scale alignment sharpens the Laplace certificate. Together, these results identify the quantities governing each field’s distinct failure modes and clarify when each guarantee is informative.

Our guarantees are time-averaged and conditional on density, quadrature, and
stability assumptions; they do not yet establish global distributional
convergence or a terminal-time guarantee.  The analysis also idealizes
training as an exact continuous-time particle system, while practical neural
network training adds discretization, approximation, optimization, and
minibatch errors.  Future work should connect the entropy
bounds to terminal convergence, extend the analysis to finite-step learned
generators and broader kernels, develop adaptive methods for low-density
particles and scale mismatch, and validate these mechanisms on realistic
high-dimensional data.

\label{main:lastpage}
\subsection*{AI use statement}
Generative AI tools (ChatGPT) were used to check mathematical arguments, identify proof gaps, assist in formulating and writing revisions to the Laplace analysis and reciprocal-KDE bounds, and edit and format the manuscript. They also assisted with literature searches and preparation of the LaTeX source. Responsibility for the mathematical claims, citations, and final manuscript rests with the authors.

\bibliography{ref}

\begin{thebibliography}{31}
\providecommand{\natexlab}[1]{#1}
\providecommand{\url}[1]{\texttt{#1}}
\expandafter\ifx\csname urlstyle\endcsname\relax
  \providecommand{\doi}[1]{doi: #1}\else
  \providecommand{\doi}{doi: \begingroup \urlstyle{rm}\Url}\fi

\bibitem[Albergo and Vanden-Eijnden(2023)]{albergo2023building}
Michael~Samuel Albergo and Eric Vanden-Eijnden.
\newblock {Building Normalizing Flows with Stochastic Interpolants}.
\newblock In \emph{The Eleventh International Conference on Learning
  Representations}, 2023.
\newblock URL \url{https://openreview.net/forum?id=li7qeBbCR1t}.

\bibitem[Banerjee et~al.(2025)Banerjee, Balasubramanian, and
  Ghosal]{balasubramanian2024improved}
Sayan Banerjee, Krishnakumar Balasubramanian, and Promit Ghosal.
\newblock {Improved Finite-Particle Convergence Rates for Stein Variational
  Gradient Descent}.
\newblock In \emph{The Thirteenth International Conference on Learning
  Representations}, 2025.
\newblock URL \url{https://openreview.net/forum?id=sbG8qhMjkZ}.

\bibitem[Barbour(1988)]{barbour1988stein}
Andrew~D Barbour.
\newblock {S}tein's method and {P}oisson process convergence.
\newblock \emph{Journal of Applied Probability}, 25\penalty0 (A):\penalty0
  175--184, 1988.

\bibitem[Boffi et~al.(2025)Boffi, Albergo, and
  Vanden-Eijnden]{boffi2024flowmap}
Nicholas~Matthew Boffi, Michael~Samuel Albergo, and Eric Vanden-Eijnden.
\newblock {Flow map matching with stochastic interpolants: A mathematical
  framework for consistency models}.
\newblock \emph{Transactions on Machine Learning Research}, 2025.
\newblock ISSN 2835-8856.
\newblock URL \url{https://openreview.net/forum?id=cqDH0e6ak2}.

\bibitem[Cao et~al.(2026)Cao, Wei, and Liu]{cao2026gradient}
Jiarui Cao, Zixuan Wei, and Yuxin Liu.
\newblock {Gradient flow drifting: Generative modeling via Wasserstein gradient
  flows of KDE-approximated divergences}.
\newblock \emph{arXiv preprint arXiv:2603.10592}, 2026.

\bibitem[Deng et~al.(2026)Deng, Li, Li, Du, and He]{deng2026drifting}
Mingyang Deng, He~Li, Tianhong Li, Yilun Du, and Kaiming He.
\newblock {Generative modeling via drifting}.
\newblock \emph{arXiv preprint arXiv:2602.04770}, 2026.

\bibitem[Dumont et~al.(2026)Dumont, Lacombe, and Vialard]{dumont2026learning}
Th{\'e}o Dumont, Th{\'e}o Lacombe, and Fran{\c c}ois-Xavier Vialard.
\newblock {Learning Monge maps with constrained drifting models}.
\newblock \emph{arXiv preprint arXiv:2603.25182}, 2026.

\bibitem[Esteban-Casadevall et~al.(2026)Esteban-Casadevall, Carrasco-Pollo,
  Welling, van~de Meent, Bekkers, and Eijkelboom]{estebancasadevall2026kernel}
Maria Esteban-Casadevall, Jorge Carrasco-Pollo, Max Welling, Jan-Willem van~de
  Meent, Erik~J. Bekkers, and Floor Eijkelboom.
\newblock {Kernel-gradient drifting models}.
\newblock \emph{arXiv preprint arXiv:2605.10727}, 2026.

\bibitem[Feng et~al.(2017)Feng, Wang, and Liu]{feng2017amortized}
Yihao Feng, Dilin Wang, and Qiang Liu.
\newblock {Learning to draw samples with amortized Stein variational gradient
  descent}.
\newblock In \emph{The Conference on Uncertainty in Artificial Intelligence
  (UAI)}, 2017.

\bibitem[Franz et~al.(2026)Franz, Hoffmann, and Martius]{franz2026drifting}
Leonard Franz, Sebastian Hoffmann, and Georg Martius.
\newblock Drifting fields are not conservative.
\newblock \emph{arXiv preprint arXiv:2604.06333}, 2026.

\bibitem[Gorham and Mackey(2015)]{gorham2015measuring}
Jackson Gorham and Lester Mackey.
\newblock {Measuring sample quality with Stein's method}.
\newblock \emph{Advances in neural information processing systems}, 28, 2015.

\bibitem[Gretton et~al.(2026)Gretton, Wenliang, Galashov, Thornton, De~Bortoli,
  and Doucet]{gretton2026wasserstein}
Arthur Gretton, Li~Kevin Wenliang, Alexandre Galashov, James Thornton, Valentin
  De~Bortoli, and Arnaud Doucet.
\newblock {On the Wasserstein Gradient Flow Interpretation of Drifting Models}.
\newblock \emph{arXiv preprint arXiv:2605.05118}, 2026.

\bibitem[He et~al.(2026{\natexlab{a}})He, Khangaonkar, Pirsiavash, Bai, and
  Kolouri]{he2026sinkhorn}
Ping He, Om~Khangaonkar, Hamed Pirsiavash, Yikun Bai, and Soheil Kolouri.
\newblock {Sinkhorn-Drifting Generative Models}.
\newblock \emph{arXiv preprint arXiv:2603.12366}, 2026{\natexlab{a}}.

\bibitem[He et~al.(2026{\natexlab{b}})He, Balasubramanian, Banerjee, and
  Ghosal]{he2026finite}
Ye~He, Krishnakumar Balasubramanian, Sayan Banerjee, and Promit Ghosal.
\newblock {Finite-Particle Rates for Regularized Stein Variational Gradient
  Descent}.
\newblock \emph{arXiv preprint arXiv:2602.05172}, 2026{\natexlab{b}}.

\bibitem[Ho et~al.(2020)Ho, Jain, and Abbeel]{ho2020denoising}
Jonathan Ho, Ajay Jain, and Pieter Abbeel.
\newblock {Denoising Diffusion Probabilistic Models}.
\newblock In \emph{Advances in Neural Information Processing Systems},
  volume~33, pages 6840--6851, 2020.

\bibitem[Lai et~al.(2026)Lai, Nguyen, Murata, Takida, Uesaka, Mitsufuji, Ermon,
  and Tao]{lai2026unified}
Chieh-Hsin Lai, Bac Nguyen, Naoki Murata, Yuhta Takida, Toshimitsu Uesaka, Yuki
  Mitsufuji, Stefano Ermon, and Molei Tao.
\newblock A unified view of drifting and score-based models.
\newblock \emph{arXiv preprint arXiv:2603.07514}, 2026.

\bibitem[Lee and Chun(2026)]{lee2026identifiability}
Hak~Geun Lee and Hyonho Chun.
\newblock {Identifiability and Stability of Generative Drifting with
  Companion-Elliptic Kernel Families}.
\newblock \emph{arXiv preprint arXiv:2604.24196}, 2026.

\bibitem[Li and Zhu(2026)]{li2026longshort}
Zhiqi Li and Bo~Zhu.
\newblock {A Long-Short Flow-Map Perspective for Drifting Models}.
\newblock \emph{arXiv preprint arXiv:2602.20463}, 2026.

\bibitem[Lipman et~al.(2023)Lipman, Chen, Ben-Hamu, Nickel, and
  Le]{lipman2023flow}
Yaron Lipman, Ricky T.~Q. Chen, Heli Ben-Hamu, Maximilian Nickel, and Matthew
  Le.
\newblock {Flow Matching for Generative Modeling}.
\newblock In \emph{The Eleventh International Conference on Learning
  Representations}, 2023.
\newblock URL \url{https://openreview.net/forum?id=PqvMRDCJT9t}.

\bibitem[Liu(2017)]{liu2017stein}
Qiang Liu.
\newblock {Stein variational gradient descent as gradient flow}.
\newblock \emph{Advances in neural information processing systems}, 30, 2017.

\bibitem[Liu and Wang(2016)]{liu2016stein}
Qiang Liu and Dilin Wang.
\newblock {Stein variational gradient descent: A general purpose bayesian
  inference algorithm}.
\newblock \emph{Advances in neural information processing systems}, 29, 2016.

\bibitem[Liu et~al.(2023)Liu, Gong, and Liu]{liu2022flow}
Xingchao Liu, Chengyue Gong, and Qiang Liu.
\newblock {Flow Straight and Fast: Learning to Generate and Transfer Data with
  Rectified Flow}.
\newblock In \emph{The Eleventh International Conference on Learning
  Representations}, 2023.
\newblock URL \url{https://openreview.net/forum?id=XVjTT1nw5z}.

\bibitem[Penrose(2003)]{penrose2003random}
Mathew Penrose.
\newblock \emph{Random Geometric Graphs}.
\newblock Oxford University Press, 2003.

\bibitem[Salimans and Ho(2022)]{salimans2022progressive}
Tim Salimans and Jonathan Ho.
\newblock Progressive distillation for fast sampling of diffusion models.
\newblock In \emph{International Conference on Learning Representations}, 2022.
\newblock URL \url{https://openreview.net/forum?id=TIdIXIpzhoI}.

\bibitem[Shi and Mackey(2023)]{shi2023finite}
Jiaxin Shi and Lester Mackey.
\newblock {A finite-particle convergence rate for Stein variational gradient
  descent}.
\newblock \emph{Advances in Neural Information Processing Systems},
  36:\penalty0 26831--26844, 2023.

\bibitem[Song et~al.(2021)Song, Sohl-Dickstein, Kingma, Kumar, Ermon, and
  Poole]{song2021scorebased}
Yang Song, Jascha Sohl-Dickstein, Diederik~P Kingma, Abhishek Kumar, Stefano
  Ermon, and Ben Poole.
\newblock {Score-Based Generative Modeling through Stochastic Differential
  Equations}.
\newblock In \emph{International Conference on Learning Representations}, 2021.
\newblock URL \url{https://openreview.net/forum?id=PxTIG12RRHS}.

\bibitem[Song et~al.(2023)Song, Dhariwal, Chen, and
  Sutskever]{song2023consistency}
Yang Song, Prafulla Dhariwal, Mark Chen, and Ilya Sutskever.
\newblock Consistency models.
\newblock In \emph{International Conference on Machine Learning}, 2023.

\bibitem[Tsybakov(2009)]{tsybakov2009introduction}
Alexandre~B. Tsybakov.
\newblock \emph{Introduction to Nonparametric Estimation}.
\newblock Springer Series in Statistics. Springer, New York, 2009.

\bibitem[Turan and Ovsjanikov(2026)]{turan2026generative}
Erkan Turan and Maks Ovsjanikov.
\newblock {Generative drifting is secretly score matching: A spectral and
  variational perspective}.
\newblock \emph{arXiv preprint arXiv:2603.09936}, 2026.

\bibitem[Weber(2023)]{weber2023the}
Romann~M. Weber.
\newblock {The Score-Difference Flow for Implicit Generative Modeling}.
\newblock \emph{Transactions on Machine Learning Research}, 2023.
\newblock ISSN 2835-8856.
\newblock URL \url{https://openreview.net/forum?id=dpGSNLUCzu}.

\bibitem[Zhang et~al.(2026)Zhang, Niwa, and Kleijn]{zhang2026lookahead}
Guoqiang Zhang, Kenta Niwa, and W.~Bastiaan Kleijn.
\newblock Lookahead drifting model.
\newblock \emph{arXiv preprint arXiv:2605.04060}, 2026.

\end{thebibliography}
\bibliographystyle{plainnat}
\clearpage
\appendix

\section{Conservative versus Non-conservative Drift: A Visual Comparison}
\label{subsec:conservative-vs-non}

To build intuition for the distinction between the conservative and non-conservative drifting fields, consider a simple two-dimensional setup in which the data distribution \(\nu\) is supported on a thin horizontal cluster along the \(x\)-axis and the model distribution \(\mu\) is initialized on a thin vertical cluster along the \(y\)-axis. We run the finite-particle dynamics from the same initial configuration under each of the two velocity fields, $  b_{\nu,\mu,h}(z)$ with a Gaussian kernel for the conservative case, and $u_{\nu,\mu,h}^{\mathrm{Lap}}(z):=
  M_{\nu,h}(z)-M_{\mu,h}(z)$ with a Laplace kernel for the non-conservative case, as in~\cite{deng2026drifting}. We track four model particles whose initial positions are identical across the two simulations.

Figure~\ref{fig:xy-paths} shows the resulting trajectories. Both fields
transport the displayed model cloud toward the data cluster, and the four
pairs of paths share their starting points. At time $t$, the conservative
field is the spatial gradient of
\[
  \Phi_t(z):=\log\rho_{\nu,h}(z)-\log\rho_{\mu_t,h}(z),
  \qquad b_{\nu,\mu_t,h}(z)=\nabla_z\Phi_t(z).
\]
Thus its circulation around a closed curve is zero whenever the spatial
derivatives are defined. The potential itself evolves with $\mu_t$.
Along a trajectory $\dot Z(t)=\nabla_z\Phi_t(Z(t))$, the chain rule gives
\[
  \frac{\dd}{\dd t}\Phi_t(Z(t))
  =\partial_t\Phi_t(Z(t))+\norm{\nabla_z\Phi_t(Z(t))}^2.
\]
The sign of the full time derivative therefore depends on the changing
model density as well as on the instantaneous drift.

For the Laplace field, the position-dependent normalization can produce
nonzero circulation. Curvature of a trajectory by itself does not
distinguish the two fields: a spatial gradient field can also have curved
integral curves. The algebraic distinction is captured by the local scale
$a_{\alpha,h}=R_{\alpha,h}/Q_{\alpha,h}$ in
Lemma~\ref{lem:laplace-sharp-score}. Its spatial variation enters the curl,
while its difference between the data and model enters the scale-mismatch
term. The weighted and localized estimates in
Theorem~\ref{thm:laplace-main-rate} keep this scale dependence explicit;
see Remark~\ref{rem:laplace-curl-and-delta} for the circulation identity.

\subsection{Helmholtz--Hodge decomposition of the initial fields}
\label{subsec:figure-one-hodge}

We test the distinction directly in a deterministic replicate of the
Figure~\ref{fig:xy-paths} setting.  The replicate uses 400 data atoms in a
thin horizontal cloud and the same 80 model particles in a thin vertical
cloud for both fields.  The long data and model coordinates are sampled
uniformly on \([-2.85,2.85]\) and \([-2.65,2.65]\), respectively; the
transverse noise is Gaussian with standard deviation \(0.065\), and each
cloud is recentered.  We use seed 1701, bandwidth \(h=0.55\), and a
\(101\times101\) node grid over \([-3.4,3.4]^2\).

Because a finite-domain Helmholtz decomposition depends on its boundary
condition, we avoid a periodic Fourier projection.  Instead, we sample the
component tangent to each grid edge and use the orthogonal graph-Hodge
decomposition
\[
  f=G\phi+r,\qquad G^\top r=0,
\]
where \(G\) is the node-to-edge gradient.  The first term is a discrete
curl-free field, and the residual is discrete divergence-free under the
natural grid boundary condition.  This construction exactly reconstructs
the sampled field and separates circulation from the curvature of individual
trajectories.

\begin{figure}[t]
 \centering
 \includegraphics[width=\linewidth]{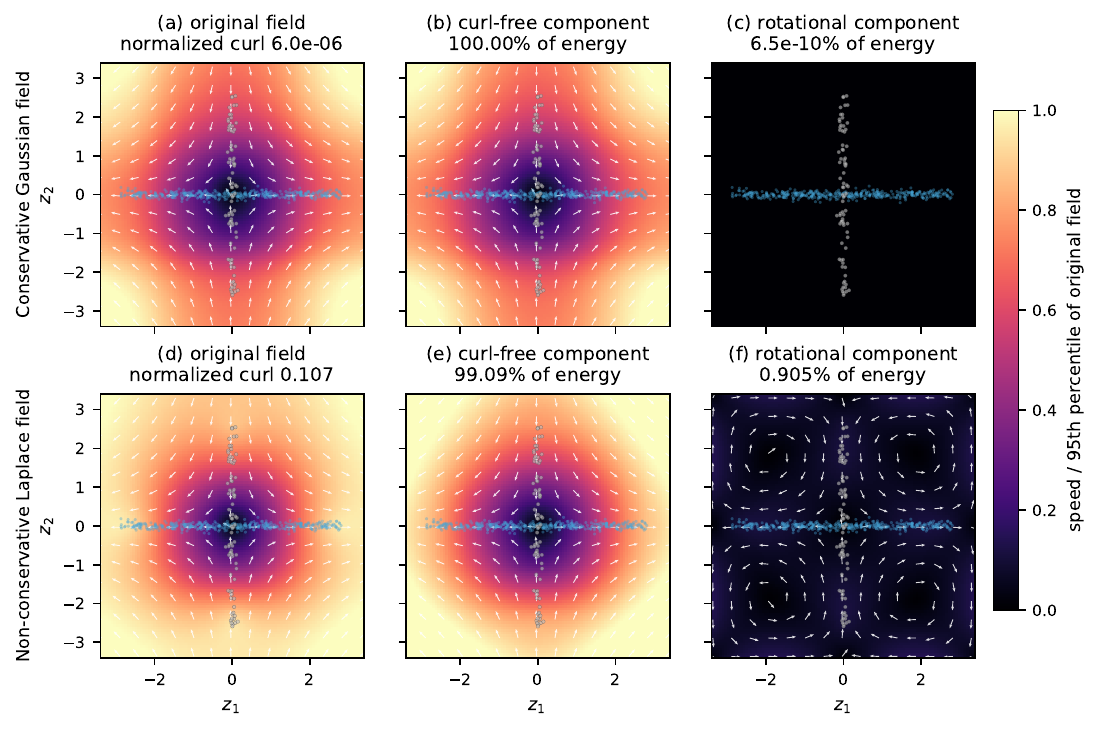}
 \caption{\textbf{Helmholtz--Hodge decomposition in the
 Figure~\ref{fig:xy-paths} setting.} Rows show the conservative Gaussian
 field and non-conservative Laplace field at initialization; columns show the
 original field, its curl-free projection, and its rotational residual.
 Color is component speed divided by the 95th percentile of the corresponding
 original field, so the magnitude scale is shared within each row.  Arrow
 lengths are normalized wherever the component exceeds 0.5\% of that
 reference and therefore show direction rather than relative magnitude.  The
 conservative residual contains \(6.5\times10^{-10}\%\) of the edge energy
 and has bandwidth-normalized root-mean-square curl \(6.0\times10^{-6}\),
 consistent with discretizing an exact gradient.  The Laplace residual
 contains \(0.905\%\) of the energy, has normalized curl \(0.107\), and forms
 four coherent circulation cells.  On a \(51\times51\) grid, the corresponding
 fractions are \(1.1\times10^{-8}\%\) and \(0.937\%\), respectively.}
 \label{fig:figure-one-hodge}
\end{figure}

Figure~\ref{fig:figure-one-hodge} makes two points.  First, the conservative
field is recovered almost entirely by the gradient projection, and its
remaining rotational energy decreases rapidly under grid refinement.
Second, the Laplace field is still predominantly gradient-like, but it has a
small, spatially organized rotational component that persists under
refinement.  Non-conservative therefore does not mean ``mostly rotational'';
it means that a nonzero circulation component remains after the best
curl-free projection.

\FloatBarrier
\section{Further Related Work}
\label{sec:related-work}

Drifting models were introduced by~\cite{deng2026drifting} as a training-time transport paradigm for one-step generation: generated samples are repeatedly moved by a data-attractive and model-repulsive drift field, and the generator is trained to reproduce the drifted samples through a stop-gradient target. This viewpoint is also related to score-difference flows for implicit generative modeling, where the velocity is written as the difference between target and source scores~\citep{weber2023the}. Subsequent work has connected drifting to score matching and Wasserstein gradient flows of smoothed divergences (in the infinite-particle setting), especially in the Gaussian-kernel case where the displacement field can be written as a difference of KDE log-density gradients~\citep{cao2026gradient,lai2026unified,turan2026generative,gretton2026wasserstein}. 

A complementary line of work emphasizes the non-conservative nature of the original displacement-based field. \cite{franz2026drifting} show that the position-dependent normalization in the original field generally destroys conservatism, with the Gaussian kernel being a special case where the field remains a gradient field; they also introduce sharp normalizations that restore conservatism for broader radial kernels.~\cite{lee2026identifiability} study identifiability and stability for companion-elliptic kernel families, including the Laplace kernel, showing that vanishing drift can identify the target under appropriate kernel structure while also highlighting possible instability modes. These works motivate our distinction between the conservative drifting method, based on KDE-score gradients, and the non-conservative drifting method by~\cite{deng2026drifting}, based on the original displacement velocity. For the latter, our analysis uses weighted coercivity to retain the local scale and controls the scale mismatch either additively or through relative alignment.

Recent drifting-related works study complementary algorithmic and variational perspectives. \citet{dumont2026learning} study learning Monge maps through constrained drifting models: they lift a divergence, such as relative entropy, to the space of transport maps, constrain the dynamics to the set of optimal transport maps, and prove long-time existence and convergence toward the OT map under convexity assumptions. \citet{he2026sinkhorn} introduce Sinkhorn-Drifting Generative Models, showing that the attraction--repulsion structure of drifting can be viewed as a one-sided surrogate for a Sinkhorn-divergence gradient flow, while two-sided Sinkhorn scaling gives improved identifiability and stability. \citet{zhang2026lookahead} propose Lookahead Drifting Models, which compute several sequential drift terms within each training iteration and regress toward a weighted sum of these terms, with the goal of incorporating higher-order drift information while preserving one-step inference.

More broadly, modern generative modeling is often organized around learned transport dynamics. Diffusion and score-based models learn reverse-time denoising dynamics and generate samples by numerically solving an SDE or probability-flow ODE \citep{ho2020denoising,song2021scorebased}, while flow matching and rectified flow learn time-dependent vector fields that transport a base distribution to the data distribution along prescribed probability paths \citep{albergo2023building, lipman2023flow,liu2022flow}. Flow-map and distillation-based approaches instead try to learn longer-time solution operators, or compress many solver steps into one or a few neural evaluations, as in flow-map matching, consistency models, and progressive distillation \citep{boffi2024flowmap,song2023consistency,salimans2022progressive}. From this viewpoint, drifting models can be interpreted as moving part of the inference-time transport computation into training: the long-short flow-map perspective of \citet{li2026longshort} decomposes a global transport into a long-horizon map followed by a short terminal correction, and recovers the drifting field in the limit of a vanishing terminal interval.

The proposed conservative velocity approach is also closely related to deterministic interacting-particle methods based on Stein discrepancies. Stein variational gradient descent (SVGD) was introduced by~\cite{liu2016stein} as a deterministic particle method for approximate Bayesian inference, and~\cite{liu2017stein} interpreted its population limit as a gradient flow of KL divergence under a Stein-induced geometry. Amortized SVGD~\citep{feng2017amortized} trains a sampler network to reproduce SVGD particle updates, which is conceptually close to the stop-gradient regression view of drifting. At the level of finite-particle analysis,~\cite{balasubramanian2024improved} obtain improved SVGD rates by differentiating the joint entropy of the particle law relative to the product target, improving upon earlier results by~\cite{shi2023finite}. Extension to a regularized setting was studied by~\cite{he2026finite}. We leverage and extend this high-level entropy-dissipation strategy in our proofs, but the controlled quantities and error terms are specific to drifting: for conservative drifting, the controlled quantity is the mean squared conservative residual with a reciprocal-KDE self-interaction and an $h$-dependent quadrature error; for the non-conservative Laplace method, the rate additionally contains the Laplace scale-mismatch residual.

\section{Detailed Set-up and Assumptions for Conservative Drifting}
\label{app:conservative-setup}

\subsection{Set-up for Conservative Drifting}

Let $\rho$ be a positive probability density on $\R^d$. For the conservative drifting method proposed in this paper, $\rho$ should be read as the kernel-smoothed data density $\rho_{\nu,h}=K_h\ast\nu$. Write
\[
  s_\rho(z) := \nabla \log \rho(z).
\]
Let $K:\R^d\to(0,\infty)$ be a normalized kernel and, for a bandwidth $h>0$, define
\[
  K_h(u) := h^{-d}K(u/h).
\]
For a particle configuration
\[
  x=(x_1,\ldots,x_N)\in(\R^d)^N,
\]
define the kernel density estimate
\[
  q_x(z) := \frac1N\sum_{j=1}^N K_h(z-x_j),
  \qquad
  s_x(z) := \nabla \log q_x(z),
\]
and the conservative drift field
\[
  b_x(z) := s_\rho(z)-s_x(z).
\]
The center-evaluation finite-particle ODE for the conservative drifting method is
\begin{equation}
  \dot X_i(t)=b_i(X(t)),
  \qquad
  b_i(x):=b_x(x_i),
  \qquad
  i=1,\ldots,N.
  \label{eq:particle-ode}
\end{equation}

This is a deterministic interacting particle system: each particle follows an
ordinary differential equation, but its velocity depends on the full
configuration through the KDE and its score. In this sense, the dynamics are structurally similar to SVGD
\citep{liu2016stein,liu2017stein}, where particles also evolve deterministically
through an empirical-measure-dependent velocity field; the key difference is
that the conservative drifting velocity is the KDE-score residual \(s_\rho-s_x\), rather than the
kernelized Stein velocity used in SVGD.

Let $p_t^N$ denote the joint density of $X(t)=(X_1(t),\ldots,X_N(t))$. We measure the joint law against the product target density
\[
  \rho^{\otimes N}(x):=\prod_{i=1}^N \rho(x_i)
\]
through the joint relative entropy
\begin{equation}
  H_N(t) := \KL(p_t^N\|\rho^{\otimes N})
  =\int_{(\R^d)^N}p_t^N(x)
  \log\frac{p_t^N(x)}{\rho^{\otimes N}(x)}\dd x.
  \label{eq:joint-entropy}
\end{equation}

For a vector field $f:\R^d\to\R^d$, define the Stein divergence~\citep{barbour1988stein,gorham2015measuring} associated with $\rho$ by
\begin{equation}
  \mathcal A_\rho f(z)
  :=\divergence f(z)+s_\rho(z)\cdot f(z).
  \label{eq:stein-divergence}
\end{equation}
For a configuration $x$, define the empirical Stein drift
\begin{equation}
  \mathsf S_N(x)
  :=\frac1N\sum_{i=1}^N \mathcal A_\rho b_x(x_i),
  \label{eq:emp-stein-drift}
\end{equation}
the KDE-smoothed Fisher discrepancy
\begin{equation}
  \mathsf I_N(x)
  :=\int_{\R^d}\norm{b_x(z)}^2q_x(z)\dd z,
  \label{eq:fisher}
\end{equation}
and the squared center velocity
\begin{equation}
  \mathsf V_N(x)
  :=\frac1N\sum_{i=1}^N\norm{b_x(x_i)}^2.
  \label{eq:center-velocity}
\end{equation}
Finally set
\begin{align}
  \mathsf R_N(x)
  :=\frac1N\sum_{i=1}^N\frac1{q_x(x_i)}.
  \label{eq:reciprocal-kde}
\end{align}

\subsection{Assumptions}

The following assumptions are used in the entropy calculation and in the conversion from empirical Stein drift to Fisher-type discrepancies.

\begin{assumption}[Kernel]
\label{ass:kernel}
The base kernel $K$ is positive, even, normalized, and $C^3$:
\[
  K(u)>0,
  \qquad
  K(u)=K(-u),
  \qquad
  \int_{\R^d}K(u)\dd u=1.
\]
Moreover
\[
  m_2(K):=\int_{\R^d}\norm{u}^2K(u)\dd u<\infty,
\]
and $\Delta K(0)$ is finite. Since $K$ is even, $\nabla K(0)=0$.
\end{assumption}

\begin{assumption}[Target and regular flow]
\label{ass:regularity}
The density $\rho$ is positive and sufficiently smooth so that $s_\rho$ is $C^2$. The ODE \eqref{eq:particle-ode} has a unique global solution for the initial law under consideration. The joint density $p_t^N$ is differentiable in $t$ and solves the Liouville equation
\[
  \partial_t p_t^N(x)+\sum_{i=1}^N\nabla_{x_i}\cdot\{p_t^N(x)b_i(x)\}=0.
\]
All integrals appearing below are finite, and the boundary terms in the integrations by parts on $(\R^d)^N$ vanish. A sufficient, but not necessary, condition is rapid decay of $p_t^N(x)b_i(x)$ and of the products obtained by multiplying $p_t^N(x)b_i(x)$ by $\log(p_t^N(x)/\rho^{\otimes N}(x))$. See~\cite{balasubramanian2024improved} for related conditions for SVGD.
\end{assumption}

\begin{assumption}[Reciprocal KDE control]
\label{ass:reciprocal}
For the time horizon $T>0$ under consideration, there is a finite constant $\Lambda_T$ such that
\begin{equation}
  \sup_{0\le t\le T}\E_t[\mathsf R_N(X(t))]\le \Lambda_T,
  \label{eq:reciprocal-assumption}
\end{equation}
where $\E_t$ denotes expectation with respect to $p_t^N$ and $\mathsf R_N$ is defined in~\eqref{eq:reciprocal-kde}.
\end{assumption}

\begin{remark}[Conditional nature of reciprocal-KDE control]
\label{rem:reciprocal-control-circularity}
The reciprocal-KDE assumption is a stability assumption, not a consequence of
the entropy identity. The local-occupancy results in Section~\ref{sec:reciprocal-control}
give deterministic or high-probability sufficient conditions under which
\[
  q_{X(t)}(X_i(t))\ge \lambda>0
  \qquad\text{for all }i
\]
at a fixed time, and under a propagated occupancy condition. However, propagation
arguments that use a flow Lipschitz constant are conditional: the Lipschitz
constant of the conservative vector field itself depends on reciprocal KDE
quantities such as \(1/q_x\), and sometimes on higher inverse powers of \(q_x\).
Thus there is a potential bootstrap structure.

The stability condition is imposed on the full deterministic interval
$[0,T]$. A lower-denominator region can be used in a bootstrap argument
if its invariance is proved separately. Likewise, a high-probability
occupancy event can be combined with a bound on its complement to
control the unconditional expectation in
\eqref{eq:reciprocal-assumption}. The entropy calculation uses the
original evolving joint density on this interval. Stopping particles
at the boundary of a denominator region may create a singular law,
so such a stopped dynamics requires its own transport and boundary
analysis.

\end{remark}

\begin{assumption}[Quadrature smoothness]
\label{ass:quadrature-smoothness}
For the time horizon $T>0$, there are finite constants $B_{A,T}(h)$ and $B_{V,T}(h)$ such that, for every configuration $X(t)$ visited by the flow for $0\le t\le T$,
\begin{align}
  \sup_{z\in\R^d}\norm{D^2\{\mathcal A_\rho b_{X(t)}\}(z)}_{\mathrm{op}}
  &\le B_{A,T}(h),
  \label{eq:BA}\\
  \sup_{z\in\R^d}\norm{D^2\{\norm{b_{X(t)}}^2\}(z)}_{\mathrm{op}}
  &\le B_{V,T}(h).
  \label{eq:BV}
\end{align}
Here $D^2$ denotes the Hessian, and $\norm{\cdot}_{\mathrm{op}}$ is the operator norm.
\end{assumption}

\begin{remark}[Bandwidth dependence of the quadrature constants]
The constants $B_{A,T}(h)$ and $B_{V,T}(h)$ may depend on $N,h,T,\rho,K$, and on the region of state space visited by the dynamics. This dependence is not a lower-order issue. Writing
\[
  r_x(z):=\log\rho(z)-\log q_x(z),
  \qquad
  b_x(z)=\nabla r_x(z),
\]
we have
\[
  \mathcal A_\rho b_x
  =
  \Delta r_x+\nabla\log\rho\cdot\nabla r_x,
\]
and
\[
  D^2\norm{b_x}^2
  =
  2(Db_x)^\top(Db_x)
  +
  2\sum_{\ell=1}^d b_{x,\ell}D^2b_{x,\ell}.
\]
Thus the quadrature constants involve log-derivatives of $\rho$ and $q_x$ up to fourth order. If these log-derivatives vary on the kernel scale, for example if
\[
  \sup_z\norm{D^j(\log\rho-\log q_x)(z)}=O(h^{-j}),
  \qquad j=1,2,3,4,
\]
and the target-score factors have compatible derivative bounds, then $B_{A,T}(h)+B_{V,T}(h)=O(h^{-4})$. The target-score bounds are needed because differentiating $\nabla\log\rho\cdot\nabla r_x$ also differentiates $\nabla\log\rho$. In that regime the quadrature contribution
\[
  \{B_{A,T}(h)+B_{V,T}(h)\}m_2(K_h)=O(h^{-2})
\]
scales like $h^{-2}$, not $h^2$, and shrinking the bandwidth worsens this term. Consequently, the rates below are non-asymptotic inequalities with the quadrature constants displayed. Any asymptotic bandwidth optimization must explicitly account for the $h$-dependence of $B_{A,T}(h)$ and $B_{V,T}(h)$.
\end{remark}

\begin{remark}[Sufficient derivative bounds for quadrature growth]
The exponent $\beta$ describes the regularity of the drift along the
trajectory, including its interaction with the target score.
For example, suppose $0<\ell_h\le1$ and, uniformly in the configuration,
time, and spatial variable,
\[
  \norm{D^jr_x}_{\mathrm{op}}\le C_r\ell_h^{-j},
  \quad j=1,2,3,4,\qquad
  \norm{D^ks_\rho}_{\mathrm{op}}\le C_\rho,
  \quad k=0,1,2.
\]
Here norms of multilinear derivatives use the Euclidean norm in each
argument and in the range. For unit vectors $v,w$, the product rule gives
\begin{align*}
 D^2(s_\rho\cdot\nabla r_x)[v,w]
 ={}&D^2s_\rho[v,w]\cdot\nabla r_x
      +Ds_\rho[v]\cdot D^2r_x[w]\\
    &+Ds_\rho[w]\cdot D^2r_x[v]
      +s_\rho\cdot D^3r_x[v,w].
\end{align*}
The four terms have absolute values bounded, respectively, by
$C_\rho C_r\ell_h^{-1}$,
$C_\rho C_r\ell_h^{-2}$,
$C_\rho C_r\ell_h^{-2}$, and
$C_\rho C_r\ell_h^{-3}$.
Also $\norm{D^2\Delta r_x}_{\mathrm{op}}\le dC_r\ell_h^{-4}$,
because the Laplacian sums over $d$ coordinate directions. Thus
\[
  \norm{D^2\{\mathcal A_\rho b_x\}}_{\mathrm{op}}
  \le(dC_r+4C_\rho C_r)\ell_h^{-4}.
\]
Differentiating $\norm{\nabla r_x}^2$ twice gives
\[
  D^2\norm{\nabla r_x}^2[v,w]
  =2D^2r_x[v]\cdot D^2r_x[w]
   +2\nabla r_x\cdot D^3r_x[v,w].
\]
Each term is bounded by $2C_r^2\ell_h^{-4}$, so
$B_{A,N}(h)+B_{V,N}(h)=O(\ell_h^{-4})$.
Taking $\ell_h=1$ gives the uniform regime; taking
$\ell_h=h^\alpha$ with $0<\alpha<1/2$ gives $\beta=4\alpha<2$.
These are sufficient hypotheses to be verified for the target and
trajectory under consideration.

Global bounded-score conditions do not hold for every smooth kernel.
For a finite Gaussian KDE,
\[
  s_x(z)=h^{-2}\left\{\sum_jw_j(z)x_j-z\right\},
  \qquad w_j(z)\ge0,\quad\sum_jw_j(z)=1.
\]
Writing $R_x=\max_j\norm{x_j}$, the triangle inequality gives
$\norm{s_x(z)}\ge h^{-2}(\norm{z}-R_x)$.
Thus its first log derivative is unbounded on $\R^d$.
In such settings, the scalar integrand bounds in
Assumption~\ref{ass:quadrature-smoothness} must be verified directly,
using any cancellations available. Local derivative bounds can also
be used when the corresponding quadrature tails are included.
\end{remark}

Finally, we emphasize that our conservative drifting analysis assumes that the smoothing kernel is
sufficiently smooth, for example \(K\in C^3\) or stronger when the quadrature
constants involve higher derivatives. This assumption excludes the exact
Laplace kernel $  K_h(u)=c_dh^{-d}\exp(-\norm{u}/h),$ which is not differentiable at the origin. This is why the exact Laplace kernel
is treated separately in Section~\ref{sec:nonconservative-laplace-drifting} through the
non-conservative displacement field and the sharp-companion-kernel
decomposition. A smooth regularization of the Laplace kernel could be included
in the conservative analysis, but the resulting estimates would have constants
depending on the regularization scale.

The forthcoming results do not require the target to be log-concave,
strongly log-concave, or to satisfy a log-Sobolev, Poincar\'e, or
dissipativity condition. They control time-averaged residual drift and
smoothed Fisher discrepancy under the stated regularity assumptions.
Converting these quantities into global KL or Wasserstein convergence
requires additional functional inequalities or stability estimates.

\section{Proof of the Conservative Entropy Identity}
\label{app:conservative-entropy}

The next theorem is the basic finite-particle identity. Differentiating the
particle's own KDE contribution produces a self-interaction term in the
joint-entropy calculation.

\begin{theorem}[Joint-entropy identity]
\label{thm:entropy-identity}
Let Assumptions~\ref{ass:kernel} and~\ref{ass:regularity} hold. Then, for every $t$ for which the quantities are finite,
\begin{equation}
  \frac{\dd}{\dd t}H_N(t)
  =
  -N\E_t[\mathsf S_N(X(t))]
  -\Delta K_h(0)\,\E_t[\mathsf R_N(X(t))].
  \label{eq:entropy-identity}
\end{equation}
Equivalently, with
\[
  a_h:=\{-\Delta K_h(0)\}_+,
\]
one has the upper bound
\begin{equation}
  \frac{\dd}{\dd t}H_N(t)
  \le
  -N\E_t[\mathsf S_N(X(t))]
  +a_h\E_t[\mathsf R_N(X(t))].
  \label{eq:entropy-identity-upper}
\end{equation}
\end{theorem}

\begin{proof}
Let
\[
  r_N(x):=\rho^{\otimes N}(x).
\]
By definition,
\[
  H_N(t)=\int p_t^N(x)\log\frac{p_t^N(x)}{r_N(x)}\dd x.
\]
Differentiating under the integral gives
\[
  \frac{\dd}{\dd t}H_N(t)
  =\int \partial_t p_t^N(x)
  \left\{\log\frac{p_t^N(x)}{r_N(x)}+1\right\}\dd x.
\]
Since $p_t^N$ is a probability density,
\[
  \int \partial_t p_t^N(x)\dd x
  =\frac{\dd}{\dd t}\int p_t^N(x)\dd x=0.
\]
Thus
\[
  \frac{\dd}{\dd t}H_N(t)
  =\int \partial_t p_t^N(x)\log\frac{p_t^N(x)}{r_N(x)}\dd x.
\]
Using the Liouville equation,
\[
  \frac{\dd}{\dd t}H_N(t)
  =-
  \sum_{i=1}^N
  \int \nabla_{x_i}\cdot\{p_t^N(x)b_i(x)\}
  \log\frac{p_t^N(x)}{r_N(x)}\dd x.
\]
Integrating by parts in $x_i$,
\[
  \frac{\dd}{\dd t}H_N(t)
  =
  \sum_{i=1}^N
  \int p_t^N(x)b_i(x)\cdot
  \nabla_{x_i}\log\frac{p_t^N(x)}{r_N(x)}\dd x.
\]
Now
\[
  \nabla_{x_i}\log\frac{p_t^N(x)}{r_N(x)}
  =\nabla_{x_i}\log p_t^N(x)-\nabla_{x_i}\log r_N(x).
\]
Since $r_N(x)=\prod_{j=1}^N\rho(x_j)$,
\[
  \nabla_{x_i}\log r_N(x)=s_\rho(x_i).
\]
Therefore
\begin{align*}
  \frac{\dd}{\dd t}H_N(t)
  &=
  \sum_{i=1}^N
  \int p_t^N b_i\cdot \nabla_{x_i}\log p_t^N\dd x
  -
  \sum_{i=1}^N
  \int p_t^N b_i\cdot s_\rho(x_i)\dd x.
\end{align*}
For the first integral,
\[
  p_t^N\nabla_{x_i}\log p_t^N=\nabla_{x_i}p_t^N.
\]
Hence
\[
  \int p_t^N b_i\cdot \nabla_{x_i}\log p_t^N\dd x
  =\int b_i\cdot\nabla_{x_i}p_t^N\dd x
  =-\int p_t^N\nabla_{x_i}\cdot b_i\dd x.
\]
Consequently,
\begin{equation}
  \frac{\dd}{\dd t}H_N(t)
  =
  -\E_t\left[
  \sum_{i=1}^N
  \left\{
  \nabla_{x_i}\cdot b_i(X(t))
  +s_\rho(X_i(t))\cdot b_i(X(t))
  \right\}
  \right].
  \label{eq:H-before-self}
\end{equation}
It remains to compute $\nabla_{x_i}\cdot b_i(x)$. Since
\[
  b_i(x)=s_\rho(x_i)-s_x(x_i),
\]
we have
\[
  \nabla_{x_i}\cdot b_i(x)
  =\nabla\cdot s_\rho(x_i)-\nabla_{x_i}\cdot s_x(x_i).
\]
The derivative of $s_x(x_i)$ with respect to $x_i$ contains two pieces: the derivative of the evaluation point and the derivative of the $i$th KDE center. Write
\[
  s_x^m(z)=\frac{\partial_m q_x(z)}{q_x(z)},
  \qquad m=1,\ldots,d.
\]
For fixed $z$,
\[
  \partial_{x_i^\ell}q_x(z)=-\frac1N\partial_\ell K_h(z-x_i),
\]
and
\[
  \partial_{x_i^\ell}\partial_m q_x(z)
  =-\frac1N\partial_{\ell m}^2K_h(z-x_i).
\]
Therefore
\[
  \partial_{x_i^\ell}s_x^m(z)
  =
  -\frac1N\frac{\partial_{\ell m}^2K_h(z-x_i)}{q_x(z)}
  +\frac1N
  \frac{\partial_m q_x(z)\partial_\ell K_h(z-x_i)}{q_x(z)^2}.
\]
Set $m=\ell$ and sum over $\ell$. Since $K_h$ is even, $\nabla K_h(0)=0$, and hence at $z=x_i$,
\[
  \sum_{\ell=1}^d\partial_{x_i^\ell}s_x^\ell(z)\big|_{z=x_i}
  =-\frac1N\frac{\Delta K_h(0)}{q_x(x_i)}.
\]
By the chain rule,
\[
  \nabla_{x_i}\cdot s_x(x_i)
  =\nabla_z\cdot s_x(z)\big|_{z=x_i}
  -\frac1N\frac{\Delta K_h(0)}{q_x(x_i)}.
\]
Thus
\begin{align*}
  \nabla_{x_i}\cdot b_i(x)
  &=\nabla\cdot s_\rho(x_i)
  -\nabla_z\cdot s_x(z)\big|_{z=x_i}
  +\frac1N\frac{\Delta K_h(0)}{q_x(x_i)}\\
  &=\nabla_z\cdot b_x(z)\big|_{z=x_i}
  +\frac1N\frac{\Delta K_h(0)}{q_x(x_i)}.
\end{align*}
Since $b_i(x)=b_x(x_i)$,
\[
  \nabla_{x_i}\cdot b_i(x)+s_\rho(x_i)\cdot b_i(x)
  =
  \mathcal A_\rho b_x(x_i)
  +\frac1N\frac{\Delta K_h(0)}{q_x(x_i)}.
\]
Substituting this identity into \eqref{eq:H-before-self} gives
\begin{align*}
  \frac{\dd}{\dd t}H_N(t)
  &=-\E_t\left[
  \sum_{i=1}^N\mathcal A_\rho b_{X(t)}(X_i(t))
  +\frac{\Delta K_h(0)}{N}\sum_{i=1}^N\frac1{q_{X(t)}(X_i(t))}
  \right]\\
  &=-N\E_t[\mathsf S_N(X(t))]
  -\Delta K_h(0)\E_t[\mathsf R_N(X(t))].
\end{align*}
This proves \eqref{eq:entropy-identity}. Since
\[
  -\Delta K_h(0)\le \{-\Delta K_h(0)\}_+=a_h,
\]
\eqref{eq:entropy-identity-upper} follows.
\end{proof}

\begin{remark}[Interpretation of the self-interaction term]
The first term on the right-hand side of \eqref{eq:entropy-identity} is the desired entropy dissipation. The second term is finite-particle specific: it appears because the velocity at $x_i$ depends on the same particle through the KDE denominator. Its sign and size are controlled by $-\Delta K_h(0)$ and by the reciprocal quantity $\mathsf R_N$. This is why the subsequent rate theorem requires reciprocal-KDE control. If particles are locally well populated at bandwidth $h$, then $q_x(x_i)$ is bounded below and the self-interaction correction remains of order $1/N$ after the entropy normalization.

For Gaussian kernels, the self-interaction correction has a definite positive
sign. Indeed, for $  K(u)=(2\pi)^{-d/2}\exp(-\norm{u}^2/2),
$ we have, for each coordinate \(j\), $  \partial_{jj}K(0)=-K(0)=-(2\pi)^{-d/2}.$ Therefore
\[
  \Delta K(0)
  =
  \sum_{j=1}^d\partial_{jj}K(0)
  =
  -d(2\pi)^{-d/2}<0,
\]
and hence $a_1:=\{-\Delta K(0)\}_+
  =
  d(2\pi)^{-d/2}>0.$

For the bandwidth-scaled kernel \(K_h(u)=h^{-d}K(u/h)\),
\[
  -\Delta K_h(0)
  =
  h^{-d-2}a_1.
\]
Thus the reciprocal-KDE term in the conservative finite-particle bound is a
genuine positive self-interaction correction for Gaussian kernels.

\end{remark}

\section{Proofs and Further Results for Conservative Rates}
\label{app:conservative-rates}

\subsection{Rate for the Empirical Stein Drift}

\begin{theorem}[Entropy rate for the empirical Stein drift]
\label{thm:S-rate}
Let Assumptions~\ref{ass:kernel},~\ref{ass:regularity}, and~\ref{ass:reciprocal} hold. Then, for every $T>0$,
\begin{equation}
  \frac1T\int_0^T\E_t[\mathsf S_N(X(t))]\dd t
  \le
  \frac{H_N(0)}{NT}
  +\frac{a_h\Lambda_T}{N},
  \qquad
  a_h:=\{-\Delta K_h(0)\}_+.
  \label{eq:S-rate}
\end{equation}
\end{theorem}

\begin{proof}
By Theorem~\ref{thm:entropy-identity},
\[
  H_N'(t)
  \le
  -N\E_t[\mathsf S_N(X(t))]
  +a_h\E_t[\mathsf R_N(X(t))].
\]
Using Assumption~\ref{ass:reciprocal},
\[
  H_N'(t)
  \le
  -N\E_t[\mathsf S_N(X(t))]
  +a_h\Lambda_T.
\]
Rearrange:
\[
  N\E_t[\mathsf S_N(X(t))]
  \le
  -H_N'(t)+a_h\Lambda_T.
\]
Integrate over $t\in[0,T]$:
\[
  N\int_0^T\E_t[\mathsf S_N(X(t))]\dd t
  \le
  H_N(0)-H_N(T)+a_h\Lambda_TT.
\]
Since relative entropy is nonnegative,
\[
  H_N(0)-H_N(T)\le H_N(0).
\]
Therefore
\[
  N\int_0^T\E_t[\mathsf S_N(X(t))]\dd t
  \le
  H_N(0)+a_h\Lambda_TT.
\]
Dividing by $NT$ proves \eqref{eq:S-rate}.
\end{proof}

\begin{remark}[What the Stein-drift rate controls]
The quantity $\mathsf S_N$ is the empirical Stein drift associated with the vector field $b_x$. It is not written as a square, so \eqref{eq:S-rate} is not yet a residual-velocity bound. The role of the quadrature step below is to compare this empirical Stein drift with the KDE-averaged Fisher discrepancy, which is nonnegative and is directly related to the squared velocity of the particles.
\end{remark}

\subsection{Quadrature Conversion to Fisher and Center Velocity}

The empirical Stein drift $\mathsf S_N$ is not manifestly nonnegative. The next lemma identifies the nonnegative population quantity that it approximates.

\begin{lemma}[KDE-averaged Stein identity]
\label{lem:stein-to-fisher}
For every fixed configuration $x$ for which the integrations by parts are justified,
\begin{equation}
  \int_{\R^d}\mathcal A_\rho b_x(z)q_x(z)\dd z
  =\mathsf I_N(x).
  \label{eq:stein-fisher-identity}
\end{equation}
\end{lemma}

\begin{proof}
By definition of $\mathcal A_\rho$,
\[
  \int q_x\mathcal A_\rho b_x
  =\int q_x\nabla\cdot b_x+
  \int q_xs_\rho\cdot b_x.
\]
Integrating the first term by parts,
\[
  \int q_x\nabla\cdot b_x
  =-\int \nabla q_x\cdot b_x.
\]
Since $\nabla q_x=q_xs_x$,
\[
  -\int \nabla q_x\cdot b_x
  =-\int q_xs_x\cdot b_x.
\]
Thus
\[
  \int q_x\mathcal A_\rho b_x
  =\int q_x(s_\rho-s_x)\cdot b_x.
\]
Because $b_x=s_\rho-s_x$,
\[
  \int q_x\mathcal A_\rho b_x
  =\int q_x\norm{b_x}^2
  =\mathsf I_N(x).
\]
This proves the lemma.\end{proof}

\begin{lemma}[Point-evaluation quadrature error]
\label{lem:quadrature}
Let $K_h$ be even and have finite second moment
\[
  m_2(K_h):=\int_{\R^d}\norm{u}^2K_h(u)\dd u.
\]
Under Assumption~\ref{ass:quadrature-smoothness}, for every configuration $X(t)$ visited by the flow,
\begin{align}
  \abs{\mathsf S_N(X(t))-\mathsf I_N(X(t))}
  &\le \frac12B_{A,T}(h)m_2(K_h),
  \label{eq:S-I-quad}\\
  \abs{\mathsf V_N(X(t))-\mathsf I_N(X(t))}
  &\le \frac12B_{V,T}(h)m_2(K_h).
  \label{eq:V-I-quad}
\end{align}
\end{lemma}

\begin{proof}
We prove \eqref{eq:S-I-quad}; the proof of \eqref{eq:V-I-quad} is identical with $\norm{b_x}^2$ replacing $\mathcal A_\rho b_x$.

Fix $x$ and define
\[
  \phi_x(z):=\mathcal A_\rho b_x(z).
\]
By Lemma~\ref{lem:stein-to-fisher},
\[
  \mathsf I_N(x)=\int q_x(z)\phi_x(z)\dd z.
\]
Since
\[
  q_x(z)=\frac1N\sum_{i=1}^NK_h(z-x_i),
\]
we have
\[
  \mathsf I_N(x)
  =\frac1N\sum_{i=1}^N
  \int K_h(z-x_i)\phi_x(z)\dd z.
\]
Meanwhile,
\[
  \mathsf S_N(x)=\frac1N\sum_{i=1}^N\phi_x(x_i).
\]
Therefore
\[
  \mathsf S_N(x)-\mathsf I_N(x)
  =\frac1N\sum_{i=1}^N
  \left\{\phi_x(x_i)-\int K_h(z-x_i)\phi_x(z)\dd z\right\}.
\]
Set $u=z-x_i$. Then
\[
  \int K_h(z-x_i)\phi_x(z)\dd z
  =\int K_h(u)\phi_x(x_i+u)\dd u.
\]
Taylor's formula gives
\[
  \phi_x(x_i+u)
  =\phi_x(x_i)+\nabla\phi_x(x_i)\cdot u
  +\int_0^1(1-r)u^\top D^2\phi_x(x_i+ru)u\dd r.
\]
Integrating against $K_h(u)\dd u$, the zeroth-order term gives $\phi_x(x_i)$ because $\int K_h=1$. The first-order term vanishes because $K_h$ is even, hence $\int uK_h(u)\dd u=0$. Thus
\begin{align*}
  &\abs{\int K_h(u)\phi_x(x_i+u)\dd u-\phi_x(x_i)}\\
  &\qquad\le
  \int K_h(u)\int_0^1(1-r)
  \norm{D^2\phi_x(x_i+ru)}_{\mathrm{op}}\norm{u}^2\dd r\dd u.
\end{align*}
Using Assumption~\ref{ass:quadrature-smoothness},
\[
  \abs{\int K_h(u)\phi_x(x_i+u)\dd u-\phi_x(x_i)}
  \le
  B_{A,T}(h)\int_0^1(1-r)\dd r\int \norm{u}^2K_h(u)\dd u.
\]
Since $\int_0^1(1-r)\dd r=1/2$,
\[
  \abs{\int K_h(u)\phi_x(x_i+u)\dd u-\phi_x(x_i)}
  \le \frac12B_{A,T}(h)m_2(K_h).
\]
Averaging over $i=1,\ldots,N$ proves \eqref{eq:S-I-quad}.\end{proof}

\begin{theorem}[Continuous-time rates for Fisher discrepancy and center velocity]
\label{thm:main-rate}
Let Assumptions~\ref{ass:kernel},~\ref{ass:regularity},~\ref{ass:reciprocal}, and~\ref{ass:quadrature-smoothness} hold. Then, for every $T>0$,
\begin{align}
  \frac1T\int_0^T\E_t[\mathsf I_N(X(t))]\dd t
  &\le
  \frac{H_N(0)}{NT}
  +\frac{a_h\Lambda_T}{N}
  +\frac12B_{A,T}(h)m_2(K_h),
  \label{eq:I-rate}\\
  \frac1T\int_0^T\E_t[\mathsf V_N(X(t))]\dd t
  &\le
  \frac{H_N(0)}{NT}
  +\frac{a_h\Lambda_T}{N}
  +\frac12\{B_{A,T}(h)+B_{V,T}(h)\}m_2(K_h).
  \label{eq:V-rate}
\end{align}
\end{theorem}

\begin{proof}
From Lemma~\ref{lem:quadrature},
\[
  \mathsf I_N(X(t))
  \le
  \mathsf S_N(X(t))+\frac12B_{A,T}(h)m_2(K_h).
\]
Taking expectations, integrating over $[0,T]$, dividing by $T$, and applying Theorem~\ref{thm:S-rate} gives \eqref{eq:I-rate}.

Similarly,
\[
  \mathsf V_N(X(t))
  \le
  \mathsf I_N(X(t))+\frac12B_{V,T}(h)m_2(K_h).
\]
Combining this inequality with \eqref{eq:I-rate} gives \eqref{eq:V-rate}.
\end{proof}

\begin{remark}[Implication for one-step generation]
Theorem~\ref{thm:main-rate} bounds the average squared speed of the generated particle cloud under the conservative drift. If a time $\tau$ is sampled uniformly from $[0,T]$, then
\[
  \E\left[\mathsf V_N(X(\tau))\right]
  =
  \frac1T\int_0^T\E_t[\mathsf V_N(X(t))]\dd t.
\]
Consequently, a small right-hand side in \eqref{eq:V-rate} implies that, at a typical training time, one more explicit drift step changes the empirical generated law by at most
\[
  \E\left[W_2^2\left(\mu_{X(\tau)}^N,
  \frac1N\sum_{i=1}^N\delta_{X_i(\tau)+\eta b_{X(\tau)}(X_i(\tau))}\right)\right]
  \le
  \eta^2
  \frac1T\int_0^T\E_t[\mathsf V_N(X(t))]\dd t.
\]
Thus the theorem is a convergence-to-small-residual-drift statement for one-step generation. With the quadrature comparison, it also controls the smoothed Fisher discrepancy between $q_x$ and $\rho$.
\end{remark}

\begin{corollary}[Bandwidth form]
\label{cor:bandwidth-rate}
Let $K_h(u)=h^{-d}K(u/h)$ and let the assumptions of Theorem~\ref{thm:main-rate} hold. Then
\[
  m_2(K_h)=h^2m_2(K),
  \qquad
  \Delta K_h(0)=h^{-d-2}\Delta K(0).
\]
Consequently, with $a_1:=\{-\Delta K(0)\}_+$,
\begin{equation}
  \frac1T\int_0^T\E_t[\mathsf V_N(X(t))]\dd t
  \le
  \frac{H_N(0)}{NT}
  +\frac{a_1\Lambda_T}{Nh^{d+2}}
  +\frac12\{B_{A,T}(h)+B_{V,T}(h)\}m_2(K)h^2.
  \label{eq:bandwidth-rate}
\end{equation}
If, in addition, $H_N(0)\le \kappa_0N$ and $T=N$, then
\begin{equation}
  \frac1N\int_0^N\E_t[\mathsf V_N(X(t))]\dd t
  \le
  \frac{\kappa_0}{N}
  +\frac{a_1\Lambda_N}{Nh^{d+2}}
  +\frac12\{B_{A,N}(h)+B_{V,N}(h)\}m_2(K)h^2.
  \label{eq:T-equals-N-rate}
\end{equation}
Thus the corresponding root-discrepancy bound is
\begin{equation}
  \left(\frac1N\int_0^N\E_t[\mathsf V_N(X(t))]\dd t\right)^{1/2}
  \le
  \sqrt{\frac{\kappa_0}{N}}
  +\sqrt{\frac{a_1\Lambda_N}{Nh^{d+2}}}
  +h\sqrt{\frac12\{B_{A,N}(h)+B_{V,N}(h)\}m_2(K)}.
  \label{eq:root-rate}
\end{equation}
\end{corollary}

\begin{proof}
The scaling identities follow by the change of variables $u=hv$:
\[
  m_2(K_h)=\int \norm{u}^2h^{-d}K(u/h)\dd u
  =h^2\int \norm{v}^2K(v)\dd v=h^2m_2(K).
\]
Also,
\[
  \Delta K_h(u)=h^{-d-2}(\Delta K)(u/h),
\]
so $\Delta K_h(0)=h^{-d-2}\Delta K(0)$. Substituting these identities into \eqref{eq:V-rate} gives \eqref{eq:bandwidth-rate}. If $H_N(0)\le\kappa_0N$ and $T=N$, then
\[
  \frac{H_N(0)}{NT}\le \frac{\kappa_0N}{N^2}=\frac{\kappa_0}{N},
\]
which gives \eqref{eq:T-equals-N-rate}. Finally, \eqref{eq:root-rate} follows from $  \sqrt{a+b+c}\le \sqrt a+\sqrt b+\sqrt c$, for $a,b,c\ge0$.
\end{proof}

\subsection{Optimizing the Conservative Bandwidth and The One-step Size}
\label{subsec:kde-optimized-rate}

The bandwidth form \eqref{eq:T-equals-N-rate} separates three effects: the entropy initialization term, the finite-particle self-interaction term, and the quadrature term. Since the quadrature constants may depend on $h$, we write them explicitly as functions of the bandwidth. Define
\[
  A_N(h):=a_1\Lambda_N(h),
  \qquad
  C_N(h):=\frac12\{B_{A,N}(h)+B_{V,N}(h)\}m_2(K).
\]
Then \eqref{eq:T-equals-N-rate} reads
\begin{equation}
  \frac1N\int_0^N\E_t[\mathsf V_N(X(t))]\dd t
  \le
  \frac{\kappa_0}{N}
  +
  \frac{A_N(h)}{Nh^{d+2}}
  +
  C_N(h)h^2.
  \label{eq:kde-rate-ABC}
\end{equation}
Thus the commonly stated bandwidth tradeoff is valid only after specifying how $A_N(h)$ and $C_N(h)$ scale with $h$.

Assume that, on the admissible bandwidth range, there are constants $A,C>0$ and $0\le \beta<2$ such that
\begin{equation}
  A_N(h)\le A,
  \qquad
  C_N(h)\le Ch^{-\beta}.
  \label{eq:h-dependent-quadrature-growth}
\end{equation}
Equivalently, the quadrature constants satisfy
\[
  B_{A,N}(h)+B_{V,N}(h)=O(h^{-\beta}).
\]
Then \eqref{eq:kde-rate-ABC} implies
\begin{equation}
  \frac1N\int_0^N\E_t[\mathsf V_N(X(t))]\dd t
  \le
  \frac{\kappa_0}{N}
  +
  \frac{A}{Nh^{d+2}}
  +
  Ch^{2-\beta}.
  \label{eq:kde-beta-rate}
\end{equation}
The last two terms are minimized at
\begin{equation}
  h_{\star,N}
  :=
  \left(
  \frac{(d+2)A}{(2-\beta)CN}
  \right)^{1/(d+4-\beta)}.
  \label{eq:kde-optimal-h}
\end{equation}
Indeed, differentiating
\[
  h\mapsto \frac{A}{Nh^{d+2}}+Ch^{2-\beta}
\]
gives
\[
  -\frac{(d+2)A}{Nh^{d+3}}+(2-\beta)Ch^{1-\beta}.
\]
Setting this derivative equal to zero gives \eqref{eq:kde-optimal-h}. At this bandwidth,
\[
  \frac{A}{Nh_{\star,N}^{d+2}}
  =
  \frac{2-\beta}{d+2}C h_{\star,N}^{2-\beta}.
\]
Therefore
\begin{equation}
  \frac1N\int_0^N\E_t[\mathsf V_N(X(t))]\dd t
  \le
  \frac{\kappa_0}{N}
  +
  \frac{d+4-\beta}{d+2}C
  \left(
  \frac{(d+2)A}{(2-\beta)CN}
  \right)^{(2-\beta)/(d+4-\beta)}.
  \label{eq:kde-optimized-rate}
\end{equation}
Thus, under \eqref{eq:h-dependent-quadrature-growth}, the squared residual velocity is
\[
  O\left(N^{-(2-\beta)/(d+4-\beta)}\right),
\]
and the root residual velocity is
\[
  O\left(N^{-(2-\beta)/(2(d+4-\beta))}\right).
\]
The $h$-uniform quadrature case is the special case $\beta=0$, which recovers the root rate $N^{-1/(d+4)}$. If $\beta\ge2$, the quadrature term $h^{2-\beta}$ does not vanish as $h\to0$, and this quadrature argument does not yield a vanishing residual by bandwidth shrinkage. For example, the natural kernel-scale derivative scaling $B_{A,N}(h)+B_{V,N}(h)=O(h^{-4})$ corresponds to $\beta=4$ and falls outside the regime of the optimized vanishing-rate statement.

The occupancy condition remains compatible with \eqref{eq:kde-optimal-h} whenever $Nh_{\star,N}^d\to\infty$ sufficiently quickly. For $0\le\beta<2$,
\[
  Nh_{\star,N}^d
  \asymp
  N^{1-d/(d+4-\beta)}
  =
  N^{(4-\beta)/(d+4-\beta)},
\]
which diverges faster than $\log N$ for fixed $d$ and fixed $\beta<2$.

The explicit one-step size $\eta$ does not enter the continuous-time entropy inequality. It enters when the learned generator is interpreted through one more explicit correction $z\mapsto z+\eta b_x(z)$. For a fixed $h$, the theorem gives
\[
  \E\left[W_2^2(\mu_x^N,\widetilde\mu_x^N)\right]
  \le
  \eta^2 R_{N,h},
\]
where $R_{N,h}$ denotes the right-hand side of \eqref{eq:kde-rate-ABC}. Taking $h=h_{\star,N}$ gives a residual correction of order
\[
  \eta\,N^{-(2-\beta)/(2(d+4-\beta))}
\]
in root mean square, up to constants. For a fixed configuration, if $b_x$
has Lipschitz constant $L_b(h)$, then the map $z\mapsto z+\eta b_x(z)$ satisfies
\[
  \norm{\{z+\eta b_x(z)\}-\{z'+\eta b_x(z')\}-(z-z')}
  \le\eta L_b(h)\norm{z-z'}.
\]
Thus $\eta L_b(h)\le c<1$ controls the deformation of a single
frozen-field step. Under an estimate $L_b(h)=O(h^{-2})$, this suggests
\[
  \eta_{\star,N}\asymp h_{\star,N}^2
  \asymp
  N^{-2/(d+4-\beta)}.
\]
For Gaussian kernels, the displacement velocity satisfies
$u^{\mathrm{disp}}=h^2b$, so this conservative-step choice corresponds
to a constant-order step size for the displacement formulation.
Approximating the interacting-particle ODE over a time horizon also
requires control of the full configuration vector field and the
accumulated error; see Section~\ref{app:stop-gradient-ode}.

\begin{remark}[Initialization]
Suppose the initial particles are independent with common density
$\mu_0$. If $\KL(\mu_0\|\rho)<\infty$, then
\[
  p_0^N=\mu_0^{\otimes N}
  \quad\Longrightarrow\quad
  H_N(0)=N\KL(\mu_0\|\rho).
\]
Thus the condition $H_N(0)\le\kappa_0N$ holds with $\kappa_0=\KL(\mu_0\|\rho)$.
\end{remark}

\section[Sufficient Conditions for Reciprocal KDE Control on R-d]{Sufficient Conditions for Reciprocal KDE control on $\R^d$}\label{sec:reciprocal-control}

The reciprocal-KDE condition \eqref{eq:reciprocal-assumption} is needed because the center-evaluation divergence differentiates the $i$th KDE center and produces a factor $1/q_x(x_i)$. This section gives checkable sufficient conditions. The occupancy estimates in this section are elementary and are included for
completeness. The deterministic implication is simply that if every particle has
enough neighbors in an \(O(h)\)-ball, then the KDE denominator at that particle
is bounded below. The high-probability version follows from the standard
binomial Chernoff lower-tail bound, applied conditionally on each particle
location, together with a union bound over particles. Such local-neighborhood occupancy
arguments are also standard in random geometric graph theory
\citep{penrose2003random}. We use these estimates only as sufficient conditions
for reciprocal-KDE control; the finite-particle entropy bounds themselves remain
conditional on the stated reciprocal-denominator assumptions.

\subsection{Deterministic Local-occupancy Condition}

\begin{assumption}[Kernel lower bound near the origin]
\label{ass:kernel-lower}
There exist constants $r_K>0$ and $\kappa_K>0$ such that
\begin{equation}
  K(u)\ge \kappa_K
  \qquad\text{whenever }\norm{u}\le r_K.
  \label{eq:kernel-lower}
\end{equation}
Equivalently,
\[
  K_h(u)\ge \kappa_K h^{-d}
  \qquad\text{whenever }\norm{u}\le r_Kh.
\]
\end{assumption}

For a configuration $x$, define the local neighbor count
\begin{equation}
  M_i(x;h):=\#\{j\in\{1,\ldots,N\}:\norm{x_j-x_i}\le r_Kh\}.
  \label{eq:local-neighbor-count}
\end{equation}
The count includes $j=i$.

\begin{proposition}[Local occupancy implies reciprocal control]
\label{prop:deterministic-occupancy}
Let Assumption~\ref{ass:kernel-lower} hold. Suppose that a configuration $x$ satisfies
\begin{equation}
  M_i(x;h)\ge \alpha Nh^d
  \qquad\text{for every }i=1,\ldots,N
  \label{eq:local-occupancy}
\end{equation}
for some $\alpha>0$. Then
\begin{equation}
  q_x(x_i)\ge \kappa_K\alpha
  \qquad\text{for every }i=1,\ldots,N,
  \label{eq:q-lower-from-occupancy}
\end{equation}
and hence
\begin{equation}
  \mathsf R_N(x)\le \frac1{\kappa_K\alpha}.
  \label{eq:R-local-occupancy}
\end{equation}
\end{proposition}

\begin{proof}
Fix $i$. By definition,
\[
  q_x(x_i)=\frac1N\sum_{j=1}^NK_h(x_i-x_j).
\]
For every $j$ with $\norm{x_j-x_i}\le r_Kh$, Assumption~\ref{ass:kernel-lower} gives
\[
  K_h(x_i-x_j)\ge \kappa_Kh^{-d}.
\]
Therefore
\[
  q_x(x_i)
  \ge \frac1N M_i(x;h)\kappa_Kh^{-d}.
\]
Using \eqref{eq:local-occupancy},
\[
  q_x(x_i)
  \ge \frac1N(\alpha Nh^d)\kappa_Kh^{-d}
  =\kappa_K\alpha.
\]
Taking reciprocals and averaging over $i$ yields \eqref{eq:R-local-occupancy}.\end{proof}

\subsection{Expectation and High-probability Bounds for I.I.D. Particles}

The previous proposition is deterministic. For i.i.d. samples, an exact binomial identity controls the expected reciprocal KDE, while a lower-tail estimate controls all local counts simultaneously.

\begin{assumption}[Uniform lower local mass]
\label{ass:lower-local-mass}
Let $\mu$ be a probability measure on $\R^d$. For the bandwidth $h$ under consideration, there is a constant $p_0>0$ such that
\begin{equation}
  \mu(B(y,r_Kh))\ge p_0h^d
  \qquad\text{for }\mu\text{-a.e. }y.
  \label{eq:lower-local-mass}
\end{equation}
\end{assumption}

\begin{remark}[One way to verify Assumption~\ref{ass:lower-local-mass}]
Suppose $\mu$ has density $f$ on a closed set $S\subset\R^d$, $f(y)\ge m_0>0$ for $y\in S$, and $S$ satisfies the thickness condition
\[
  \operatorname{Leb}(S\cap B(y,r))\ge \theta v_dr^d
  \qquad\text{for all }y\in S\text{ and }0<r\le r_0,
\]
where $v_d$ is the volume of the unit Euclidean ball. If $r_Kh\le r_0$, then
\[
  \mu(B(y,r_Kh))
  \ge m_0\operatorname{Leb}(S\cap B(y,r_Kh))
  \ge m_0\theta v_dr_K^dh^d.
\]
Thus Assumption~\ref{ass:lower-local-mass} holds with
\[
  p_0=m_0\theta v_dr_K^d.
\]
\end{remark}

\begin{proposition}[I.i.d. local occupancy]
\label{prop:iid-occupancy}
Let $Y_1,\ldots,Y_N$ be i.i.d. from a probability measure $\mu$ satisfying Assumption~\ref{ass:lower-local-mass}. Let
\[
  Y=(Y_1,\ldots,Y_N).
\]
Assume $N\ge2$. Define the event
\[
  \mathcal E_h
  :=\left\{M_i(Y;h)\ge \frac{p_0}{4}Nh^d\text{ for every }i=1,\ldots,N\right\}.
\]
Then
\begin{equation}
  \mathbb P(\mathcal E_h^c)
  \le
  N\exp\left(-\frac{p_0}{16}Nh^d\right).
  \label{eq:occupancy-prob}
\end{equation}
Consequently, under Assumption~\ref{ass:kernel-lower}, on $\mathcal E_h$,
\begin{equation}
  \mathsf R_N(Y)\le \frac4{\kappa_Kp_0}.
  \label{eq:R-iid-event}
\end{equation}
Moreover, since $K_h(0)=h^{-d}K(0)$,
\begin{equation}
  \mathsf R_N(Y)
  \le \frac{Nh^d}{K(0)}
  \qquad\text{always},
  \label{eq:self-bound}
\end{equation}
and, under Assumption~\ref{ass:kernel-lower},
\begin{equation}
  \E[\mathsf R_N(Y)]\le\frac1{\kappa_Kp_0}.
  \label{eq:expected-R-iid}
\end{equation}
\end{proposition}

\begin{proof}
Fix $i$. Conditional on $Y_i=y$, define
\[
  Z_i:=\sum_{j\ne i}\one\{\norm{Y_j-y}\le r_Kh\}.
\]
Then $Z_i$ is binomial with parameters $N-1$ and
\[
  p(y):=\mu(B(y,r_Kh)).
\]
By Assumption~\ref{ass:lower-local-mass}, $p(y)\ge p_0h^d$ for $\mu$-a.e. $y$. The neighbor count including the particle itself is
\[
  M_i(Y;h)=1+Z_i.
\]
Let $m_i:=(N-1)p(y)$. For $N\ge2$,
\[
  m_i\ge (N-1)p_0h^d\ge \frac12Np_0h^d.
\]
If
\[
  Z_i\ge \frac12m_i,
\]
then
\[
  M_i(Y;h)=1+Z_i\ge \frac12m_i\ge \frac14Np_0h^d.
\]
Therefore
\[
  \mathbb P\left(M_i(Y;h)<\frac14Np_0h^d\mid Y_i=y\right)
  \le
  \mathbb P\left(Z_i<\frac12m_i\mid Y_i=y\right).
\]
For a binomial random variable, the multiplicative Chernoff bound gives
\[
  \mathbb P\left(Z_i<\frac12m_i\mid Y_i=y\right)
  \le \exp\left(-\frac{m_i}{8}\right).
\]
Using $m_i\ge (N-1)p_0h^d\ge Np_0h^d/2$,
\[
  \exp\left(-\frac{m_i}{8}\right)
  \le
  \exp\left(-\frac{p_0}{16}Nh^d\right).
\]
This bound is uniform in $y$. Integrating over $Y_i$ and taking a union bound over $i=1,\ldots,N$ proves \eqref{eq:occupancy-prob}.

On $\mathcal E_h$, Proposition~\ref{prop:deterministic-occupancy} applies with $\alpha=p_0/4$, giving \eqref{eq:R-iid-event}. For the deterministic self-bound, observe that
\[
  q_Y(Y_i)=\frac1N\sum_{j=1}^NK_h(Y_i-Y_j)
  \ge \frac1NK_h(0)=\frac{K(0)}{Nh^d}.
\]
Thus
\[
  \frac1{q_Y(Y_i)}\le \frac{Nh^d}{K(0)}
  \qquad\text{for every }i,
\]
and averaging over $i$ gives \eqref{eq:self-bound}.

For the expectation, retain the actual neighbor count:
\[
  q_Y(Y_i)\ge\frac{\kappa_Kh^{-d}}N(1+Z_i).
\]
For $Z\sim\operatorname{Bin}(N-1,p)$ and $p>0$, the identity
$(k+1)^{-1}\binom{N-1}{k}=N^{-1}\binom N{k+1}$ gives
\begin{align*}
  \E\left[\frac1{1+Z}\right]
  &=\sum_{k=0}^{N-1}\frac1{k+1}\binom{N-1}{k}
             p^k(1-p)^{N-1-k}\\
  &=\frac1{Np}\sum_{\ell=1}^N\binom N\ell p^\ell(1-p)^{N-\ell}\\
  &=\frac{1-(1-p)^N}{Np}.
\end{align*}
The last equality follows by subtracting the $\ell=0$ term from
the binomial expansion of $(p+1-p)^N$. Therefore
\[
  \E\left[\left.\frac1{q_Y(Y_i)}\right|Y_i=y\right]
  \le\frac{h^d}{\kappa_K}
       \frac{1-(1-p(y))^N}{p(y)}
  \le\frac1{\kappa_Kp_0},
\]
where the final inequality uses $p(y)\ge p_0h^d$ and
$1-(1-p(y))^N\le1$. Integrating over $y$ and averaging over $i$
proves \eqref{eq:expected-R-iid}.
\end{proof}

\begin{remark}[Expected control and tail mass]
\label{rem:reciprocal-tail-adaptive}
The same calculation without a uniform lower local mass gives
\begin{equation}
  \E[\mathsf R_N(Y)]
  \le\frac{h^d}{\kappa_K}
       \int\frac{1-(1-p_h(y))^N}{p_h(y)}\,\mu(\dd y),
  \qquad p_h(y):=\mu(B(y,r_Kh)).
  \label{eq:reciprocal-tail-adaptive}
\end{equation}
At $p_h(y)=0$ the ratio is defined by its limit $N$.
For $p>0$, the finite geometric sum gives
\[
  \frac{1-(1-p)^N}{p}
  =\sum_{k=0}^{N-1}(1-p)^k\le\min\{N,p^{-1}\}.
\]
Thus \eqref{eq:reciprocal-tail-adaptive} retains the contribution of
regions with small local mass. Under uniform lower local mass, the
expected bound \eqref{eq:expected-R-iid} holds without a logarithmic
sample-size condition. A condition of order $Nh^d\gtrsim\log N$
is instead used when \eqref{eq:occupancy-prob} must control every
particle simultaneously.
\end{remark}

\subsection{Propagation of Occupancy Along the Continuous-time Flow}

The i.i.d. result controls the reciprocal KDE at a single time. To control it along the flow, one may combine initial occupancy at a slightly smaller scale with a Lipschitz-distortion bound for pairwise distances.

\begin{assumption}[Lipschitz distortion of the drift]
\label{ass:lipschitz-distortion}
There exists a deterministic nonnegative integrable function $L_t$ such that, for almost every realized trajectory on $[0,T]$,
\begin{equation}
  \norm{b_{X(t)}(z)-b_{X(t)}(z')}
  \le L_t\norm{z-z'}
  \qquad\text{for all }z,z'\in\R^d,
  \label{eq:b-lipschitz}
\end{equation}
for all $0\le t\le T$. Define
\begin{equation}
  \Gamma_T:=\int_0^T L_t\dd t.
  \label{eq:Gamma}
\end{equation}
\end{assumption}

\begin{lemma}[Pairwise distance distortion]
\label{lem:distance-distortion}
Under Assumption~\ref{ass:lipschitz-distortion}, for every pair $i,j$ and every $0\le t\le T$,
\begin{equation}
  \norm{X_i(t)-X_j(t)}
  \le e^{\Gamma_T}\norm{X_i(0)-X_j(0)}.
  \label{eq:distance-distortion}
\end{equation}
\end{lemma}

\begin{proof}
For $i,j$, set $D_{ij}(t):=X_i(t)-X_j(t)$. From \eqref{eq:particle-ode},
\[
  \frac{\dd}{\dd t}D_{ij}(t)
  =b_{X(t)}(X_i(t))-b_{X(t)}(X_j(t)).
\]
Using Assumption~\ref{ass:lipschitz-distortion},
\[
  \frac{\dd}{\dd t}\norm{D_{ij}(t)}
  \le L_t\norm{D_{ij}(t)}
\]
whenever $D_{ij}(t)\ne0$. The same inequality holds in the upper Dini derivative sense when $D_{ij}(t)=0$. Gronwall's inequality gives
\[
  \norm{D_{ij}(t)}
  \le \exp\left(\int_0^tL_s\dd s\right)\norm{D_{ij}(0)}
  \le e^{\Gamma_T}\norm{D_{ij}(0)}.
\]
This proves the lemma.\end{proof}

\begin{proposition}[Trajectory-level reciprocal control from initial occupancy]
\label{prop:trajectory-occupancy}
Assume Assumptions~\ref{ass:kernel-lower} and~\ref{ass:lipschitz-distortion}. Suppose the initial configuration satisfies
\begin{equation}
  \#\{j:\norm{X_j(0)-X_i(0)}\le r_Khe^{-\Gamma_T}\}
  \ge \alpha Nh^de^{-d\Gamma_T}
  \qquad\text{for every }i.
  \label{eq:initial-smaller-occupancy}
\end{equation}
Then, for every $0\le t\le T$,
\begin{equation}
  \mathsf R_N(X(t))\le \frac{e^{d\Gamma_T}}{\kappa_K\alpha}.
  \label{eq:trajectory-R-bound}
\end{equation}
\end{proposition}

\begin{proof}
Fix $i$ and $t\in[0,T]$. If
\[
  \norm{X_j(0)-X_i(0)}\le r_Khe^{-\Gamma_T},
\]
then by Lemma~\ref{lem:distance-distortion},
\[
  \norm{X_j(t)-X_i(t)}
  \le e^{\Gamma_T}r_Khe^{-\Gamma_T}=r_Kh.
\]
Thus the number of particles within distance $r_Kh$ of $X_i(t)$ is at least the number of particles within distance $r_Khe^{-\Gamma_T}$ of $X_i(0)$. By \eqref{eq:initial-smaller-occupancy},
\[
  M_i(X(t);h)\ge \alpha Nh^de^{-d\Gamma_T}.
\]
The proof of Proposition~\ref{prop:deterministic-occupancy} gives
\[
  q_{X(t)}(X_i(t))
  \ge
  \frac1N\{\alpha Nh^de^{-d\Gamma_T}\}\kappa_Kh^{-d}
  =\kappa_K\alpha e^{-d\Gamma_T}.
\]
Taking reciprocals and averaging over $i$ proves \eqref{eq:trajectory-R-bound}.\end{proof}

\begin{corollary}[I.i.d. initial data plus controlled distortion]
\label{cor:iid-trajectory-reciprocal}
Let $X_1(0),\ldots,X_N(0)$ be i.i.d. from a probability measure $\mu$. Suppose that, with
\[
  \widetilde h:=he^{-\Gamma_T},
\]
the lower local mass condition
\begin{equation}
  \mu(B(y,r_K\widetilde h))\ge p_0\widetilde h^{d}
  \qquad\text{for }\mu\text{-a.e. }y
  \label{eq:lower-local-mass-shrunk}
\end{equation}
holds. Assume Assumptions~\ref{ass:kernel-lower} and~\ref{ass:lipschitz-distortion}. Then, with probability at least
\begin{equation}
  1-N\exp\left(-\frac{p_0}{16}Nh^de^{-d\Gamma_T}\right),
  \label{eq:traj-high-prob}
\end{equation}
one has
\begin{equation}
  \sup_{0\le t\le T}\mathsf R_N(X(t))
  \le \frac{4e^{d\Gamma_T}}{\kappa_Kp_0}.
  \label{eq:traj-R-high-prob}
\end{equation}
Furthermore,
\begin{equation}
  \E\left[\sup_{0\le t\le T}\mathsf R_N(X(t))\right]
  \le
  \frac{e^{d\Gamma_T}}{\kappa_Kp_0}.
  \label{eq:traj-R-expectation}
\end{equation}
\end{corollary}

\begin{proof}
Apply Proposition~\ref{prop:iid-occupancy} at bandwidth $\widetilde h=he^{-\Gamma_T}$. With probability at least \eqref{eq:traj-high-prob}, for every $i$,
\[
  \#\{j:\norm{X_j(0)-X_i(0)}\le r_Khe^{-\Gamma_T}\}
  \ge \frac{p_0}{4}N(he^{-\Gamma_T})^d.
\]
This is \eqref{eq:initial-smaller-occupancy} with $\alpha=p_0/4$. Therefore Proposition~\ref{prop:trajectory-occupancy} gives \eqref{eq:traj-R-high-prob}.

For the expectation, let $M_i^0$ count the initial particles within
$r_K\widetilde h$ of $X_i(0)$, including particle $i$.
The distortion bound puts all these particles within $r_Kh$ of
$X_i(t)$ for every $t\le T$. Hence
\[
  q_{X(t)}(X_i(t))\ge\frac{\kappa_Kh^{-d}}N M_i^0,\qquad
  \sup_{0\le t\le T}\mathsf R_N(X(t))
  \le\frac{Nh^d}{\kappa_K}\frac1N\sum_i\frac1{M_i^0}.
\]
Conditionally on $X_i(0)=y$, the random variable $M_i^0-1$ is
binomial with parameters $N-1$ and
$p_{\widetilde h}(y)=\mu(B(y,r_K\widetilde h))$.
The binomial identity from Proposition~\ref{prop:iid-occupancy} gives
\[
  \E\left[\left.\frac1{M_i^0}\right|X_i(0)=y\right]
  =\frac{1-(1-p_{\widetilde h}(y))^N}
          {Np_{\widetilde h}(y)}
  \le\frac1{Np_0\widetilde h^d}.
\]
Taking expectations in the preceding bound for the supremum and
using $h^d/\widetilde h^d=e^{d\Gamma_T}$ proves
\eqref{eq:traj-R-expectation}.
\end{proof}

\begin{remark}[Use in Theorem~\ref{thm:main-rate}]
Under Corollary~\ref{cor:iid-trajectory-reciprocal},
Assumption~\ref{ass:reciprocal} holds with
\[
  \Lambda_T=\frac{e^{d\Gamma_T}}{\kappa_Kp_0}.
\]
If $\Gamma_T$ and $p_0^{-1}$ are bounded independently of $N,h,T$,
then this expected reciprocal bound is uniform. Establishing such
distortion control on $T=N$ is a separate stability problem.
For example, a nonnegative Lipschitz bound of order $h^{-2}$ integrated
over $[0,N]$ gives a bound of order $Nh^{-2}$ for $\Gamma_N$,
which does not give a uniform value of $e^{d\Gamma_N}$.
\end{remark}

\subsection{Reciprocal KDEs for Full-support Initial Densities}

The expectation bounds also clarify the role of the spatial tails.

\begin{proposition}[Full-support reciprocal growth]
\label{prop:full-support-reciprocal}
Let $Y_1,\ldots,Y_N$ be i.i.d. from a bounded continuous density $f$
that is positive on $\R^d$. Suppose $K$ is a nonnegative normalized
kernel with $0<K(0)<\infty$, $h_N\to0$, and $Nh_N^d\to\infty$.
Then
\[
  \E[\mathsf R_N(Y)]\longrightarrow\infty.
\]
\end{proposition}

\begin{proof}
For each fixed $y$, conditioning on $Y_1=y$ gives
\[
  \E[q_Y(Y_1)\mid Y_1=y]
  =\frac{K(0)}{Nh_N^d}
   +\left(1-\frac1N\right)(K_{h_N}*f)(y).
\]
The function $v\mapsto1/v$ is convex on $(0,\infty)$, since its
second derivative is $2/v^3>0$. Conditional Jensen's inequality
and exchangeability therefore imply
\[
  \E[\mathsf R_N(Y)]
  \ge
  \int_{\R^d}
  \frac{f(y)}
       {K(0)/(Nh_N^d)+(1-1/N)(K_{h_N}*f)(y)}\dd y.
\]
After the change of variables in the convolution,
\[
  (K_{h_N}*f)(y)=\int K(u)f(y-h_Nu)\dd u.
\]
Continuity gives $f(y-h_Nu)\to f(y)$ for every fixed $y,u$,
and the integrand is dominated by $\norm{f}_\infty K(u)$.
Dominated convergence yields $(K_{h_N}*f)(y)\to f(y)$.
The denominator in the displayed lower bound thus tends to
$f(y)>0$, so the integrand tends to one. Fatou's lemma gives
\[
  \liminf_{N\to\infty}\E[\mathsf R_N(Y)]
  \ge\int_{\R^d}1\dd y=\infty.
\]
This proves the proposition.
\end{proof}

\begin{remark}[Bandwidth optimization and initialization tails]
Proposition~\ref{prop:full-support-reciprocal} applies, in particular,
to Gaussian initialization. For kernels with $a_1>0$, the
coefficient $a_1\Lambda_N(h_N)$ in the conservative bound cannot
then be uniformly bounded in the regime of the proposition, since
the time supremum includes time zero. The non-asymptotic bound
continues to hold with this coefficient displayed. Its quantitative
growth must be included in any bandwidth optimization. Formula
\eqref{eq:reciprocal-tail-adaptive} provides an upper bound that
retains the local mass in the tails.
\end{remark}

\section{Finite-particle Rates for Non-conservative Drifting with Laplace Kernel}
\label{sec:nonconservative-laplace-drifting}

We now prove a continuous-time finite-particle rate for the non-conservative displacement-based drifting field with Laplace kernel. The main point is that, unlike the Gaussian kernel, the Laplace mean-shift field is not exactly a score-difference field. Nevertheless, it admits a useful ``sharp-score'' representation, which yields a finite-particle convergence theorem with an additional scale-mismatch residual. The sharp companion kernel construction used in this section follows the
sharp-normalization framework of \citet{franz2026drifting}, specialized to the
Laplace kernel. Our use of it is different: rather than replacing the original
Laplace normalization by the sharp normalization, we use the sharp kernel to
decompose the original non-conservative Laplace field into a preconditioned
sharp-score mismatch plus a scale-mismatch residual. 

Throughout this section, let
\[
  K_h(u)
  :=
  c_d h^{-d}\exp\!\left(-\frac{\norm{u}}{h}\right),
  \qquad u\in\R^d,
\]
where \(c_d>0\) is chosen so that \(\int_{\R^d}K_h(u)\dd u=1\). For a probability measure \(\alpha\) on \(\R^d\), define its Laplace KDE
\[
  Q_{\alpha,h}(z)
  :=
  \int_{\R^d}K_h(z-y)\,\alpha(\dd y),
\]
and the non-conservative displacement mean-shift vector
\[
  M_{\alpha,h}(z)
  :=
  \frac{\int_{\R^d}(y-z)K_h(z-y)\,\alpha(\dd y)}
       {Q_{\alpha,h}(z)}.
\]
The non-conservative Laplace drifting field from a model distribution \(\mu\) toward the data distribution \(\nu\) is
\[
  u_{\nu,\mu,h}^{\mathrm{Lap}}(z)
  :=
  M_{\nu,h}(z)-M_{\mu,h}(z).
\]

For a particle configuration \(x=(x_1,\ldots,x_N)\in(\R^d)^N\) with \(N\ge2\), let
\[
  \mu_x^N:=\frac1N\sum_{j=1}^N\delta_{x_j}.
\]
We analyze the self-masked, or leave-one-out, finite-particle version of the non-conservative Laplace drifting field. For each \(i\), define
\[
  \mu_{x,-i}^{N}
  :=
  \frac1{N-1}\sum_{j\ne i}\delta_{x_j},
\]
and
\[
  u_{x,-i}(z)
  :=
  M_{\nu,h}(z)-M_{\mu_{x,-i}^{N},h}(z).
\]

The leave-one-out form is both analytically convenient and faithful to the
self-masking used in the original drifting implementation. Analytically, it
ensures that \(u_{x,-i}\) depends on \(x_i\) only through the evaluation point
\(z=x_i\), so that
\[
  \nabla_{x_i}\cdot u_{x,-i}(x_i)
  =
  \nabla_z\cdot u_{x,-i}(z)\big|_{z=x_i},
\]
and the joint-entropy identity has no reciprocal self-interaction correction.
This matches the implementation of \citet[Algorithm 2]{deng2026drifting}, where generated
samples are reused as negative samples but the diagonal self-interaction is
masked. 

A full-KDE version without leave-one-out could also be studied, but it would
produce an additional self-interaction correction. For the exact Laplace kernel
this correction is less transparent because the field is non-conservative and
the kernel is not differentiable at the origin. We therefore analyze the
self-masked version, which is the version closest to the practical algorithm.

The continuous-time dynamics are
\begin{equation}
  \dot X_i(t)
  =
  u_{X(t),-i}(X_i(t)),
  \qquad i=1,\ldots,N.
  \label{eq:laplace-ooo-ode}
\end{equation}
The residual drift quantity we want to control is
\begin{equation}
  \mathsf V_N^{\mathrm{Lap}}(x)
  :=
  \frac1N\sum_{i=1}^N
  \norm{u_{x,-i}(x_i)}^2.
  \label{eq:laplace-particle-velocity}
\end{equation}
This is the mean squared norm of the non-conservative Laplace drift applied to the generated particles. Hence if one more explicit drift step is applied,
\[
  \widetilde x_i=x_i+\eta u_{x,-i}(x_i),
\]
then the identity coupling gives
\[
  W_2^2\left(
  \frac1N\sum_{i=1}^N\delta_{x_i},
  \frac1N\sum_{i=1}^N\delta_{\widetilde x_i}
  \right)
  \le
  \eta^2\mathsf V_N^{\mathrm{Lap}}(x).
\]
Thus a bound on \(\mathsf V_N^{\mathrm{Lap}}\) controls how much an additional non-conservative Laplace drifting correction changes the generated distribution.

\subsection{Sharp-score Representation of the Laplace Mean-shift Field}

The Laplace mean-shift field is not itself a KDE score, but it is a scaled score of a companion kernel. Define the sharp companion kernel
\begin{equation}
  L_h(u)
  :=
  h(\norm{u}+h)K_h(u).
  \label{eq:sharp-laplace-kernel}
\end{equation}
Let
\[
  R_{\alpha,h}(z)
  :=
  \int_{\R^d}L_h(z-y)\,\alpha(\dd y),
\]
and define the sharp-smoothed score
\[
  \sigma_{\alpha,h}(z)
  :=
  \nabla\log R_{\alpha,h}(z).
\]
Also define the scalar scale factor
\[
  a_{\alpha,h}(z)
  :=
  \frac{R_{\alpha,h}(z)}{Q_{\alpha,h}(z)}.
\]

\begin{lemma}[Sharp-score representation]
\label{lem:laplace-sharp-score}
For every probability measure \(\alpha\) for which the following quantities are well-defined,
\begin{equation}
  M_{\alpha,h}(z)
  =
  a_{\alpha,h}(z)\sigma_{\alpha,h}(z).
  \label{eq:mean-shift-sharp-score}
\end{equation}
Moreover,
\begin{equation}
  a_{\alpha,h}(z)
  =
  h\{\bar r_{\alpha,h}(z)+h\},
  \label{eq:laplace-scale-radius}
\end{equation}
where
\[
  \bar r_{\alpha,h}(z)
  :=
  \frac{\int_{\R^d}\norm{y-z}K_h(z-y)\,\alpha(\dd y)}
       {\int_{\R^d}K_h(z-y)\,\alpha(\dd y)}
\]
is the Laplace-weighted local mean radius.
\end{lemma}

\begin{proof}
Let \(r=\norm{u}\). Since
\[
  L_h(u)=h(r+h)c_dh^{-d}e^{-r/h},
\]
the radial derivative is
\[
  \frac{\dd}{\dd r}\{h(r+h)c_dh^{-d}e^{-r/h}\}
  =
  hc_dh^{-d}e^{-r/h}
  -
  h(r+h)c_dh^{-d}\frac1h e^{-r/h}.
\]
Combining the two terms,
\[
  \frac{\dd}{\dd r}\{h(r+h)c_dh^{-d}e^{-r/h}\}
  =
  c_dh^{1-d}e^{-r/h}
  -
  (r+h)c_dh^{-d}e^{-r/h}.
\]
Since
\[
  c_dh^{1-d}e^{-r/h}
  =
  hc_dh^{-d}e^{-r/h},
\]
we get
\[
  \frac{\dd}{\dd r}\{h(r+h)c_dh^{-d}e^{-r/h}\}
  =
  \{h-(r+h)\}c_dh^{-d}e^{-r/h}
  =
  -rK_h(u).
\]
Therefore, for \(u\ne0\),
\[
  \nabla_u L_h(u)
  =
  -rK_h(u)\frac{u}{r}
  =
  -uK_h(u).
\]
By continuity, the same identity holds at \(u=0\), because both sides vanish there. Taking \(u=z-y\), we obtain
\[
  \nabla_z L_h(z-y)
  =
  -(z-y)K_h(z-y)
  =
  (y-z)K_h(z-y).
\]
Hence
\[
  \nabla R_{\alpha,h}(z)
  =
  \int_{\R^d}\nabla_zL_h(z-y)\,\alpha(\dd y)
  =
  \int_{\R^d}(y-z)K_h(z-y)\,\alpha(\dd y).
\]
Dividing by \(Q_{\alpha,h}(z)\) gives
\[
  M_{\alpha,h}(z)
  =
  \frac{\nabla R_{\alpha,h}(z)}{Q_{\alpha,h}(z)}
  =
  \frac{R_{\alpha,h}(z)}{Q_{\alpha,h}(z)}
  \frac{\nabla R_{\alpha,h}(z)}{R_{\alpha,h}(z)}
  =
  a_{\alpha,h}(z)\sigma_{\alpha,h}(z).
\]
This proves \eqref{eq:mean-shift-sharp-score}.

It remains to prove \eqref{eq:laplace-scale-radius}. By definition of \(L_h\),
\[
  R_{\alpha,h}(z)
  =
  \int_{\R^d}h(\norm{z-y}+h)K_h(z-y)\,\alpha(\dd y).
\]
Thus
\[
  R_{\alpha,h}(z)
  =
  h\int_{\R^d}\norm{z-y}K_h(z-y)\,\alpha(\dd y)
  +
  h^2\int_{\R^d}K_h(z-y)\,\alpha(\dd y).
\]
The second integral is \(Q_{\alpha,h}(z)\). Therefore
\[
  \frac{R_{\alpha,h}(z)}{Q_{\alpha,h}(z)}
  =
  h
  \frac{\int \norm{z-y}K_h(z-y)\,\alpha(\dd y)}
       {Q_{\alpha,h}(z)}
  +
  h^2.
\]
Since \(\norm{z-y}=\norm{y-z}\), this is
\[
  a_{\alpha,h}(z)
  =
  h\bar r_{\alpha,h}(z)+h^2
  =
  h\{\bar r_{\alpha,h}(z)+h\}.
\]
The lemma follows.
\end{proof}

Let
\[
  Z_{\#,h}:=\int_{\R^d}L_h(u)\dd u.
\]
The normalized sharp-smoothed data density is
\[
  \rho_{\nu,h}^{\#}(z)
  :=
  \frac{R_{\nu,h}(z)}{Z_{\#,h}}.
\]
Since the normalizing constant does not depend on \(z\),
\[
  \nabla\log\rho_{\nu,h}^{\#}(z)
  =
  \nabla\log R_{\nu,h}(z)
  =
  \sigma_{\nu,h}(z).
\]
For the full empirical model, define
\[
  R_{x,h}:=R_{\mu_x^N,h},
  \qquad
  Q_{x,h}:=Q_{\mu_x^N,h},
  \qquad
  \sigma_{x,h}:=\nabla\log R_{x,h},
  \qquad
  a_{x,h}:=\frac{R_{x,h}}{Q_{x,h}}.
\]
Also define the full non-conservative Laplace field
\[
  u_x(z):=M_{\nu,h}(z)-M_{\mu_x^N,h}(z).
\]
By Lemma~\ref{lem:laplace-sharp-score},
\[
  u_x(z)
  =
  a_{\nu,h}(z)\sigma_{\nu,h}(z)
  -
  a_{x,h}(z)\sigma_{x,h}(z).
\]
Define the sharp score mismatch
\[
  b_x^{\#}(z)
  :=
  \sigma_{\nu,h}(z)-\sigma_{x,h}(z),
\]
and the Laplace scale-mismatch residual
\begin{equation}
  e_x(z)
  :=
  \{a_{\nu,h}(z)-a_{x,h}(z)\}\sigma_{\nu,h}(z).
  \label{eq:laplace-scale-residual}
\end{equation}
Then
\begin{equation}
  u_x(z)
  =
  a_{x,h}(z)b_x^{\#}(z)+e_x(z).
  \label{eq:laplace-decomposition}
\end{equation}
This decomposition separates the local model scale from its difference with the data scale. If the scales match, then $e_x=0$ and the drift is a positive scalar preconditioning of the sharp-score mismatch. More generally, absolute or relative bounds on $e_x$ give the coercivity estimates below.

\subsection{Entropy Identity for the Leave-one-out Non-conservative Field}

Let \(p_t^N\) denote the joint density of \(X(t)\). We measure it against the sharp-smoothed product target
\[
  (\rho_{\nu,h}^{\#})^{\otimes N}(x)
  =
  \prod_{i=1}^N\rho_{\nu,h}^{\#}(x_i),
\]
and define
\begin{equation}
  H_N^{\#}(t)
  :=
  \KL\left(p_t^N\,\middle\|\,(\rho_{\nu,h}^{\#})^{\otimes N}\right).
  \label{eq:sharp-joint-entropy}
\end{equation}
Let
\[
  \mathcal A_{\#}f(z)
  :=
  \divergence f(z)+\sigma_{\nu,h}(z)\cdot f(z)
\]
be the Stein divergence associated with \(\rho_{\nu,h}^{\#}\). Define
\begin{equation}
  \mathsf S_N^{\mathrm{Lap}}(x)
  :=
  \frac1N\sum_{i=1}^N
  \mathcal A_{\#}u_{x,-i}(x_i).
  \label{eq:laplace-empirical-stein}
\end{equation}

\begin{assumption}[Regularity for the leave-one-out Laplace flow]
\label{ass:laplace-regularity}
The ODE \eqref{eq:laplace-ooo-ode} has a unique solution on $[0,T]$,
the joint law admits a differentiable density $p_t^N$, and this density
solves the Liouville equation associated with
$x\mapsto(u_{x,-1}(x_1),\ldots,u_{x,-N}(x_N))$.
The leave-one-out fields are classically differentiable at their
evaluation centers along the trajectory. In particular, the
trajectory avoids particle collisions and, for discrete data,
data atoms at which the field is not differentiable.
The joint entropy and the integrals used below are finite. Differentiation
of the entropy under the integral is justified, the entropy is absolutely
continuous in time, and its derivative is integrable on $[0,T]$.
The joint-space integration by parts has vanishing boundary terms.
For spatial integrals of full fields, divergences are understood as
weak derivatives; all stated spatial integrations by parts are
assumed justified. Local higher differentiability is imposed only
where a quadrature lemma uses it.
\end{assumption}

\begin{theorem}[Joint entropy for leave-one-out Laplace drifting]
\label{thm:laplace-entropy-identity}
Under Assumption~\ref{ass:laplace-regularity},
\begin{equation}
  \frac{\dd}{\dd t}H_N^{\#}(t)
  =
  -N\E_t[\mathsf S_N^{\mathrm{Lap}}(X(t))].
  \label{eq:laplace-entropy-identity}
\end{equation}
\end{theorem}

\begin{proof}
Let
\[
  r_N^{\#}(x):=(\rho_{\nu,h}^{\#})^{\otimes N}(x).
\]
The joint density \(p_t^N\) satisfies
\[
  \partial_t p_t^N(x)
  +
  \sum_{i=1}^N
  \nabla_{x_i}\cdot
  \left\{
  p_t^N(x)u_{x,-i}(x_i)
  \right\}
  =0.
\]
By definition,
\[
  H_N^{\#}(t)
  =
  \int p_t^N(x)\log\frac{p_t^N(x)}{r_N^{\#}(x)}\dd x.
\]
Differentiating under the integral gives
\[
  \frac{\dd}{\dd t}H_N^{\#}(t)
  =
  \int \partial_t p_t^N(x)
  \left\{
  \log\frac{p_t^N(x)}{r_N^{\#}(x)}+1
  \right\}\dd x.
\]
Since \(p_t^N\) is a probability density,
\[
  \int \partial_t p_t^N(x)\dd x=0.
\]
Therefore
\[
  \frac{\dd}{\dd t}H_N^{\#}(t)
  =
  \int \partial_t p_t^N(x)
  \log\frac{p_t^N(x)}{r_N^{\#}(x)}\dd x.
\]
Using the Liouville equation,
\[
  \frac{\dd}{\dd t}H_N^{\#}(t)
  =
  -
  \sum_{i=1}^N
  \int
  \nabla_{x_i}\cdot
  \{p_t^N(x)u_{x,-i}(x_i)\}
  \log\frac{p_t^N(x)}{r_N^{\#}(x)}
  \dd x.
\]
Integrating by parts in \(x_i\),
\[
  \frac{\dd}{\dd t}H_N^{\#}(t)
  =
  \sum_{i=1}^N
  \int
  p_t^N(x)u_{x,-i}(x_i)\cdot
  \nabla_{x_i}\log\frac{p_t^N(x)}{r_N^{\#}(x)}
  \dd x.
\]
Now
\[
  \nabla_{x_i}\log r_N^{\#}(x)
  =
  \nabla\log\rho_{\nu,h}^{\#}(x_i)
  =
  \sigma_{\nu,h}(x_i).
\]
Hence
\[
  \nabla_{x_i}\log\frac{p_t^N(x)}{r_N^{\#}(x)}
  =
  \nabla_{x_i}\log p_t^N(x)-\sigma_{\nu,h}(x_i).
\]
Thus
\begin{align*}
  \frac{\dd}{\dd t}H_N^{\#}(t)
  &=
  \sum_{i=1}^N
  \int
  p_t^N u_{x,-i}(x_i)\cdot\nabla_{x_i}\log p_t^N
  \dd x \\
  &\qquad
  -
  \sum_{i=1}^N
  \int
  p_t^N u_{x,-i}(x_i)\cdot\sigma_{\nu,h}(x_i)
  \dd x.
\end{align*}
Since
\[
  p_t^N\nabla_{x_i}\log p_t^N=\nabla_{x_i}p_t^N,
\]
the first integral in the last display satisfies
\[
  \int
  p_t^N u_{x,-i}(x_i)\cdot\nabla_{x_i}\log p_t^N
  \dd x
  =
  \int
  u_{x,-i}(x_i)\cdot\nabla_{x_i}p_t^N
  \dd x.
\]
Integrating by parts again,
\[
  \int
  u_{x,-i}(x_i)\cdot\nabla_{x_i}p_t^N
  \dd x
  =
  -
  \int
  p_t^N
  \nabla_{x_i}\cdot u_{x,-i}(x_i)
  \dd x.
\]
Because \(u_{x,-i}\) depends on \(x_i\) only through the evaluation point \(z=x_i\), not through the leave-one-out model measure \(\mu_{x,-i}^N\), we have
\[
  \nabla_{x_i}\cdot u_{x,-i}(x_i)
  =
  \divergence_z u_{x,-i}(z)\big|_{z=x_i}.
\]
Therefore
\[
  \frac{\dd}{\dd t}H_N^{\#}(t)
  =
  -
  \E_t\left[
  \sum_{i=1}^N
  \left\{
  \divergence u_{X(t),-i}(X_i(t))
  +
  \sigma_{\nu,h}(X_i(t))\cdot u_{X(t),-i}(X_i(t))
  \right\}
  \right].
\]
The expression inside braces is exactly
\[
  \mathcal A_{\#}u_{X(t),-i}(X_i(t)).
\]
Hence
\[
  \frac{\dd}{\dd t}H_N^{\#}(t)
  =
  -
  N\E_t[\mathsf S_N^{\mathrm{Lap}}(X(t))],
\]
which proves the identity.
\end{proof}

\subsection{Laplace Coercivity with A Scale-mismatch Residual}

Let
\[
  \rho_{x,h}^{\#}(z)
  :=
  \frac{R_{x,h}(z)}{Z_{\#,h}}
\]
be the full sharp-smoothed empirical density. Define the sharp-smoothed drift energy
\begin{equation}
  \mathcal V_h^{\mathrm{Lap}}(x)
  :=
  \int_{\R^d}
  \norm{u_x(z)}^2\rho_{x,h}^{\#}(z)\dd z,
  \label{eq:laplace-pop-energy}
\end{equation}
and the sharp-smoothed Stein drift
\begin{equation}
  \mathcal J_h^{\mathrm{Lap}}(x)
  :=
  \int_{\R^d}
  \mathcal A_{\#}u_x(z)\rho_{x,h}^{\#}(z)\dd z.
  \label{eq:laplace-pop-stein}
\end{equation}

\begin{lemma}[Sharp-smoothed Stein identity]
\label{lem:laplace-stein-identity}
For every fixed configuration \(x\) for which the integrations by parts are justified,
\begin{equation}
  \mathcal J_h^{\mathrm{Lap}}(x)
  =
  \int_{\R^d}
  b_x^{\#}(z)\cdot u_x(z)\rho_{x,h}^{\#}(z)\dd z.
  \label{eq:laplace-stein-projection}
\end{equation}
\end{lemma}

\begin{proof}
By definition,
\[
  \mathcal A_{\#}u_x
  =
  \divergence u_x+\sigma_{\nu,h}\cdot u_x.
\]
Therefore
\[
  \mathcal J_h^{\mathrm{Lap}}(x)
  =
  \int \rho_{x,h}^{\#}\divergence u_x
  +
  \int \rho_{x,h}^{\#}\sigma_{\nu,h}\cdot u_x.
\]
For the first term, integration by parts gives
\[
  \int \rho_{x,h}^{\#}\divergence u_x
  =
  -
  \int \nabla\rho_{x,h}^{\#}\cdot u_x.
\]
Since
\[
  \nabla\rho_{x,h}^{\#}
  =
  \rho_{x,h}^{\#}\nabla\log\rho_{x,h}^{\#}
  =
  \rho_{x,h}^{\#}\sigma_{x,h},
\]
we obtain
\[
  \int \rho_{x,h}^{\#}\divergence u_x
  =
  -
  \int \rho_{x,h}^{\#}\sigma_{x,h}\cdot u_x.
\]
Thus
\[
  \mathcal J_h^{\mathrm{Lap}}(x)
  =
  \int
  \rho_{x,h}^{\#}
  (\sigma_{\nu,h}-\sigma_{x,h})\cdot u_x.
\]
By definition, \(b_x^{\#}=\sigma_{\nu,h}-\sigma_{x,h}\). Hence
\[
  \mathcal J_h^{\mathrm{Lap}}(x)
  =
  \int
  b_x^{\#}\cdot u_x\,\rho_{x,h}^{\#},
\]
which proves the claim.
\end{proof}

\begin{remark}[Spatial regularity of the Laplace field]
\label{rem:laplace-spatial-regularity}
The Laplace mean-shift fields are locally Lipschitz, and their spatial
divergences in integrals are understood as weak derivatives, represented
almost everywhere. At a particle center, the leave-one-out field
$u_{x,-i}$ has no self-kernel cusp. Under
Assumption~\ref{ass:laplace-regularity}, its divergence at $x_i$ is a
classical derivative.

The full field $u_x$ is used through integrals. Indeed, write its model
mean shift near $x_i$ as $P(z)/Q(z)$, omitting the common factor $1/N$.
The self-contribution to $Q$ is
\[
  K_h(z-x_i)=K_h(0)
  \left\{1-\frac{\norm{z-x_i}}h+O(\norm{z-x_i}^2)\right\},
\]
whereas the self-contribution to $P$ is
$(x_i-z)K_h(z-x_i)$, which is differentiable at $x_i$.
When the remaining contributions are smooth and $P(x_i)\ne0$, the
quotient contains a nonzero multiple of $\norm{z-x_i}$ and need not be
differentiable there. The sharp-smoothed Stein identity only uses the
almost-everywhere divergence and its integration by parts.
\end{remark}

\begin{remark}[Growth of the empirical scale]
\label{rem:laplace-scale-growth}
The scale factor has the universal lower bound $a_{x,h}\ge h^2$.
Its upper bound depends on the evaluation region. To see this, let
$R_x:=\max_j\norm{x_j}$ and define
\[
  w_j(z):=\frac{K_h(z-x_j)}{\sum_kK_h(z-x_k)}.
\]
Since the weights are positive and sum to one, the triangle inequality gives
\[
  \norm{z}-R_x
  \le \sum_jw_j(z)\norm{z-x_j}
  =\bar r_{x,h}(z)
  \le \norm{z}+R_x.
\]
Consequently,
\[
  h(\norm{z}-R_x+h)\le a_{x,h}(z)
  \le h(\norm{z}+R_x+h).
\]
In particular, $a_{x,h}$ grows without bound as $\norm{z}\to\infty$.
The weighted energy below retains the local scale inside the integral.
\end{remark}

Define the weighted drift energy and weighted scale mismatch by
\begin{align}
  \mathcal W_h^{\mathrm{Lap}}(x)
  &:=
  \int_{\R^d}\frac{\norm{u_x(z)}^2}{a_{x,h}(z)}
       \rho_{x,h}^{\#}(z)\dd z,
  \label{eq:laplace-weighted-energy}\\
  \mathcal D_h^{\mathrm{Lap}}(x)
  &:=
  \int_{\R^d}\frac{\norm{e_x(z)}^2}{a_{x,h}(z)}
       \rho_{x,h}^{\#}(z)\dd z.
  \label{eq:laplace-weighted-mismatch}
\end{align}

\begin{assumption}[Weighted Laplace scale alignment]
\label{ass:laplace-scale-alignment}
For all configurations visited by the flow on $[0,T]$, the quantities
$\mathcal V_h^{\mathrm{Lap}}$, $\mathcal J_h^{\mathrm{Lap}}$,
$\mathcal W_h^{\mathrm{Lap}}$, and $\mathcal D_h^{\mathrm{Lap}}$ are finite.
There is a constant $\Delta_h\ge0$ such that
\begin{equation}
  \mathcal D_h^{\mathrm{Lap}}(x)\le\Delta_h^2.
  \label{eq:laplace-residual-assumption}
\end{equation}
The constants may depend on $N,h,T$ and the initial law; any uniformity
used in a rate statement is specified separately.
\end{assumption}

\begin{lemma}[Weighted Laplace coercivity]
\label{lem:laplace-coercivity}
Under Assumption~\ref{ass:laplace-scale-alignment},
\begin{equation}
  \mathcal J_h^{\mathrm{Lap}}(x)
  \ge \frac12\mathcal W_h^{\mathrm{Lap}}(x)
       -\frac12\mathcal D_h^{\mathrm{Lap}}(x)
  \ge \frac12\mathcal W_h^{\mathrm{Lap}}(x)-\frac12\Delta_h^2.
  \label{eq:laplace-coercivity}
\end{equation}
If, for some $0\le\theta_h<1$,
\begin{equation}
  \mathcal D_h^{\mathrm{Lap}}(x)
  \le\theta_h\mathcal W_h^{\mathrm{Lap}}(x),
  \label{eq:laplace-relative-alignment}
\end{equation}
then
\begin{equation}
  \mathcal J_h^{\mathrm{Lap}}(x)
  \ge(1-\sqrt{\theta_h})\mathcal W_h^{\mathrm{Lap}}(x).
  \label{eq:laplace-relative-coercivity}
\end{equation}
\end{lemma}

\begin{proof}
Since $a_{x,h}>0$, the decomposition
$u_x=a_{x,h}b_x^\#+e_x$ gives
\[
  b_x^\#\cdot u_x
  =\frac{\norm{u_x}^2}{a_{x,h}}
   -\frac{e_x\cdot u_x}{a_{x,h}}.
\]
For any vectors $v,w$, the inequality
$0\le\norm{v-w}^2=\norm{v}^2+\norm{w}^2-2v\cdot w$ implies
$v\cdot w\le(\norm{v}^2+\norm{w}^2)/2$. Applying it to $e_x,u_x$,
dividing by $a_{x,h}$, and integrating gives
\[
  \int\frac{e_x\cdot u_x}{a_{x,h}}\rho_{x,h}^\#
  \le\frac12\mathcal D_h^{\mathrm{Lap}}(x)
      +\frac12\mathcal W_h^{\mathrm{Lap}}(x).
\]
Lemma~\ref{lem:laplace-stein-identity} now yields the first inequality
in \eqref{eq:laplace-coercivity}. The bound
$\mathcal D_h^{\mathrm{Lap}}\le\Delta_h^2$ gives the second.

For the relative estimate, Cauchy--Schwarz with respect to the measure
$a_{x,h}^{-1}\rho_{x,h}^\#\dd z$ gives
\[
  \abs{\int\frac{e_x\cdot u_x}{a_{x,h}}\rho_{x,h}^\#\dd z}
  \le
  \sqrt{\mathcal D_h^{\mathrm{Lap}}(x)
        \mathcal W_h^{\mathrm{Lap}}(x)}.
\]
Using \eqref{eq:laplace-relative-alignment}, the right side is at most
$\sqrt{\theta_h}\mathcal W_h^{\mathrm{Lap}}(x)$. Substituting this bound
in the identity for $\mathcal J_h^{\mathrm{Lap}}$ proves
\eqref{eq:laplace-relative-coercivity}.
\end{proof}

\begin{assumption}[Local scale and tail energy]
\label{ass:laplace-localization}
For every configuration visited on $[0,T]$, there is a measurable set
$\Omega_x\subset\R^d$ such that
\begin{equation}
  a_{x,h}(z)\le L_h\quad(z\in\Omega_x),\qquad
  \int_{\Omega_x^c}\norm{u_x(z)}^2\rho_{x,h}^\#(z)\dd z
  \le\tau_{h,N},
  \label{eq:laplace-scale-bounds}
\end{equation}
where $0<L_h<\infty$ and $\tau_{h,N}\ge0$ are independent of the
configuration and time on the stated horizon.
\end{assumption}

\begin{corollary}[Localized unweighted coercivity]
\label{cor:laplace-local-coercivity}
Under Assumptions~\ref{ass:laplace-scale-alignment}
and~\ref{ass:laplace-localization},
\begin{equation}
  \mathcal V_h^{\mathrm{Lap}}(x)
  \le 2L_h\mathcal J_h^{\mathrm{Lap}}(x)
       +L_h\mathcal D_h^{\mathrm{Lap}}(x)+\tau_{h,N}.
  \label{eq:laplace-local-coercivity}
\end{equation}
More generally, if $a_{x,h}\ge\lambda_h>0$ everywhere and
$\int\norm{e_x}^2\rho_{x,h}^\#\le\delta_h^2$, then
\[
  \mathcal J_h^{\mathrm{Lap}}(x)
  \ge\frac{\mathcal V_h^{\mathrm{Lap}}(x)-\tau_{h,N}}{2L_h}
       -\frac{\delta_h^2}{2\lambda_h}.
\]
For the Laplace kernel one may always take $\lambda_h=h^2$.
\end{corollary}

\begin{proof}
Positivity of the weighted integrand and the local upper bound give
\[
  \mathcal W_h^{\mathrm{Lap}}(x)
  \ge\int_{\Omega_x}\frac{\norm{u_x}^2}{a_{x,h}}\rho_{x,h}^\#
  \ge\frac1{L_h}\int_{\Omega_x}\norm{u_x}^2\rho_{x,h}^\#
  \ge\frac{\mathcal V_h^{\mathrm{Lap}}(x)-\tau_{h,N}}{L_h}.
\]
Substitute this inequality in the first bound of
\eqref{eq:laplace-coercivity} and multiply by $2L_h$ to obtain
\eqref{eq:laplace-local-coercivity}. Under the additional lower bound,
\[
  \mathcal D_h^{\mathrm{Lap}}(x)
  =\int\frac{\norm{e_x}^2}{a_{x,h}}\rho_{x,h}^\#
  \le\frac1{\lambda_h}\int\norm{e_x}^2\rho_{x,h}^\#
  \le\frac{\delta_h^2}{\lambda_h}.
\]
Using this in the same inequality proves the final statement.
\end{proof}

\subsection{From the Sharp-smoothed Coercivity to Particle Rates}

The entropy identity controls the leave-one-out, center-evaluated Stein
quantity. The weighted coercivity estimate controls a full-field integral.
The following assumptions record their comparison and the comparison of
the two unweighted energies.

\begin{assumption}[Quadrature and leave-one-out approximation]
\label{ass:laplace-quadrature}
For all configurations visited on $[0,T]$, there are constants
$\varepsilon_{S,h,N},\varepsilon_{V,h,N}\ge0$ such that
\begin{align}
  \abs{\mathsf S_N^{\mathrm{Lap}}(x)
             -\mathcal J_h^{\mathrm{Lap}}(x)}
  &\le\varepsilon_{S,h,N},
  \label{eq:laplace-S-quadrature}\\
  \abs{\mathsf V_N^{\mathrm{Lap}}(x)
             -\mathcal V_h^{\mathrm{Lap}}(x)}
  &\le\varepsilon_{V,h,N}.
  \label{eq:laplace-V-quadrature}
\end{align}
\end{assumption}

\begin{remark}
\label{rem:loo-errors-assumed}
The point evaluations in Assumption~\ref{ass:laplace-quadrature} involve
only the leave-one-out fields. The full field appears in integrals,
where its weak divergence is sufficient.
Lemma~\ref{lem:laplace-quadrature-sufficient} bounds the approximation
errors using local smoothness of each leave-one-out field near its
evaluation center, explicit quadrature tails, and integrated differences
of field values. These quantities may depend on $N,h,T$.
The rate statements retain their dependence; denominator control at the
centers alone does not give a uniform approximation rate.
\end{remark}

\begin{theorem}[Continuous-time finite-particle rates for Laplace drifting]
\label{thm:laplace-main-rate}
Let Assumptions~\ref{ass:laplace-regularity},
\ref{ass:laplace-scale-alignment}, and
\ref{ass:laplace-quadrature} hold on $[0,T]$. Then
\begin{equation}
  \frac1T\int_0^T
  \E_t[\mathcal W_h^{\mathrm{Lap}}(X(t))]\dd t
  \le\frac{2H_N^\#(0)}{NT}
       +\Delta_h^2+2\varepsilon_{S,h,N}.
  \label{eq:laplace-weighted-rate}
\end{equation}
If Assumption~\ref{ass:laplace-localization} also holds, then
\begin{equation}
  \frac1T\int_0^T
  \E_t[\mathsf V_N^{\mathrm{Lap}}(X(t))]\dd t
  \le\frac{2L_hH_N^\#(0)}{NT}
      +L_h\Delta_h^2+2L_h\varepsilon_{S,h,N}
      +\tau_{h,N}+\varepsilon_{V,h,N}.
  \label{eq:laplace-main-rate}
\end{equation}
In particular, if $H_N^\#(0)\le\kappa_0N$ and $T=N$, then
\begin{equation}
  \frac1N\int_0^N
  \E_t[\mathsf V_N^{\mathrm{Lap}}(X(t))]\dd t
  \le\frac{2\kappa_0L_h}{N}
      +L_h\Delta_h^2+2L_h\varepsilon_{S,h,N}
      +\tau_{h,N}+\varepsilon_{V,h,N}.
  \label{eq:laplace-main-rate-TN}
\end{equation}
\end{theorem}

\begin{proof}
Integrating Theorem~\ref{thm:laplace-entropy-identity} gives
\[
  H_N^\#(T)-H_N^\#(0)
  =-N\int_0^T\E_t[\mathsf S_N^{\mathrm{Lap}}(X(t))]\dd t.
\]
Since $H_N^\#(T)\ge0$, rearrangement and division by $NT$ yield
\begin{equation}
  \frac1T\int_0^T
  \E_t[\mathsf S_N^{\mathrm{Lap}}(X(t))]\dd t
  \le\frac{H_N^\#(0)}{NT}.
  \label{eq:laplace-S-rate}
\end{equation}
By Lemma~\ref{lem:laplace-coercivity} and
Assumption~\ref{ass:laplace-quadrature},
\[
  \mathcal W_h^{\mathrm{Lap}}(x)
  \le 2\mathcal J_h^{\mathrm{Lap}}(x)+\mathcal D_h^{\mathrm{Lap}}(x)
  \le 2\mathsf S_N^{\mathrm{Lap}}(x)
       +2\varepsilon_{S,h,N}+\Delta_h^2.
\]
Take expectations, integrate over $[0,T]$, divide by $T$, and apply
\eqref{eq:laplace-S-rate}. This proves
\eqref{eq:laplace-weighted-rate}.

For the center-velocity estimate,
Corollary~\ref{cor:laplace-local-coercivity} and both comparisons in
Assumption~\ref{ass:laplace-quadrature} give
\begin{align*}
  \mathsf V_N^{\mathrm{Lap}}(x)
  &\le \mathcal V_h^{\mathrm{Lap}}(x)+\varepsilon_{V,h,N}\\
  &\le 2L_h\mathcal J_h^{\mathrm{Lap}}(x)
        +L_h\mathcal D_h^{\mathrm{Lap}}(x)
        +\tau_{h,N}+\varepsilon_{V,h,N}\\
  &\le 2L_h\mathsf S_N^{\mathrm{Lap}}(x)
        +2L_h\varepsilon_{S,h,N}+L_h\Delta_h^2
        +\tau_{h,N}+\varepsilon_{V,h,N}.
\end{align*}
After expectation and time averaging, \eqref{eq:laplace-S-rate}
proves \eqref{eq:laplace-main-rate}. Finally,
$H_N^\#(0)/(NT)\le\kappa_0/N$ when $T=N$, which gives
\eqref{eq:laplace-main-rate-TN}.
\end{proof}

\begin{corollary}[Rates under relative scale alignment]
\label{cor:laplace-relative-rate}
Suppose the regularity and approximation assumptions of
Theorem~\ref{thm:laplace-main-rate} hold, and replace the absolute
mismatch bound by \eqref{eq:laplace-relative-alignment}, uniformly
along the trajectory, with $0\le\theta_h<1$. Then
\[
  \frac1T\int_0^T
  \E_t[\mathcal W_h^{\mathrm{Lap}}(X(t))]\dd t
  \le\frac1{1-\sqrt{\theta_h}}
       \left\{\frac{H_N^\#(0)}{NT}+\varepsilon_{S,h,N}\right\}.
\]
Under Assumption~\ref{ass:laplace-localization}, also
\[
  \frac1T\int_0^T
  \E_t[\mathsf V_N^{\mathrm{Lap}}(X(t))]\dd t
  \le\frac{L_h}{1-\sqrt{\theta_h}}
       \left\{\frac{H_N^\#(0)}{NT}+\varepsilon_{S,h,N}\right\}
       +\tau_{h,N}+\varepsilon_{V,h,N}.
\]
\end{corollary}

\begin{proof}
The relative coercivity bound and the Stein comparison imply
\[
  (1-\sqrt{\theta_h})\mathcal W_h^{\mathrm{Lap}}(x)
  \le\mathcal J_h^{\mathrm{Lap}}(x)
  \le\mathsf S_N^{\mathrm{Lap}}(x)+\varepsilon_{S,h,N}.
\]
Expectation, time integration, and \eqref{eq:laplace-S-rate} prove
the first bound after division by $1-\sqrt{\theta_h}>0$.
The localization argument gives
$\mathcal V_h^{\mathrm{Lap}}\le L_h\mathcal W_h^{\mathrm{Lap}}+\tau_{h,N}$.
Combining it with
$\mathsf V_N^{\mathrm{Lap}}\le\mathcal V_h^{\mathrm{Lap}}+\varepsilon_{V,h,N}$
and the first bound proves the second.
\end{proof}

\begin{remark}[Meaning of the Laplace rates]
The weighted estimate retains the model's local Laplace scale inside
the energy. The center-velocity estimate uses a local upper scale and
an explicit tail energy to recover the quantity applied in an additional
drift step. Absolute scale alignment contributes the term
$L_h\Delta_h^2$. Relative alignment allows the mismatch to be absorbed
into the dissipative term, as in Corollary~\ref{cor:laplace-relative-rate}.
Thus the absolute mismatch bound describes one sufficient route to a
small residual. It is not a lower bound on the limiting drift, and
does not imply a positive asymptotic error floor.
\end{remark}

\begin{remark}[Local scale variation and circulation]
\label{rem:laplace-curl-and-delta}
The sharp score $\sigma_{\alpha,h}$ is a spatial gradient. On a
smooth two-dimensional open set, the decomposition gives
\[
  \operatorname{curl}u_x
  =\nabla a_{x,h}\times b_x^\#
   +\nabla(a_{\nu,h}-a_{x,h})\times\sigma_{\nu,h},
\]
where $\nabla a\times f:=\partial_1a\,f_2-\partial_2a\,f_1$.
Indeed, the product rule gives
$\operatorname{curl}(af)=\nabla a\times f+a\operatorname{curl}f$,
and both $b_x^\#$ and $\sigma_{\nu,h}$ have zero curl.
The mismatch $e_x=(a_{\nu,h}-a_{x,h})\sigma_{\nu,h}$ depends on
differences between the data and model scales. Circulation depends
on spatial derivatives of those scales. These are related but distinct
quantities. Even when $e_x=0$, the term
$\nabla a_{x,h}\times b_x^\#$ may be nonzero. For the Gaussian
companion kernel, $a_{\alpha,h}\equiv h^2$; the mismatch and both
scale-derivative terms vanish.
\end{remark}

\subsection{Checkable Forms of The Laplace Assumptions}

We first record the normalization and tails of the sharp kernel.
Let
\[
  \overline L_h(u):=\frac{L_h(u)}{Z_{\#,h}}.
\]

\begin{lemma}[Sharp-kernel moments, scores, and tails]
\label{lem:sharp-kernel-moments}
The Laplace companion kernel satisfies
\begin{equation}
  Z_{\#,h}=(d+1)h^2,\qquad
  \int\norm{u}^2\overline L_h(u)\dd u=d(d+3)h^2.
  \label{eq:sharp-kernel-moments}
\end{equation}
For every probability measure $\alpha$,
\[
  \norm{\sigma_{\alpha,h}(z)}\le h^{-1}.
\]
For $s\ge0$, let
\[
  p_d(s):=
  e^{-s}\left\{
    \frac{d}{d+1}\sum_{k=0}^{d}\frac{s^k}{k!}
    +\frac1{d+1}\sum_{k=0}^{d-1}\frac{s^k}{k!}
  \right\}.
\]
Then
\begin{equation}
  \int_{\norm{u}>hs}\overline L_h(u)\dd u=p_d(s).
  \label{eq:sharp-kernel-tail}
\end{equation}
\end{lemma}

\begin{proof}
If $U$ has density $K_h$, radial integration gives
\[
  \E[g(\norm{U}/h)]
  =c_d\abs{\mathbb S^{d-1}}\int_0^\infty
                  g(s)s^{d-1}e^{-s}\dd s
\]
for every bounded measurable $g$.
For $m\ge1$, integration by parts gives
\[
  \int_0^\infty s^me^{-s}\dd s
  =m\int_0^\infty s^{m-1}e^{-s}\dd s,
\]
with initial value $\int_0^\infty e^{-s}\dd s=1$.
Thus the integral of $s^me^{-s}$ is $m!$.
Taking $g=1$ in the radial formula yields
$c_d\abs{\mathbb S^{d-1}}=1/(d-1)!$.
Consequently $S:=\norm{U}/h$ has density
$s^{d-1}e^{-s}/(d-1)!$ and
$\E S^k=(d+k-1)!/(d-1)!$ for nonnegative integers $k$.
It follows that
\[
  \E S=d,\quad
  \E S^2=d(d+1),\quad
  \E S^3=d(d+1)(d+2).
\]
Since $L_h(u)=h^2(1+\norm{u}/h)K_h(u)$,
\[
  Z_{\#,h}=h^2(1+\E S)=(d+1)h^2,
\]
and
\[
  \int\norm{u}^2\overline L_h(u)\dd u
  =h^2\frac{\E S^2+\E S^3}{d+1}
  =d(d+3)h^2.
\]
The sharp-kernel gradient satisfies
\[
  \frac{\nabla L_h(u)}{L_h(u)}
  =-\frac{u}{h(\norm{u}+h)}.
\]
Its norm is at most $1/h$, including at $u=0$ by continuity.
Furthermore,
\[
  \sigma_{\alpha,h}(z)
  =\int\frac{L_h(z-y)}{R_{\alpha,h}(z)}
        \frac{\nabla L_h(z-y)}{L_h(z-y)}\,\alpha(\dd y).
\]
The coefficients define a probability measure, so the triangle
inequality gives $\norm{\sigma_{\alpha,h}}\le1/h$.

For the tail, radial integration gives
\[
  \int_{\norm{u}>hs}\overline L_h(u)\dd u
  =\frac1{(d+1)(d-1)!}
    \int_s^\infty(t^d+t^{d-1})e^{-t}\dd t.
\]
Writing $I_m(s)=\int_s^\infty t^me^{-t}\dd t$, integration by parts
gives $I_m(s)=s^me^{-s}+mI_{m-1}(s)$ and $I_0(s)=e^{-s}$.
Successive substitution in this recurrence yields
$I_m(s)=m!e^{-s}\sum_{k=0}^m s^k/k!$.
Substituting the formulas for $I_d$ and $I_{d-1}$ proves
\eqref{eq:sharp-kernel-tail}.
\end{proof}

\begin{corollary}[An integrated shell-alignment condition]
\label{cor:laplace-shell-alignment}
Suppose, uniformly along the trajectory,
\[
  \norm{\sigma_{\nu,h}}_\infty\le G_h,\qquad
  \int\abs{\bar r_{\nu,h}(z)-\bar r_{x,h}(z)}^2
            \rho_{x,h}^\#(z)\dd z\le\delta_{r,h}^2.
\]
Then the weighted mismatch bound holds with
\[
  \Delta_h^2=G_h^2\delta_{r,h}^2.
\]
One may always use $G_h=h^{-1}$, while a smaller bound can be used
when the data score has additional regularity.
\end{corollary}

\begin{proof}
The sharp-score decomposition gives
\[
  e_x(z)=h\{\bar r_{\nu,h}(z)-\bar r_{x,h}(z)\}\sigma_{\nu,h}(z).
\]
Since $a_{x,h}\ge h^2$,
\[
  \frac{\norm{e_x(z)}^2}{a_{x,h}(z)}
  \le
  G_h^2\abs{\bar r_{\nu,h}(z)-\bar r_{x,h}(z)}^2.
\]
Integration against $\rho_{x,h}^\#$ proves the stated mismatch bound.
The universal choice of $G_h$ follows from
Lemma~\ref{lem:sharp-kernel-moments}.
\end{proof}

\begin{proposition}[Confinement and sharp-kernel tails]
\label{prop:laplace-confinement}
Suppose $\nu$ is supported in $B(0,R_\nu)$ and, along the trajectory,
$\max_i\norm{X_i(t)}\le R_T$ with deterministic $R_T<\infty$.
For every $r>0$, Assumption~\ref{ass:laplace-localization} holds with
\begin{equation}
  \Omega_x=B(0,R_T+r),\qquad
  L_h=h(2R_T+r+h),\qquad
  \tau_{h,N}=(R_\nu+R_T)^2p_d(r/h).
  \label{eq:laplace-confinement-constants}
\end{equation}
\end{proposition}

\begin{proof}
For $z\in\Omega_x$, each distance $\norm{z-x_j}$ is at most
$2R_T+r$. Its kernel-weighted average satisfies the same upper bound.
Thus $a_{x,h}(z)\le h(2R_T+r+h)$.

Write $M_{\alpha,h}(z)=m_{\alpha,h}(z)-z$, where
$m_{\alpha,h}$ is the kernel-weighted mean of $y$.
The support bounds imply
$\norm{m_{\nu,h}(z)}\le R_\nu$ and
$\norm{m_{\mu_x^N,h}(z)}\le R_T$.
The $-z$ terms cancel in their difference, giving
$\norm{u_x(z)}\le R_\nu+R_T$.
Also,
\[
  \rho_{x,h}^\#(z)=\frac1N\sum_i\overline L_h(z-x_i).
\]
If $\norm{z}>R_T+r$, then $\norm{z-x_i}>r$ for every $i$.
Therefore
\begin{align*}
  \int_{\Omega_x^c}\rho_{x,h}^\#(z)\dd z
  &\le\frac1N\sum_i
       \int_{\norm{z-x_i}>r}\overline L_h(z-x_i)\dd z\\
  &=p_d(r/h).
\end{align*}
Multiplying by $(R_\nu+R_T)^2$ proves the tail-energy bound.
\end{proof}

\begin{remark}[Finite-horizon confinement]
For bounded data support, a deterministic initial bound
$\max_i\norm{X_i(0)}\le R_0$ gives a finite-horizon bound without
assuming a uniform density lower bound. Let $R(t)=\max_i\norm{X_i(t)}$.
The leave-one-out weighted model mean lies in the convex hull of the
other centers, so
$\norm{u_{X(t),-i}(X_i(t))}\le R_\nu+R(t)$.
Integration for each particle, followed by maximization, yields
\[
  R(t)\le R_0+\int_0^t\{R_\nu+R(s)\}\dd s.
\]
Gronwall's inequality therefore gives
$R(t)\le(R_0+R_\nu)e^t-R_\nu$.
This ensures finite constants on a fixed horizon where the regular
flow exists. Uniform constants on a growing horizon require a
stronger confinement estimate.
\end{remark}

We next compare the leave-one-out particle sums with the full-field
integrals. For a fixed configuration, set
\begin{align*}
  k_i(z)&:=\overline L_h(z-x_i),
  &f_i(z)&:=\mathcal A_\#u_{x,-i}(z),\\
  g_i(z)&:=\norm{u_{x,-i}(z)}^2,
  &d_i(z)&:=u_{x,-i}(z)-u_x(z).
\end{align*}
The functions $f_i$ are defined almost everywhere and take their
classical values at $x_i$. Define
\begin{align*}
  q_S(x)&:=\frac1N\sum_i
      \abs{f_i(x_i)-\int k_i(z)f_i(z)\dd z},\\
  q_V(x)&:=\frac1N\sum_i
      \abs{g_i(x_i)-\int k_i(z)g_i(z)\dd z},\\
  \ell_1(x)&:=\frac1N\sum_i\int k_i(z)\norm{d_i(z)}\dd z,\\
  \ell_V(x)&:=\frac1N\sum_i\int k_i(z)
       \{\norm{u_{x,-i}(z)}+\norm{u_x(z)}\}\norm{d_i(z)}\dd z.
\end{align*}

\begin{lemma}[Integrated leave-one-out comparison]
\label{lem:laplace-quadrature-sufficient}
Assume the displayed integrals are finite and integration by parts
against each $k_i$ is justified for $d_i$. Then
\begin{align}
  \abs{\mathsf S_N^{\mathrm{Lap}}(x)-\mathcal J_h^{\mathrm{Lap}}(x)}
  &\le q_S(x)+\frac2h\ell_1(x),
  \label{eq:laplace-integrated-S}\\
  \abs{\mathsf V_N^{\mathrm{Lap}}(x)-\mathcal V_h^{\mathrm{Lap}}(x)}
  &\le q_V(x)+\ell_V(x).
  \label{eq:laplace-integrated-V}
\end{align}
\end{lemma}

\begin{proof}
Since $\rho_{x,h}^\#=N^{-1}\sum_i k_i$ and the Stein operator is
linear,
\[
  \frac1N\sum_i\int k_i f_i-\mathcal J_h^{\mathrm{Lap}}(x)
  =\frac1N\sum_i\int k_i\,\mathcal A_\# d_i.
\]
Integration by parts gives, for each $i$,
\begin{align*}
  \int k_i\,\mathcal A_\#d_i
  &=\int k_i\,\divergence d_i+\int k_i\sigma_{\nu,h}\cdot d_i\\
  &=-\int\nabla k_i\cdot d_i+\int k_i\sigma_{\nu,h}\cdot d_i\\
  &=\int k_i(\sigma_{\nu,h}-\nabla\log k_i)\cdot d_i.
\end{align*}
Lemma~\ref{lem:sharp-kernel-moments} bounds both score norms by
$1/h$. Thus the absolute value of the last expression is at most
$(2/h)\int k_i\norm{d_i}$. Averaging over $i$ gives
\[
  \abs{\frac1N\sum_i\int k_i f_i-\mathcal J_h^{\mathrm{Lap}}(x)}
  \le\frac2h\ell_1(x).
\]
The difference between $\mathsf S_N^{\mathrm{Lap}}$ and the averaged
integral on the left is at most $q_S(x)$. The triangle inequality
proves \eqref{eq:laplace-integrated-S}.

For the energy comparison, the identity
$\norm{v}^2-\norm{w}^2=(v-w)\cdot(v+w)$ gives
\[
  \abs{\norm{u_{x,-i}}^2-\norm{u_x}^2}
  \le\norm{d_i}(\norm{u_{x,-i}}+\norm{u_x}).
\]
Multiplying by $k_i$, integrating, and averaging yields
\[
  \abs{\frac1N\sum_i\int k_i g_i-\mathcal V_h^{\mathrm{Lap}}(x)}
  \le\ell_V(x).
\]
The difference between $\mathsf V_N^{\mathrm{Lap}}$ and the averaged
integral is at most $q_V(x)$. The triangle inequality proves
\eqref{eq:laplace-integrated-V}.
\end{proof}

\begin{lemma}[Local quadrature with an explicit tail]
\label{lem:laplace-local-quadrature}
Let $f$ be twice continuously differentiable on a neighborhood of
$\overline B(x_i,r_i)$, where $r_i>0$, and assume
$\sup_{z\in B(x_i,r_i)}\norm{D^2f(z)}_{\mathrm{op}}\le B_i$.
Suppose the tail integral below is finite. Then
\begin{align}
  \abs{f(x_i)-\int k_i(z)f(z)\dd z}
  \le{}&
  \frac{B_i}{2}\int_{\norm{u}\le r_i}
             \norm{u}^2\overline L_h(u)\dd u \notag\\
  &+\int_{\norm{u}>r_i}
       \overline L_h(u)\abs{f(x_i+u)-f(x_i)}\dd u.
  \label{eq:laplace-local-quadrature}
\end{align}
\end{lemma}

\begin{proof}
Change variables $z=x_i+u$ and split the difference integral into
$\norm{u}\le r_i$ and its complement. In the ball, Taylor's formula is
\[
  f(x_i+u)-f(x_i)
  =\nabla f(x_i)\cdot u
   +\int_0^1(1-s)u^\top D^2f(x_i+su)u\dd s.
\]
The ball and $\overline L_h$ are symmetric, so the integral of the
linear term over the ball vanishes. Taking absolute values of the
remaining integral and using $\int_0^1(1-s)\dd s=1/2$ gives the
first term of \eqref{eq:laplace-local-quadrature}. Taking the
absolute value inside the integral over the complement gives the
second term. Adding the two estimates proves the result.
\end{proof}

For a collision-free configuration whose centers lie outside the finite
support of the data, choose
\[
  0<r_i<\min\left\{\min_{j\ne i}\norm{x_i-x_j},
                         \operatorname{dist}(x_i,\operatorname{supp}\nu)\right\}.
\]
Each closed ball then has a neighborhood avoiding all other particle
centers and all data atoms. Both $f_i$ and $g_i$ are smooth on this
neighborhood.
Applying Lemma~\ref{lem:laplace-local-quadrature} to each $f_i$ and
averaging bounds $q_S$; applying it to each $g_i$ bounds $q_V$.
By \eqref{eq:sharp-kernel-moments}, their local second-moment
contributions are at most
\[
  \frac{d(d+3)h^2}{2N}\sum_i B_i
\]
with the appropriate Hessian bounds for each family. The tail
integrals remain explicit. For a continuous data measure, the same
conclusion holds whenever the data contribution has the stated
local differentiability. These comparisons require no derivative
of the full field at a particle center.

\begin{corollary}[Quadrature tails under bounded support]
\label{cor:laplace-quadrature-tails}
Suppose $\nu$ is supported in $B(0,R_\nu)$ and all centers lie in
$B(0,R_T)$. Write $B=R_\nu+R_T$. Wherever the weak derivative is
represented pointwise,
\[
  \abs{f_i(z)}\le\frac{(2d+1)B}{h},\qquad 0\le g_i(z)\le B^2.
\]
The same bounds hold at the classical evaluation point $x_i$.
Thus the tail term in \eqref{eq:laplace-local-quadrature} is at most
$2(2d+1)Bh^{-1}p_d(r_i/h)$ for $f_i$, and at most
$B^2p_d(r_i/h)$ for $g_i$.
\end{corollary}

\begin{proof}
Let $\alpha$ be supported in $B(0,R_\alpha)$ and let
$m_{\alpha,h}(z)$ be its kernel-weighted mean of $y$.
Differentiating the quotient defining this mean gives, almost
everywhere,
\[
  Dm_{\alpha,h}(z)
  =\int\frac{K_h(z-y)}{Q_{\alpha,h}(z)}
       (y-m_{\alpha,h}(z))(\nabla_z\log K_h(z-y))^\top
       \,\alpha(\dd y).
\]
The Laplace score has norm $1/h$ away from its origin, and
$\norm{y-m_{\alpha,h}(z)}\le2R_\alpha$. The triangle inequality
for the integral therefore gives
$\norm{Dm_{\alpha,h}(z)}_{\mathrm{op}}\le2R_\alpha/h$.
The formula and bound are understood almost everywhere; differentiation
under the integral follows from the locally bounded denominator and
the bounded weak derivatives of the kernel.

Since $u_{x,-i}=m_{\nu,h}-m_{\mu_{x,-i}^N,h}$, it follows that
$\norm{u_{x,-i}}\le B$ and
$\norm{Du_{x,-i}}_{\mathrm{op}}\le2B/h$.
The trace bound $\abs{\operatorname{tr}A}\le d\norm{A}_{\mathrm{op}}$
and $\norm{\sigma_{\nu,h}}\le1/h$ give
\[
  \abs{f_i}
  \le\abs{\divergence u_{x,-i}}+
       \norm{\sigma_{\nu,h}}\norm{u_{x,-i}}
  \le\frac{2dB}{h}+\frac Bh.
\]
Also $g_i=\norm{u_{x,-i}}^2$ lies in $[0,B^2]$.
At a classical evaluation point the derivative bound follows from
the same local Lipschitz bound.
Hence $\abs{f_i(z)-f_i(x_i)}\le2(2d+1)B/h$ almost everywhere,
and $\abs{g_i(z)-g_i(x_i)}\le B^2$.
Integration outside the ball and
\eqref{eq:sharp-kernel-tail} prove both tail estimates.
\end{proof}

\begin{lemma}[Leave-one-out perturbations of field values]
\label{lem:laplace-loo-values}
For fixed $i$ and $z$, set
\[
  c_i(z)=K_h(z-x_i),\qquad
  D_i(z)=\sum_{j\ne i}K_h(z-x_j).
\]
Then
\begin{equation}
  d_i(z)
  =\frac{c_i(z)}{D_i(z)+c_i(z)}
       \{x_i-z-M_{\mu_{x,-i}^N,h}(z)\}.
  \label{eq:laplace-loo-values}
\end{equation}
At $z=x_i$ this also reads
\begin{equation}
  d_i(x_i)
  =-\frac{K_h(0)}{D_i(x_i)}M_{\mu_x^N,h}(x_i).
  \label{eq:laplace-loo-center-values}
\end{equation}
If $Q_{\mu_{x,-i}^N,h}(x_i)\ge\lambda>0$, then
\[
  \norm{d_i(x_i)}
  \le\frac{c_d}{(N-1)\lambda h^d}
       \norm{M_{\mu_x^N,h}(x_i)}.
\]
\end{lemma}

\begin{proof}
Let $P_i(z)=\sum_{j\ne i}(x_j-z)K_h(z-x_j)$.
The full model mean shift is
$(P_i+(x_i-z)c_i)/(D_i+c_i)$, and the leave-one-out model mean
shift is $P_i/D_i$. Their difference is
\[
  \frac{P_i+(x_i-z)c_i}{D_i+c_i}-\frac{P_i}{D_i}
  =\frac{c_i}{D_i+c_i}\left(x_i-z-\frac{P_i}{D_i}\right).
\]
The data terms cancel in $u_{x,-i}-u_x$, leaving exactly this
full-minus-leave-one-out model difference. This proves
\eqref{eq:laplace-loo-values}.
At $z=x_i$, the full model mean shift is $P_i/(D_i+K_h(0))$,
so the same difference equals
$-K_h(0)P_i/\{D_i(D_i+K_h(0))\}$, which proves
\eqref{eq:laplace-loo-center-values}.
Finally, $D_i(x_i)=(N-1)Q_{\mu_{x,-i}^N,h}(x_i)\ge(N-1)\lambda$
and $K_h(0)=c_dh^{-d}$. Taking norms proves the last bound.
\end{proof}

The value identity also bounds $\ell_1$ and $\ell_V$ when the
denominators and the local weighted moments are controlled throughout
the regions of integration. The denominator at a center is controlled
by the same local-occupancy mechanism as in the conservative analysis.

\begin{lemma}[Leave-one-out Laplace denominator control]
\label{lem:laplace-occupancy}
Fix \(r_0>0\). For the Laplace kernel,
\[
  K_h(u)\ge c_dh^{-d}e^{-r_0}
  \qquad\text{whenever }\norm{u}\le r_0h.
\]
For a configuration \(x\), define
\[
  n_i^{(-)}(x)
  :=
  \#\{j\ne i:\norm{x_j-x_i}\le r_0h\}.
\]
If
\[
  n_i^{(-)}(x)\ge \alpha (N-1)h^d
  \qquad\text{for every }i,
\]
then
\[
  Q_{\mu_{x,-i}^{N},h}(x_i)
  \ge
  c_de^{-r_0}\alpha
  \qquad\text{for every }i.
\]
\end{lemma}

\begin{proof}
By definition,
\[
  Q_{\mu_{x,-i}^{N},h}(x_i)
  =
  \frac1{N-1}\sum_{j\ne i}K_h(x_i-x_j).
\]
For every \(j\ne i\) satisfying \(\norm{x_j-x_i}\le r_0h\),
\[
  K_h(x_i-x_j)
  =
  c_dh^{-d}\exp\!\left(-\frac{\norm{x_i-x_j}}{h}\right)
  \ge
  c_dh^{-d}e^{-r_0}.
\]
Therefore
\[
  Q_{\mu_{x,-i}^{N},h}(x_i)
  \ge
  \frac1{N-1}n_i^{(-)}(x)c_dh^{-d}e^{-r_0}.
\]
Using \(n_i^{(-)}(x)\ge \alpha(N-1)h^d\),
\[
  Q_{\mu_{x,-i}^{N},h}(x_i)
  \ge
  \frac1{N-1}\{\alpha(N-1)h^d\}c_dh^{-d}e^{-r_0}
  =
  c_de^{-r_0}\alpha.
\]
This proves the claim.
\end{proof}

\subsection{Bandwidth and Step-size Choices for Laplace Drifting}
\label{subsec:laplace-bandwidth-eta}

For $H_N^\#(0)\le\kappa_0N$ and $T=N$, write the right-hand side of
\eqref{eq:laplace-main-rate-TN} as
\begin{equation}
  \mathcal R_{N,h}:=
  \frac{2\kappa_0L_h}{N}
  +L_h\Delta_h^2+2L_h\varepsilon_{S,h,N}
  +\tau_{h,N}+\varepsilon_{V,h,N}.
  \label{eq:laplace-bandwidth-bound}
\end{equation}
The bandwidth enters through all five terms. The local scale and
tail energy can be made explicit under confinement. For example,
with the hypotheses of Proposition~\ref{prop:laplace-confinement},
choose $r=ch\log N$ for a fixed $c>0$ and $N\ge2$. Then
\[
  L_h=h\{2R_T+ch\log N+h\},
  \qquad
  \tau_{h,N}=(R_\nu+R_T)^2p_d(c\log N).
\]
For fixed $d,c$, the finite polynomial in
\eqref{eq:sharp-kernel-tail} gives
\[
  \tau_{h,N}
  =O\!\left((R_\nu+R_T)^2\{1+(\log N)^d\}N^{-c}\right).
\]
The remaining terms require weighted shell alignment and the
integrated approximation estimates. A sequence $h=h_N$ gives a
vanishing root residual whenever
\[
  \frac{L_{h_N}}N\to0,\qquad
  L_{h_N}\Delta_{h_N}^2\to0,\qquad
  L_{h_N}\varepsilon_{S,h_N,N}\to0,\qquad
  \tau_{h_N,N}+\varepsilon_{V,h_N,N}\to0,
\]
with bounded $\kappa_0$ and all hypotheses valid on $[0,N]$.
Under relative alignment, one instead uses
Corollary~\ref{cor:laplace-relative-rate}, retaining the factor
$(1-\sqrt{\theta_{h_N}})^{-1}$.

The condition $Nh^d\gtrsim\log N$ can ensure simultaneous local
occupancy under an appropriate lower local mass condition. It does
not alone control the quadrature tails, the integrated leave-one-out
errors, or confinement over $[0,N]$. Thus the bandwidth is chosen by
balancing the terms in \eqref{eq:laplace-bandwidth-bound} over a range
where their dependence on $N$ and $h$ is controlled.

The explicit step size enters through
$\widetilde x_i=x_i+\eta u_{x,-i}(x_i)$. If $\tau$ is independent
and uniform on $[0,N]$, the identity coupling gives
\[
  \E W_2^2(\mu_{X(\tau)}^N,\widetilde\mu_{X(\tau)}^N)
  \le\eta^2\E\mathsf V_N^{\mathrm{Lap}}(X(\tau))
  \le\eta^2\mathcal R_{N,h}.
\]
Taking square roots yields a root mean squared correction at most
$\eta\sqrt{\mathcal R_{N,h}}$.
For approximation of the interacting-particle ODE, the relevant
Lipschitz constant is that of the full configuration vector field.
The admissible step size must account for this constant and the
time horizon as described in Section~\ref{app:stop-gradient-ode}.

\subsection{Comparison Between the Conservative and Non-conservative Results}
\label{subsec:laplace-vs-gradient-kde}

The conservative field $b_x=s_\rho-s_x$ is a score mismatch, and its
KDE-averaged Stein drift equals the smoothed Fisher discrepancy.
The joint-entropy argument and quadrature comparison give
\[
  \frac1N\int_0^N\E_t[\mathsf V_N(X(t))]\dd t
  \le\frac{\kappa_0}{N}
       +\frac{a_1\Lambda_N(h)}{Nh^{d+2}}
       +C_N(h)h^2.
\]
The reciprocal-KDE term comes from differentiation of a particle's
own contribution to the full model KDE. The exact expectation
bounds in Section~\ref{sec:reciprocal-control} control this term
under lower local mass and distortion conditions. The full-support
initialization result identifies a regime in which its coefficient
must grow with $N$.

The optimized conservative root rate is $N^{-1/(d+4)}$ when
$\kappa_0$, $a_1\Lambda_N(h)$, and $C_N(h)$ are uniformly bounded.
If instead $C_N(h)=O(h^{-\beta})$, with $0\le\beta<2$ and the other
constants uniformly bounded, the root rate is
$N^{-(2-\beta)/(2(d+4-\beta))}$.

For Laplace drifting, the decomposition
$u_x=a_{x,h}b_x^\#+e_x$ gives a weighted coercivity relation.
The scale $a_{x,h}$ grows in the tails, so the unweighted
center-velocity estimate uses local scale control and tail energy:
\[
  \frac1N\int_0^N\E_t[\mathsf V_N^{\mathrm{Lap}}(X(t))]\dd t
  \le\frac{2\kappa_0L_h}{N}
      +L_h\Delta_h^2+2L_h\varepsilon_{S,h,N}
      +\tau_{h,N}+\varepsilon_{V,h,N}.
\]
Absolute alignment bounds the weighted mismatch additively; relative
alignment absorbs it into coercivity. The leave-one-out dynamics
remove the reciprocal self-interaction correction from the joint
entropy identity. Their comparison with full-field integrals is
expressed through local leave-one-out quadrature and integrated
field-value perturbations.

The reference densities also differ. The conservative analysis uses
$\rho=K_h*\nu$, while the Laplace analysis uses
\[
  \rho_{\nu,h}^\#(z)
  =\frac1{(d+1)h^2}\int L_h(z-y)\,\nu(\dd y),
  \qquad L_h(u)=h(\norm{u}+h)K_h(u).
\]
The relation $\nabla L_h(z-y)=(y-z)K_h(z-y)$ connects the
displacement field to this sharp-smoothed score.
Both analyses control time-averaged residual drift under explicit
regularity and stability conditions. The additional arguments needed
to obtain distributional convergence or uniform finite-step guarantees
are separate from these entropy estimates.

\section{Stop-gradient Training and The Particle ODE}
\label{app:stop-gradient-ode}

We briefly explain how the continuous-time particle dynamics arise from the practical stop-gradient training rule. Let \(g_\theta:\mathcal Z\to\R^d\) be the generator and let
\[
  x_i^k:=g_{\theta_k}(\xi_i),
  \qquad i=1,\ldots,N,
\]
be the generated batch at training step \(k\), for fixed latent variables
\(\xi_1,\ldots,\xi_N\). Let \(v_{x^k}\) denote the drift field computed from the current generated batch
\[
  x^k=(x_1^k,\ldots,x_N^k)
\]
and the data batch. In the original displacement-based drifting method~\citep{deng2026drifting}, \(v_{x^k}\) is the non-conservative field \(u_{\nu,\mu_{x^k},h}^{\mathrm{disp}}\); in the conservative method studied in this paper, it is the KDE-score field \(b_{x^k}\). The practical target for the \(i\)-th generated sample is
\[
  y_i^k
  :=
  \operatorname{sg}\{x_i^k+\eta v_{x^k}(x_i^k)\}.
\]
Here \(\operatorname{sg}\) denotes the stop-gradient operator. For any quantity
\(a(\theta)\) appearing in the computational graph,
\[
  \operatorname{sg}\{a(\theta)\}=a(\theta)
  \quad\text{in the forward pass,}
  \qquad
  \nabla_\theta \operatorname{sg}\{a(\theta)\}=0
  \quad\text{in the backward pass.}
\]
Thus \(\operatorname{sg}\) preserves the numerical value of its argument but
removes all derivatives through that argument. In particular, the target $  y_i=\operatorname{sg}\{x_i+\eta v_x(x_i)\}$ is evaluated as \(x_i+\eta v_x(x_i)\), but is treated as a constant when
differentiating the regression loss with respect to the generator parameters.

The corresponding regression loss is
\begin{equation}
  \mathcal L_k(\theta)
  :=
  \frac1{2N}\sum_{i=1}^N
  \norm{g_\theta(\xi_i)-y_i^k}^2.
  \label{eq:stop-gradient-loss}
\end{equation}
The stop-gradient operation means that \(y_i^k\) is treated as a constant when differentiating \(\mathcal L_k\) with respect to \(\theta\). Thus the derivative of the loss is
\[
  \nabla_\theta\mathcal L_k(\theta)
  =
  \frac1N\sum_{i=1}^N
  D_\theta g_\theta(\xi_i)^\top
  \{g_\theta(\xi_i)-y_i^k\},
\]
with no derivative through the map \(x^k\mapsto v_{x^k}\). Therefore the training step is a frozen-target regression problem: the generator is trained to move its current outputs toward the fixed points
\[
  x_i^k+\eta v_{x^k}(x_i^k).
\]

The idealized exact-regression limit is the limit in which this regression problem is solved exactly at every step and the generator class is rich enough to interpolate the frozen targets on the sampled latent variables. In that limit, the next generator satisfies
\[
  g_{\theta_{k+1}}(\xi_i)=y_i^k
  =
  x_i^k+\eta v_{x^k}(x_i^k),
  \qquad i=1,\ldots,N.
\]
Consequently the generated particles obey the explicit Euler update
\begin{equation}
  x_i^{k+1}
  =
  x_i^k+\eta v_{x^k}(x_i^k),
  \qquad i=1,\ldots,N.
  \label{eq:frozen-field-euler}
\end{equation}
If the regression is not exact, then \eqref{eq:frozen-field-euler} holds with an additional approximation error
\[
  e_i^k
  :=
  g_{\theta_{k+1}}(\xi_i)
  -
  \{x_i^k+\eta v_{x^k}(x_i^k)\}.
\]
Our analysis neglects this optimization and approximation error and studies the idealized particle system \eqref{eq:frozen-field-euler}.

Finally, define \(t_k=k\eta\) and let \(X_i^\eta(t)\) be the piecewise-linear interpolation of the Euler iterates. Under the usual consistency and stability conditions for explicit Euler schemes, for example local Lipschitz continuity of the interacting-particle vector field
\[
  F_i(x):=v_x(x_i),
  \qquad x=(x_1,\ldots,x_N),
\]
the interpolation \(X^\eta(t)\) converges as \(\eta\to0\) on finite time intervals to the solution of
\[
  \dot X_i(t)=F_i(X(t))=v_{X(t)}(X_i(t)),
  \qquad i=1,\ldots,N.
\]
Here, \(v_x\) can either be the conservative KDE-score field or the
non-conservative leave-one-out Laplace field. These are the continuous-time particle dynamics analyzed in the main text. Thus stop-gradient is used to justify the frozen-field Euler update; once this particle ODE is assumed, the entropy identities and convergence bounds are properties of the resulting interacting-particle flow.

We emphasize that the passage from the stop-gradient Euler update to the continuous-time ODE is
an idealized modeling step. Let $  F(x):=(F_1(x),\ldots,F_N(x))$. The exact-regression
stop-gradient update is
\[
  x^{k+1}=x^k+\eta F(x^k).
\]
Suppose $F$ is $L_{h,\mathcal D}$-Lipschitz on a region
$\mathcal D$ containing the Euler iterates and the exact trajectory,
and $\sup_{\mathcal D}\norm{F}\le B_{h,\mathcal D}$.
For brevity write $L=L_{h,\mathcal D}$ and $B=B_{h,\mathcal D}$.
The exact solution satisfies, for $s\in[0,\eta]$,
$\norm{X(t+s)-X(t)}\le Bs$, and therefore
\begin{align*}
  \norm{X(t+\eta)-X(t)-\eta F(X(t))}
  &\le\int_0^\eta\norm{F(X(t+s))-F(X(t))}\dd s\\
  &\le\int_0^\eta LBs\dd s
   =\frac12LB\eta^2.
\end{align*}
Writing $e_k=\norm{x^k-X(k\eta)}$ with $e_0=0$, the Lipschitz
bound for one Euler step gives
\[
  e_{k+1}\le(1+L\eta)e_k+\frac12LB\eta^2.
\]
For $L>0$, repeated substitution and the finite geometric sum yield
\[
  e_k\le\frac12LB\eta^2\sum_{j=0}^{k-1}(1+L\eta)^j
       =\frac{B\eta}{2}\{(1+L\eta)^k-1\}
       \le\frac{B\eta}{2}(e^{Lk\eta}-1).
\]
If $L=0$, the field is constant on $\mathcal D$ and Euler is exact.
Consequently, for fixed $T,h,\mathcal D$,
\[
  \max_{k\eta\le T}\norm{x^k-X(k\eta)}
  \le\frac{B_{h,\mathcal D}\eta}{2}
       (e^{L_{h,\mathcal D}T}-1).
\]
The constants involve derivatives of the full configuration map,
including its dependence on all KDE centers.
When $h$ and $T$ vary with $N$, a small value of
$\eta L_{h,\mathcal D}$ alone does not ensure that this global bound
vanishes. A sufficient condition from the displayed estimate is
\[
  \eta B_{h,\mathcal D}(e^{L_{h,\mathcal D}T}-1)\longrightarrow0.
\]
Stronger dynamical stability can give a more favorable dependence
on the horizon. The continuous-time entropy rates are therefore
interpreted with their independent regularity and stability
conditions when applied to a finite-step training procedure.

If the regression step is not exact, the particle update has the form $  x^{k+1}=x^k+\eta F(x^k)+e^k.
$, as mentioned above. Under the same Lipschitz assumptions, the discrete-to-continuous error contains
an additional accumulated term of size
\[
  \sum_{k\eta\le T} e^{L_{h,\mathcal D}(T-k\eta)}\norm{e^k}.
\]
Thus tracking the practical neural-network training dynamics would require
bounds on optimization error, approximation error, minibatch noise, and the
\(h\)-dependent stability constant. These effects are outside the scope of the
continuous-time entropy analysis.

\end{document}